\documentclass{article}

\usepackage[final, nonatbib]{neurips_2022}

\usepackage[utf8]{inputenc} 
\usepackage[T1]{fontenc}    
\usepackage[colorlinks,
            linkcolor=red,
            anchorcolor=blue,
            citecolor=blue, 
            backref=page
            ]{hyperref}       
\usepackage{url}            
\usepackage{booktabs}       
\usepackage{amsfonts}       
\usepackage{nicefrac}       
\usepackage{microtype}      
\usepackage{xcolor}         

\usepackage{amsmath, amsfonts, amssymb, amsthm, amsbsy, amscd, bm, bbm,mathrsfs} 
\usepackage{graphicx,subfig}
\usepackage{wrapfig}

\newcommand{\EE}{\operatorname{\mathbb{E}}}

\newcommand{\PP}{\mathbb{P}}

\newcommand{\RR}{\mathbb{R}}
\renewcommand{\SS}{\mathbb{S}}

\newcommand{\cF}{\mathcal{F}}
\newcommand{\cG}{\mathcal{G}}
\newcommand{\cH}{\mathcal{H}}

\newcommand{\cK}{\mathcal{K}}

\newcommand{\cN}{\mathcal{N}}

\newcommand{\cX}{\mathcal{X}}

\newcommand{\hh}{\hat{h}}

\newcommand{\hk}{\hat{k}}

\newcommand{\hv}{\hat{v}}

\newcommand{\hK}{\hat{K}}

\newcommand{\ttheta}{\tilde{\theta}}

\newcommand{\<}{\left\langle}
\renewcommand{\>}{\right\rangle}

\newcommand{\half}{\frac{1}{2}}

\DeclareMathOperator{\relu}{ReLU}

\newcommand{\unif}{\mathop\mathrm{Unif}}

\newcommand{\wrt}{with respect to }
\renewcommand{\wp}{\text{w.p.~}}
\newcommand{\iid}{\textrm{i.i.d.~}}

\renewcommand{\leq}{\leqslant}
\renewcommand{\geq}{\geqslant}

\newcommand{\erisk}{L}
\newcommand{\prisk}{L}

\newcommand{\erad}{\widehat{\operatorname{Rad}}_n}

\newcommand{\fn}{\frac{1}{n}}

\newcommand{\sumin}{\sum_{i=1}^n}
\newcommand{\sumjn}{\sum_{j=1}^n}

\newcommand{\sumim}{\sum_{i=1}^m}

\definecolor{lw}{RGB}{100,100,255}

\newtheorem{theorem}{Theorem}[section]
\newtheorem{lemma}[theorem]{Lemma}
\newtheorem{proposition}[theorem]{Proposition}

\newtheorem{definition}[theorem]{Definition}

\newtheorem{assumption}{Assumption}

\DeclareMathOperator{\tr}{Tr}

\newcommand{\hEE}{\hat{\EE}}

\newcommand{\bchi}{\bar{\chi}}
\newcommand{\lin}{\mathrm{lin}}
\newcommand{\bmu}{\mu}

\title{When does SGD favor flat minima? A stability analysis}
\title{The noise geometry is critical for SGD selecting flat minima: A stability analysis}
\title{How noise geometry helps SGD avoid sharp minima: \\ A stability analysis}
\title{How SGD noise help avoid sharp minima: \\ A stability analysis}
\title{The alignment property of SGD noise and how it \\ helps find flat minima}
\title{The alignment property of noise is critical for SGD selecting flat minima: A stability analysis}
\title{The alignment property of SGD noise and how it helps select flat minima: A stability analysis}

\author{%
  Lei Wu \\
  School of Mathematical Sciences\\
  Peking University
  \And
  Mingze Wang\\
  School of Mathematical Sciences\\
  Peking University
  \AND
  Weijie Su\\
  Department of Statistics and Data Science\\
  University of Pennsylvania
}

\begin{document}

\maketitle

\begin{abstract}
The phenomenon that stochastic gradient descent (SGD) favors flat minima has played a critical role in   understanding the implicit regularization of SGD. In this paper, we provide an explanation of this striking phenomenon by relating the particular noise structure  of SGD to its \emph{linear stability}  (Wu et al., 2018). Specifically, we consider training over-parameterized models with square loss. We prove that if a global minimum $\theta^*$ is linearly stable for SGD, then it must satisfy $\|H(\theta^*)\|_F\leq O(\sqrt{B}/\eta)$, where $\|H(\theta^*)\|_F, B,\eta$ denote the Frobenius norm of Hessian at $\theta^*$, batch size, and learning rate, respectively. Otherwise, SGD will escape from that minimum \emph{exponentially} fast. Hence, for minima accessible to SGD, the sharpness---as measured by the Frobenius norm of the Hessian---is bounded \emph{independently} of the model size and sample size.   The key to obtaining these results is exploiting the particular structure of SGD noise: The noise concentrates in sharp directions of local landscape and  the magnitude is proportional to loss value. This alignment property of SGD noise provably holds for linear networks and random feature models (RFMs), and is empirically verified for nonlinear networks. Moreover, the validity and practical relevance of our theoretical findings are also justified by extensive experiments on CIFAR-10 dataset.
\end{abstract}

\section{Introduction}
\vspace*{-.4em}
Modern machine learning (ML) models are often operated with far more unknown parameters than training examples, a regime referred to as over-parameterization. In this regime, there are many global minima, all of which have zero training loss but their test performance can be significantly different \cite{wu2017towards}. Fortunately, it is often observed that SGD converges to those generalizable ones, even without needing any explicit regularizations \cite{zhang2017understanding}. This suggests there must exist certain ``implicit regularization'' mechanism at work  \cite{neyshabur2014search,he2019local,wu2017towards,blanc2020implicit}. 

More mysteriously,  SGD solutions often generalize better than gradient descent (GD) solutions \cite{keskar2016large,su2021neurashed}. Therefore, the SGD noise must play a critical role in implicit regularization.
The most popular explanation is that SGD favors flatter minima \cite{keskar2016large} and flatter minima generalize better \cite{hochreiter1997flat}. This flat-minima principle has been extensively and successfully adopted in practice to tune the hyperparameters of SGD \cite{wu2020noisy,keskar2016large} and to design new optimizers \cite{izmailov2018averaging,foret2020sharpness,wu2020adversarial} for improving generalization. Therefore, understanding how SGD noise biases SGD towards flatter minima is of paramount importance, which is the main focus of this paper.

The works 
\cite{wu2020noisy,feng2021inverse,xie2020diffusion} show that SGD noise is highly anisotropic; 
\cite{mori2021logarithmic,wojtowytsch2021stochastic} find the magnitude of SGD noise to be loss dependent. Both structures are shown to be critical for SGD picking flat minima. 
However, these works all make unrealistic (even wrong) over-simplifications of SGD noise (see the related work section for more details) in their analysis. In addition,  instead of studying SGD, they all consider the continuous-time stochastic differential equation (SDE), which is a good modeling of SGD only in finite time and small learning rate (LR) regime \cite{pmlr-v70-li17f}. It is generally unclear how this SDE modeling is relevant for understanding  SGD with a large LR---a regime preferred in practice. 
Consequently,  these works only provide intuitive and empirical analyses, lacking a quantitative characterization of when and how  SGD favors flat minima.




Another line of works \cite{wu2018sgd,ma2021linear} relate the selection bias of SGD to the  \emph{dynamical stability}.  In over-parameterized case, all global minima are fixed points of SGD but their dynamical stabilities can be very different. At unstable minima, a small perturbation will drive SGD to leave away, whereas, for stable minima, SGD can stay around and even converge back after initial perturbations.  Thus SGD prefers stable minima over unstable ones. Specifically, \cite{wu2018sgd,ma2021linear} analyze the  linear stability \cite{arnold2012geometrical} of SGD, showing that a linearly stable minimum must be flat and uniform.
Different from SDE-based analysis, this stability-based analysis is  relevant for large-LR SGD and is even empirically accurate in predicting the properties of minima selected by SGD \cite{wu2018sgd,jastrzebski2019break,cohen2021gradient}.

In this work, we follow the linear stability analysis in  \cite{wu2018sgd,ma2021linear} but take
the  particular geometry-aware structure of SGD noise into consideration. We establish a direct connection between linear stability and flatness, which allows us to obtain a quantitative characterization of how the learning rate and batch affect the flatness of minima accessible to SGD.
In contrast, \cite{wu2018sgd,ma2021linear}  have to introduce the another quantity: non-uniformity together with flatness to characterize linear stability because of neglecting the noise structure.


\paragraph*{Setup}
Let $\{(x_i,y_i)\}_{i=1}^n$ with $x_i\in\RR^d, y_i\in\RR^K$ be the training set and $f(\cdot;\theta)$ with $\theta\in\RR^p$ be our model. The \emph{model size} is defined to be the number of parameters $p$ and in this paper. For simplicity, we will always assume $K=1$ and the extension to the case of $K>1$ is straightforward. Let $L_i(\theta)=\half |f(x_i;\theta)-y_i|^2$ be the fitting error at the $i$-th sample and
$
\erisk(\theta) =\fn\sumin L_i(\theta)
$
be the empirical risk. We shall focus on the over-parameterized case in the sense that $\inf_{\theta} \erisk(\theta)=0$.
To minimize $L(\cdot)$, we consider  the mini-batch SGD:
\begin{align}\label{eqn: sgd-1}
\theta_{t+1} = \theta_t - \frac{\eta}{B}\sum_{i\in I_t} \nabla L_i(\theta_t),
\end{align}
where $\eta$ and $B$ are the learning rate and batch size, respectively. 
This SGD can be rewritten as  
$
\theta_{t+1} = \theta_t - \eta (\nabla L(\theta_t) + \xi_t),
$
where $\xi_t$ is the noise, satisfying $\EE[\xi_t]=0$ and 
$
    \EE[\xi_t\xi_t^T] = \Sigma(\theta_t)/B.
$
Here the noise covariance $\Sigma(\theta)=\fn \sumin \nabla L_i(\theta) \nabla L_i(\theta)^T - \nabla \erisk(\theta)\erisk(\theta)^T$.
To characterize the local geometry of loss landscape, we consider  the Fisher matrix:
$
G(\theta)=\fn\sumin \nabla f(x_i;\theta)\nabla f(x_i;\theta)^T  
$
 and the Hessian matrix:
$
H(\theta)=G(\theta) + \fn\sumin (f(x_i;\theta)-y_i)\nabla^2 f(x_i;\theta).
$
When the loss value is small,  $H(\theta)\approx G(\theta)$ and in particular, if $\theta^*$ is an global minimum, $H(\theta^*)=G(\theta^*)$. 

\vspace*{-.4em}
\paragraph*{Notations.} For a vector $a$, let $\|a\|=\sqrt{a^Ta}$ and $\|a\|_W=\sqrt{a^TWa}$. 
For a matrix $A$, denote by $\{\lambda_j(A)\}$ the eigenvalue of $A$ in a decreasing order. 
For other notations, we refer to Appendix \ref{sec: notation}.


\textbf{Our main contributions} are summarized as follows.
\vspace*{-.4em}
\begin{itemize}
    \item 
 We first show that for many ML models,  the SGD noise is geometry aware: 1) the noise magnitude is proportional to the loss value; 2) the noise covariance aligns well with the Fisher matrix.  Specifically, to  quantify the alignment strength, we define a loss-scaled alignment factor $\mu(\theta)$,
 which is proved to be bounded from below, i.e., there exists a \emph{size-independent} positive constant $\mu_0$ such that $\mu(\theta)\geq \mu_0$,  for linear networks (Proposition \ref{thm: over-para-linear-noise-covariance}) and RFMs  (Proposition \ref{pro: random-relu}), and is also empirically  justified for nonlinear networks. 
 Moreover, we identify that it is the uniformity of model gradient norms $\{\|\nabla f(x_i;\theta)\|_{G(\theta)}\}_i$  that accounts for  this \emph{alignment property} of SGD noise. 

  \item  We then provide a thorough analysis of the linear stability of SGD by exploiting the alignment property of noise. We prove in Theorem \ref{thm: Fro-norm-Hessian-bound} that if a global minimum $\theta^*$ is linearly stable, then $\|H(\theta^*)\|_F\leq \eta^{-1}\sqrt{B/\mu_0}$. Here the constant $\mu_0$ quantifies the alignment strength of  SGD noise.
  Hence, for minima accessible to SGD,  the Hessian's Frobenius norm---the flatness perceived by SGD---is bounded independently of the model size and sample size.  
  Moreover,  
  if a minimum is too sharp, violating the preceding stability condition,  SGD will escape from it \textit{exponentially} fast (Theorem \ref{thm: escape}). Together, we obtain a quantitative characterization of when and how much SGD dislikes sharp minima.

  \item Our theoretical findings are also corroborated with well-designed  experiments on a variety of models including linear networks, RFMs, convolutional networks, and fully-connected networks. In particular, the practical relevance is demonstrated in Section \ref{sec: large-scale-experiment} by  extensive experiments on  classifying full CIFAR-10 dataset with VGG nets and ResNets. 

\end{itemize}

\subsection{Related work}
\vspace*{-.5em}

\paragraph*{Noise structures.}  
\cite{zhu2019anisotropic,jastrzkebski2017three,li2021happens} consider the Hessian-based approximation: $\Sigma(\theta)\approx \sigma^2 H(\theta)$, where $\sigma$ is a small constant. 
\cite{ziyin2021minibatch} proposes an improved version: $\Sigma(\theta)\approx 2\erisk(\theta) H(\theta)$. But these approximations in general cannot be accurate since  Hessian is not semi-positive definite (SPD) in non-convex region.
More recently, 
\cite{mori2021logarithmic} and \cite{wojtowytsch2021stochastic} study  SGD by assuming $\Sigma(\theta)=2\erisk(\theta) H(\theta^*)$, where $\theta^*$ is a minimum of interest, and $\Sigma(\theta)=\sigma^2\erisk(\theta) I_p$, respectively. These assumptions completely ignore the state-dependence of noise direction. 
In contrast, we assume  $\Sigma(\theta)=2 \erisk(\theta) C(\theta)$ with $C(\theta)$ having a nontrivial alignment with the Fisher matrix $G(\theta)$, which does not impose any explicit structural assumption on $C(\theta)$. 
As a result, our assumption is much weaker and  can be rigorously justified  for popular ML models both theoretically and empirically. More importantly,  we show that this weak alignment property is  sufficient for analyzing the linear stability of SGD.  We anticipate that our alignment assumption can be also adopted to analyze other properties of SGD.



\vspace*{-.5em}
\paragraph*{Escape from sharp minima.} 
The escape behavior of SGD  was first studied in \cite{zhu2019anisotropic,wu2018sgd}, as an indicator of how much SGD dislikes sharp minima. One of the most mysterious observation is that the escape happens in an unreasonably efficient way.
However, the theoretical analysis there assumes the noise to be state-independent, and consequently, the derived escape time depends polynomially on the loss barrier. 
Later \cite{xie2020diffusion,mori2021logarithmic} attempt to study this issue using the classical diffusion-based framework \cite{gardiner2009stochastic} (It\^{o}-SDE), which cannot explain the unreasonable escape efficiency at all since the resulting escape rates depend on the loss barrier exponentially.
See also \cite{mori2021logarithmic} for an improved analysis.
\cite{simsekli2019tail,zhou2020towards} argues that the SGD noise is heavy-tailed and thus SGD should be modeled as  L\'{e}vy-SDE instead of It\^{o}-SDE. Moreover, it is shown that the heavy-tailedness can ensure the escape rate  depends on the basin volume instead of the loss barrier. Unfortunately, the volume in high dimensions  always scales with the  dimension exponentially and consequently, this  does not explain the escape efficiency in high dimensions. Moreover, whether SGD noise is really heavy-tailed and whether the heavy-tailedness is really important for generalization are still debatable  for neural networks \cite{wang2021eliminating,li2021validity}. In contrast, we show that the unreasonable escape efficiency comes from the particular geometry-aware structure of SGD noise, regardless of whether the noise  is heavy- or light-tailed.

\paragraph*{Flatness metrics}
In the literature, a variety of flatness metrics have been adopted, such as the largest eigenvalue of Hessian \cite{keskar2016large}, the trace of Hessian \cite{damian2021label,blanc2020implicit}, the basin volume \cite{zhou2020towards}, and the ones scaled by parameter norms \cite{pmlr-v89-liang19a,tsuzuku2020normalized} in order to achieve the scaling-invariance for ReLU nets. These metrics are proposed for either computation easiness or bounding generalization gaps. It is unclear if they are perceivable to SGD, let alone how the boundedness of them depends on the batch size, learning rate, as well as the model size and sample size. 
We show that for SGD solutions, the Frobenius norm of Hessian---a flatness perceived by SGD through the linear stability---is bounded by a size-independent quantity. 
Note that a similar stability argument also applies to GD but only yielding the boundedness of the largest eigenvalue of Hessian \cite{wu2018sgd,mulayoff2021implicit}.

Lastly, we particularly mention the work \cite{pesme2021implicit}, 
which provides a fine-grained analysis of the  implicit bias of training two-layer diagonal linear networks.
This work is related to ours since we both consider the  magnitude and direction structure of SGD noise simultaneously. However, the analysis in \cite{pesme2021implicit} is limited to the specific toy model but ours is relevant for general  models.


\section{The alignment property of SGD noise} \label{sec: alignment}
\vspace*{-.4em}
Since $\nabla L_i(\theta)=(f(x_i;\theta)-y_i)\nabla f(x_i;\theta)$, we have the following intuitive approximation \cite{mori2021logarithmic}:
\vspace*{-.3em}
\begin{align}\label{eqn: decoupl-approx}
\notag \Sigma(\theta) &= \frac{2}{n}\sum_{i=1}^n L_i(\theta) \nabla f(x_i;\theta) \nabla f(x_i;\theta)^T - \nabla \erisk(\theta) \nabla \erisk(\theta)^T\stackrel{(i)}{\approx} \frac{2}{n}\sum_{i=1}^n L_i(\theta) \nabla f(x_i;\theta) \nabla f(x_i;\theta)^T\\
& \stackrel{(ii)}{\approx}  2\big(\fn\sumin L_i(\theta)\big) \fn\sumin \nabla f(x_i;\theta) \nabla f(x_i;\theta)^T=2\erisk(\theta)G(\theta),
\vspace*{-.5em}
\end{align}
where $(i)$ assumes that the full-batch gradient $\nabla \erisk(\theta)$ to be negligible compared with the sample gradients $\{\nabla L_i(\theta)\}$; $(ii)$ assumes that $\{\nabla f(x_i;\theta)\}_i$ and  $\{L_i(\theta)\}_i$ are nearly decoupled.  This  approximation  cannot be true in general but tells us that 1)
The noise magnitude is proportional to the loss value; 2) the noise covariance aligns with the Fisher matrix. 

Motivated by \eqref{eqn: decoupl-approx}, we define
$
    \alpha(\theta) = \frac{\tr(G(\theta)\Sigma(\theta))}{\|G(\theta)\|_F\|\Sigma(\theta)\|_F}, \, \beta(\theta) = \frac{\|\Sigma(\theta)\|_F}{2L(\theta)\|G(\theta)\|_F}
$
. Here $\alpha(\theta)$ quantifies the similarity between $\Sigma(\theta)$ and $G(\theta)$, characterizing how much the noise concentrates in sharp directions of local landscape.
$\beta(\theta)$ characterizes the relative non-degeneracy of noise (with respect to the loss value). Then we define the loss-scaled alignment factor:
\begin{equation}\label{eqn: alignment}
\mu(\theta)=\alpha(\theta)\beta(\theta)=\frac{\tr(\Sigma(\theta)G(\theta))}{2L(\theta)\|G(\theta)\|_F^2},
\end{equation}
which characterizes the direction and magnitude alignment simultaneously. Intuitively speaking, if $\mu(\theta)$ is bounded below,  SGD noise is non-degenerate in sharp directions of local landscape. In particular, 
$\alpha(\theta)=\beta(\theta)=\mu(\theta)=1$ if the approximation \eqref{eqn: decoupl-approx} holds. 

We say the noise satisfies the $\mu$-\textbf{alignment} if $\mu(\theta)>0$. Compared with the decoupling approximation \eqref{eqn: decoupl-approx}, the $\mu$-alignment is a much weaker condition. Note that this specific weak quantification of alignment  is inspired for analyzing the linear stability of SGD, which is the focus of this paper. Specifically, Theorem \ref{thm: Fro-norm-Hessian-bound} shows that $\mu(\theta)$ along with the Frobenius norm of Hessian determines the linear stability of SGD.
One may define other metrics to quantify alignment strength for studying other properties of SGD, but this is beyond the scope of this paper. 





\paragraph*{A relaxed alignment.}
Let $\Sigma_1(\theta) = \frac 1 n\sumin \nabla L_i(\theta)\nabla L_i(\theta)^T, \Sigma_2(\theta) = \nabla L(\theta)\nabla L(\theta)^T$. Then $\Sigma(\theta)=\Sigma_1(\theta)-\Sigma_2(\theta)$.
It is often believed that the full-batch gradient $\nabla L$ is relatively small compared to the sample gradients $\{\nabla L_i\}_i$. As a result, the influence of $\Sigma_2(\theta)$ should be negligible compared to $\Sigma_1(\theta)$. To disentangle the influences of them, we define  
\[
\mu_1(\theta)=\frac{\tr(\Sigma_1(\theta)G(\theta))}{2L(\theta)\|G(\theta)\|^2_F}, \quad \mu_2(\theta)=\frac{\tr(\Sigma_2(\theta)G(\theta))}{2L(\theta)\|G(\theta)\|^2_F}. 
\]
Then $\mu(\theta)=\mu_1(\theta)-\mu_2(\theta)$.  Our linear stability analysis in Section \ref{sec: linear-stability} show that  $\mu_1(\theta)\geq \bmu_1>0$, a condition we refer to as  \textbf{$\mu_1$-alignment}, is sufficient to ensure that SGD only selects flat minima. 
 
\subsection{Why does the alignment property hold?}
\begin{definition}[Norm uniformity of model gradients]\label{def: uniformity}
Let $g_i(\theta) = \nabla f(x_i;\theta)$, $\chi_i(\theta):= \|g_i(\theta)\|_{G(\theta)}^2=g_i^T(\theta)G(\theta)g_i(\theta)$, $\bchi(\theta) = \fn\sumin \chi_i(\theta)$.  Define  
$
\gamma(\theta):= \min_{i\in [n]}\frac{\chi_i(\theta)}{\bchi(\theta)}.
$
\end{definition}
The quantity $\gamma(\theta)$ measures the uniformity of model gradient norms and this property can guarantee the $\mu_1$-alignment as stated below. 
 
\begin{lemma}\label{pro: alignment-linearized-nonlinear-model}
$\mu_1(\theta)\geq \gamma(\theta)$
\end{lemma}

\vspace*{-1em}
\begin{proof}
Noticing $\bchi(\theta)=\fn\sumin g_i(\theta)^TG(\theta)g_i(\theta)=\tr\big( G(\theta) \fn\sumin g_i(\theta))g_i(\theta)^T\big)=\|G(\theta)\|_F^2$, we have 
\vspace*{-1em}
\begin{align*}
\tr(\Sigma_1(\theta)G(\theta))&=\frac{2}{n}\sumin L_i(\theta) \tr(g_i(\theta)g_i(\theta)^TG(\theta))\\
&=\frac{2}{n}\sumin L_i(\theta) \chi_i\geq \frac{2}{n} \sumin L_i(\theta) \gamma \bchi=2\gamma L(\theta)\|G(\theta)\|_F^2
\end{align*}
Thus  $\mu_1(\theta)=\tr(\Sigma_1(\theta)G(\theta))/(2L(\theta)\|G(\theta)\|_F^2)\geq \gamma(\theta)$.
\end{proof}

\vspace*{-.5em}
The above proof suggests that the ``decoupling'' approximation in \eqref{eqn: decoupl-approx} holds in a weak sense if $\{\|\nabla f(x_i;\theta)\|_{G(\theta)}\}_i$ are uniform.  One can apply a similar argument by assuming the uniformity of the fitting errors $\{L_i(\theta)\}$, which, unfortunately, we  find  never hold in practice. In contrast, we will show that the norm uniformity of model gradients provably holds for linear networks and RFMs, and can be empirically justified for nonlinear networks.

\paragraph*{Over-parameterized linear models.}
Consider an over-parameterized linear model (OLM): $f(x;\theta)=F(\theta)^Tx$, where $F:\Omega\mapsto\RR^d$ denotes a general re-parameterization function.  Note that $f(\cdot;\theta)$  only represents linear functions but the corresponding landscape can be highly non-convex. Typical examples include the  linear network: $F(\theta) =W_1W_2\cdots W_L$  and the diagonal linear network: $F(\theta)=(\alpha_1^2-\beta_1^2,\dots, \alpha_d^2-\beta_d^2)^T$. Both have attracted a lot of attention in analyzing the implicit bias of GD and SGD \cite{arora2019implicit,woodworth2020kernel,pesme2021implicit,haochen2021shape,azulay2021implicit}.
The following proposition provides a precise characterization of the noise covariance for OLM models, whose  proof is deferred to Appendix \ref{sec: proof-OLM}.
\begin{proposition}\label{thm: over-para-linear-noise-covariance}
Denote by $\cN(0,S)$ the Gaussian distribution with mean zero and covariance matrix $S$. Suppose 
$f(\cdot;\theta)$ is a general OLM and $x\sim\cN(0, S)$. 
 Consider the online SGD setting, i.e., $n=\infty$. Then, $\Sigma(\theta)=\nabla\prisk(\theta)\nabla \prisk(\theta)^T+2\prisk(\theta) G(\theta)$ and $\mu_1(\theta)\geq \mu(\theta)\geq 1$
\end{proposition}
This proposition shows that the alignment property holds in the entire parameter space and moreover, the alignment strength is independent of model size. Here the alignment is only proved  for the  infinite-sample case. Similar results should hold for finite-sample cases  by concentration inequalities as long as $n$ is relatively large, but this straightforward extension does not bring any new insights. It is more interesting to consider the low-sample regime (i.e., $n<d$), where the alignments indeed hold (at least in typical regions explored by SGD) as demonstrated empirically in  Figure \ref{fig: alignment-small-scale}.  In addition,  the above proposition  provides a closed-form expression of the noise covariance, which might be useful for analyzing other properties of SGD beyond the linear stability. A comprehensive analysis of these issues is left to future work.

\paragraph*{Feature-based models.}
Consider a feature-based model $f(x;\theta)=\sum_{j=1}^m \theta_j\varphi_j(x)=\langle \theta, \Phi(x)\rangle$. In this case, the model gradients: $g_i=g_i(\theta)=\Phi(x_i)$ and the Hessian and Fisher matrix: $G=H=\fn\sumin g_ig_i^T$ all are constant. But the noise covariance $\Sigma(\theta)$  is still state-dependent. The norm uniformity of model gradients also becomes constant: $\gamma=\min_i \chi_i/(\fn\sumin \chi_i)$ with $\chi_i=\|g_i\|_{G}^2$.

\begin{lemma}\label{lemma: gen-feature-based-model}
$\mu_1(\theta)\geq \gamma, \mu_2(\theta)\leq \tau(G):=\lambda_{1}^2(G)/\sum_{j}\lambda_j^2(G)$, and $\mu(\theta)\geq \gamma - \tau(G)$.
\end{lemma}
The above lemma suggests that the $\mu_2(\theta)$ term is negligible as long as the Fisher matrix is not nearly rank-1. By bounding the $\gamma$ and $\tau(G)$, we can prove the $\mu_1$- and $\mu$-alignment for random ReLU feature models as stated in the following proposition. The proof is presented in Appendix \ref{sec: proof-rfm-alignemt}, where similar results for general RFMs are provided (see Proposition \ref{pro: rfm-alignment}).
\begin{proposition}\label{pro: random-relu}
If 
$\varphi_j(x)=\relu(w_j^Tx)$ with $w_j\stackrel{iid}{\sim} \unif(\sqrt{d}\SS^{d-1})$ and $x\sim\unif(\SS^{d-1})$. Then, for any $\delta\in (0,1)$, if  $m\geq n\gtrsim d^{5}\log(1/\delta)$, then \wp at least $1-\delta$, $\mu_1(\theta)\geq 1$ and $\mu(\theta)\gtrsim d^{-1}$.
\end{proposition}
  
Although feature-based models are linear, their analysis is still applicable to  understand nonlinear models, as long as the nonlinear model  \emph{locally}  behaves like the linearized one: $f_{\lin}(x;\theta):=f(x;\theta^*) + \langle \theta-\theta^*, \nabla f(x;\theta^*)\rangle$ with  $\nabla f(x;\theta^*)$ learned from data. Hence, Proposition \ref{pro: random-relu} can explain why the alignment property holds in a local region around global minima. Note that this is sufficient for characterizing the linear stability of SGD, which is a local property  in nature. 


\subsection{Empirical validations}\label{sec: experiment-alignment}
Figure \ref{fig: alignment-sgdpath} reports the values of $\alpha(\theta_t),\beta(\theta_t),\mu(\theta_t)$ during the SGD training of four types of models, including the linear networks and RFMs analyzed above, and fully-connected networks (FCN) and  convolutional neural networks (CNN). 
First, one can see that $\alpha(\theta_t)$'s are quite close to $1$ during the entire training, suggesting the strong concentration of SGD noise in sharp directions of local landscape. Second,  $\beta(\theta_t)$'s keep bounded below, implying the noise magnitudes are sufficiently large \wrt the training loss. As a result,    $\mu(\theta_t)$'s are significantly positive for all models examined. In particular, for the linear networks,  the alignment holds in the low-sample regime, where $n<d$.

\paragraph*{The size independence.}
Figure \ref{fig: alignment-modelsize} further examines how the extent of over-parameterization affects the  alignment strength. One can see clearly that for linear networks and RFMs, $\mu(\theta)$'s are independent of the model size, which confirms our theoretical findings proved above. In addition, we also observe that  for nonlinear networks, the alignment strength is also (nearly) independent of the model size. For instance, for CNN, the value of $\mu(\theta)$ only decreases from around $1.05$ to $1.0$ as  the model size is increased by more than two orders of magnitude.

\begin{figure}[!h]
\centering
\captionsetup[subfloat]{farskip=.3pt,captionskip=0.2pt}
\subfloat[]{\label{fig: alignment-sgdpath}
\includegraphics[width=0.24\textwidth]{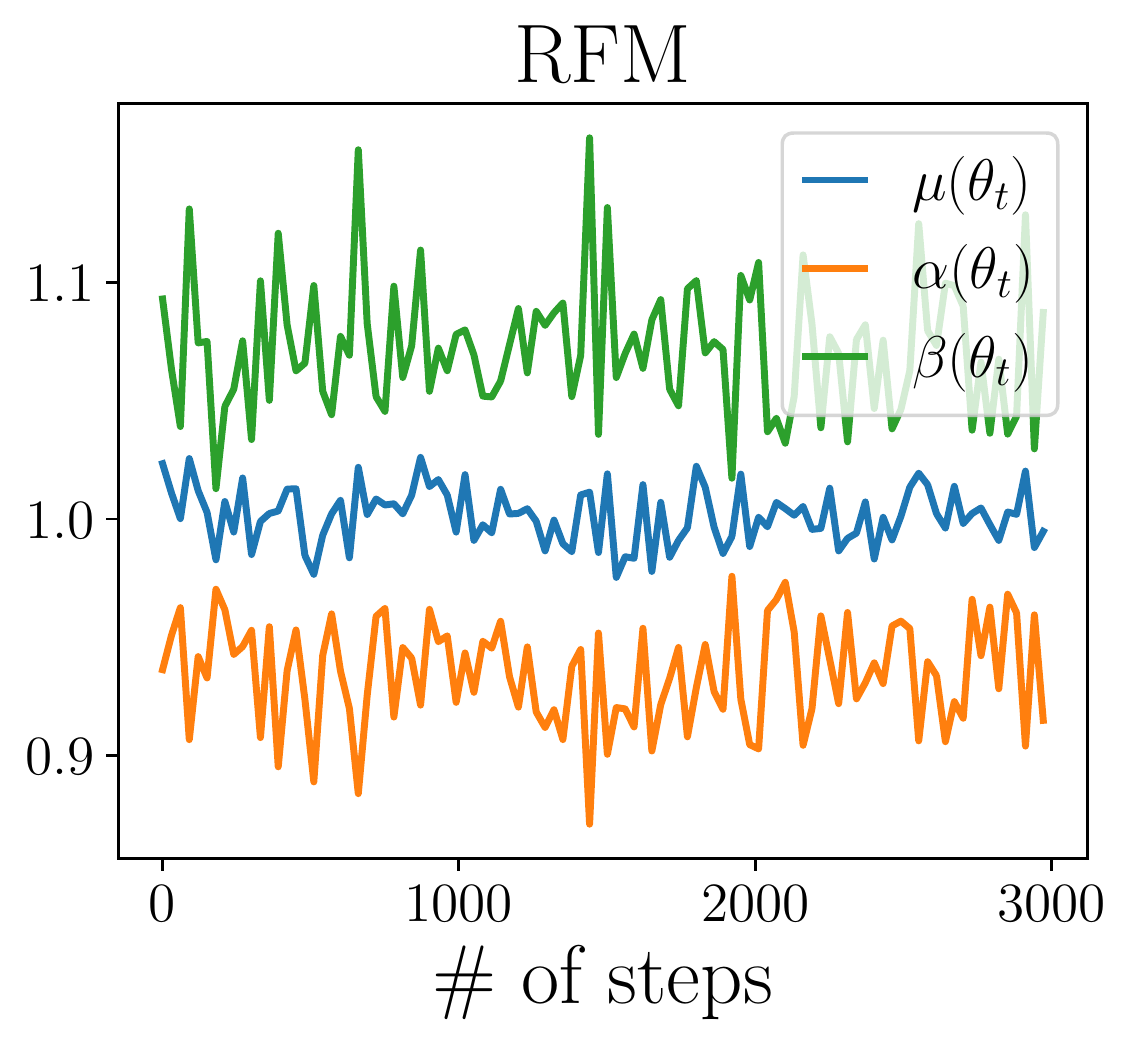}
\includegraphics[width=0.24\textwidth]{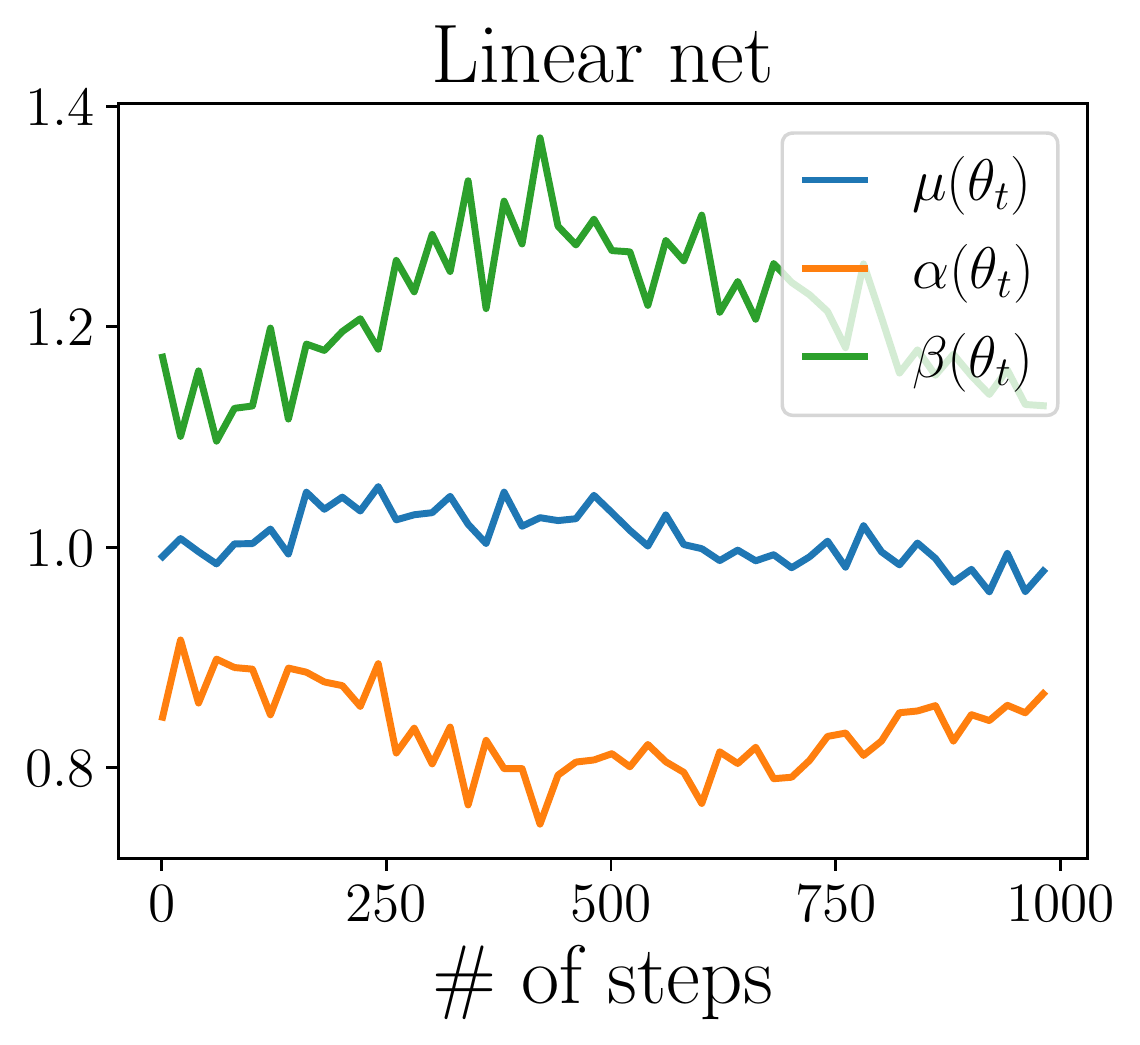}
\includegraphics[width=0.24\textwidth]{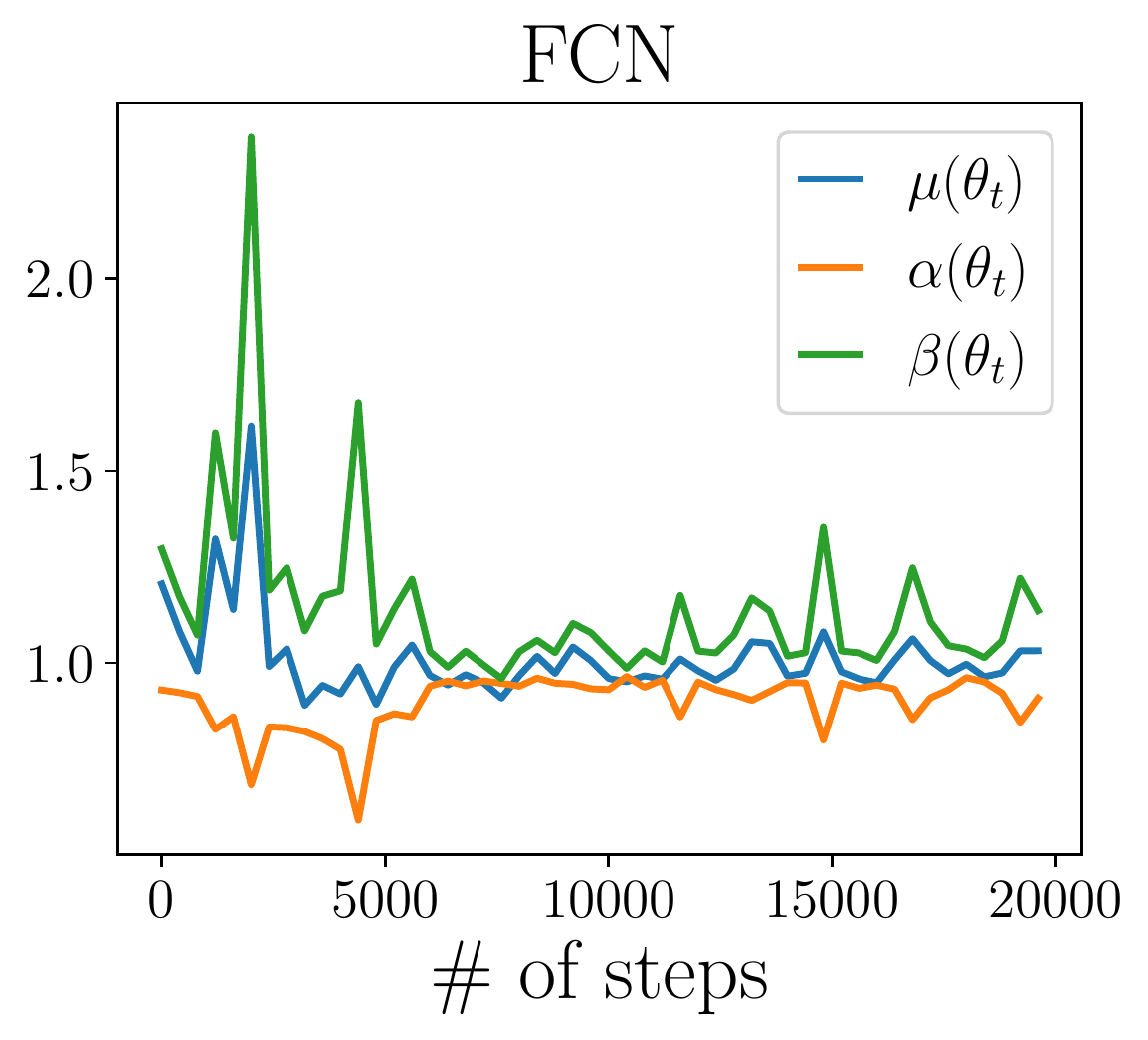}
\includegraphics[width=0.24\textwidth]{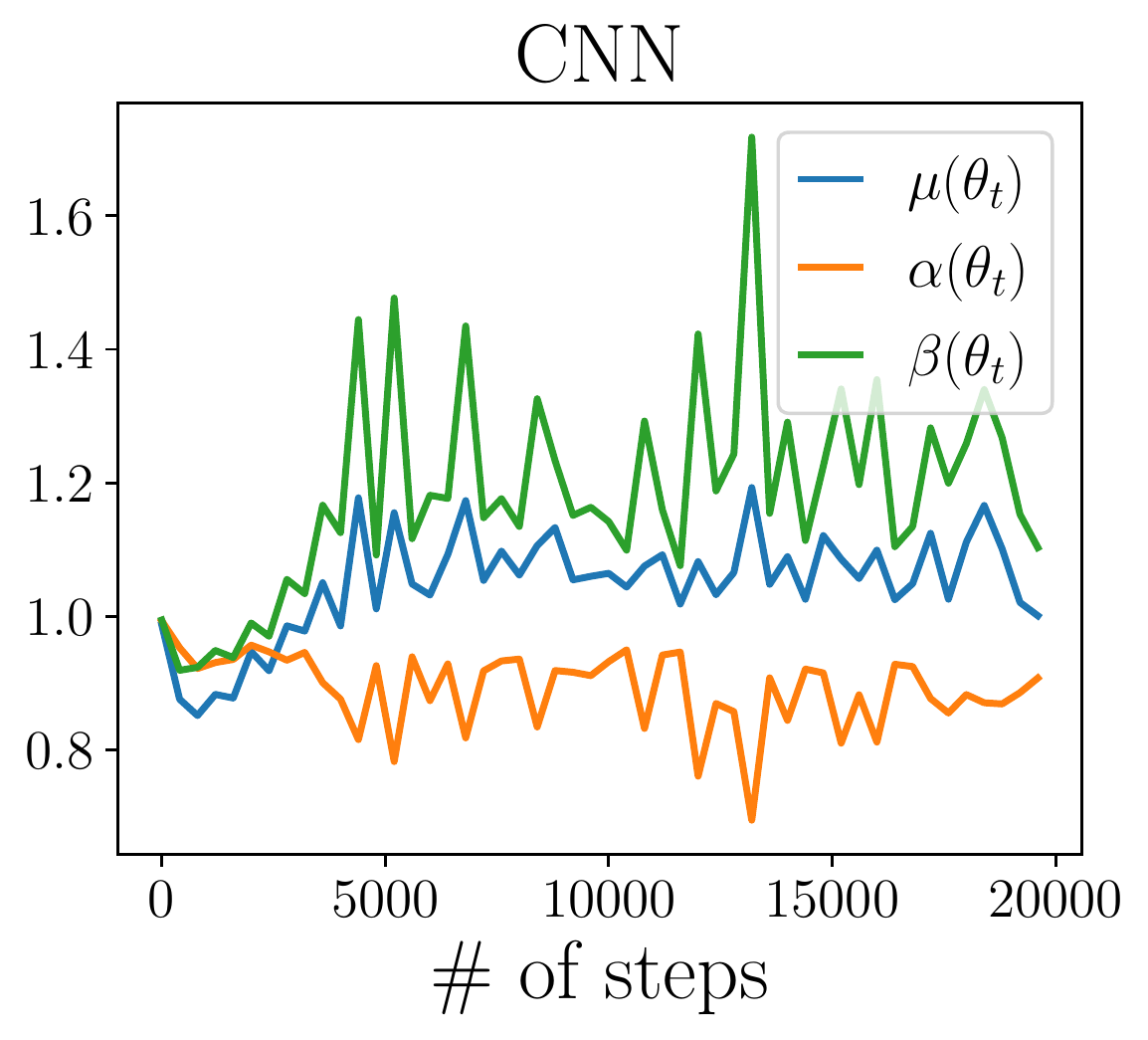}
}

\vspace*{-.5em}
\subfloat[]{\label{fig: alignment-modelsize}
\includegraphics[width=0.24\textwidth]{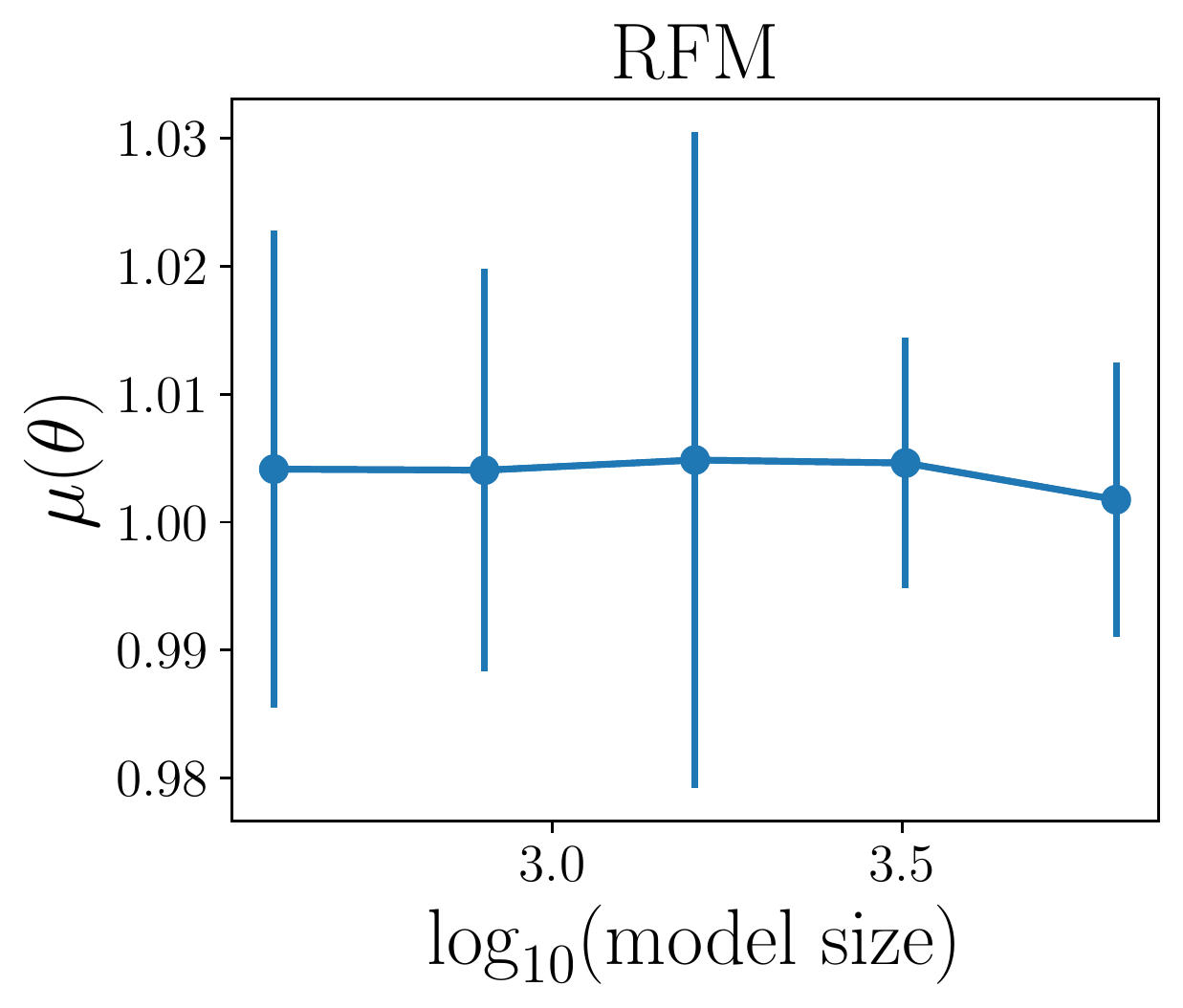}
\includegraphics[width=0.242\textwidth]{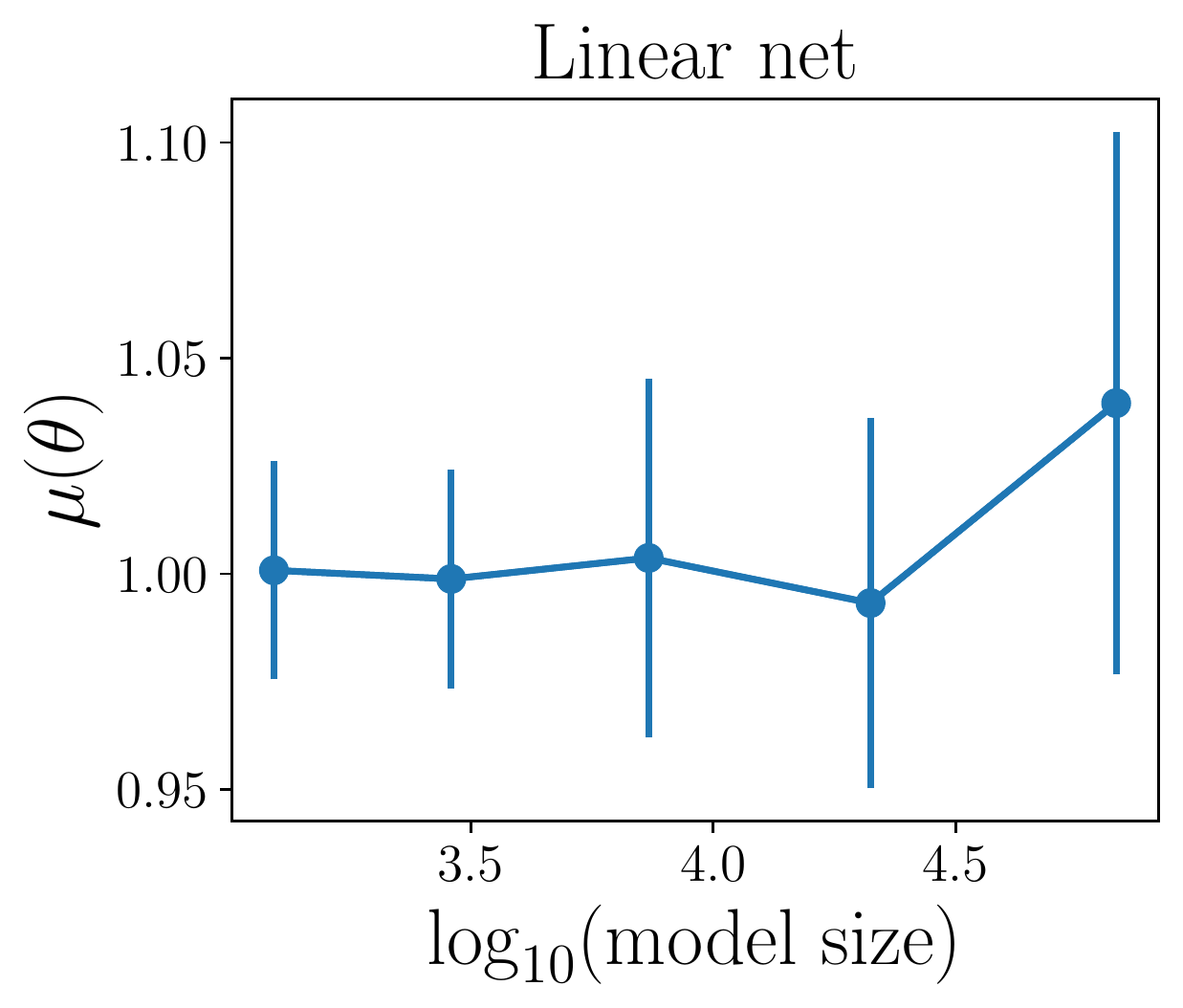}
  \includegraphics[width=0.23\textwidth]{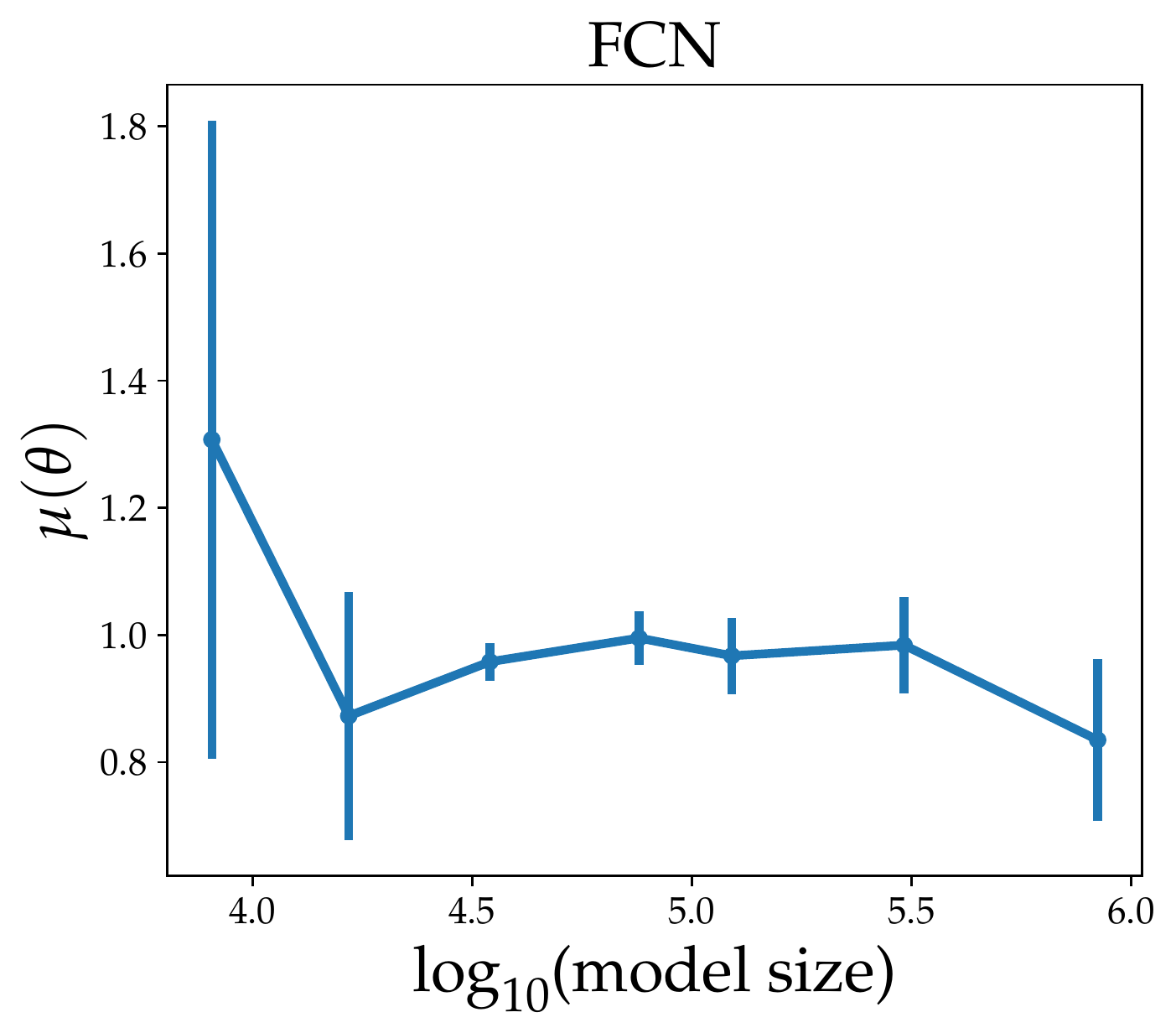}
\includegraphics[width=0.235\textwidth]{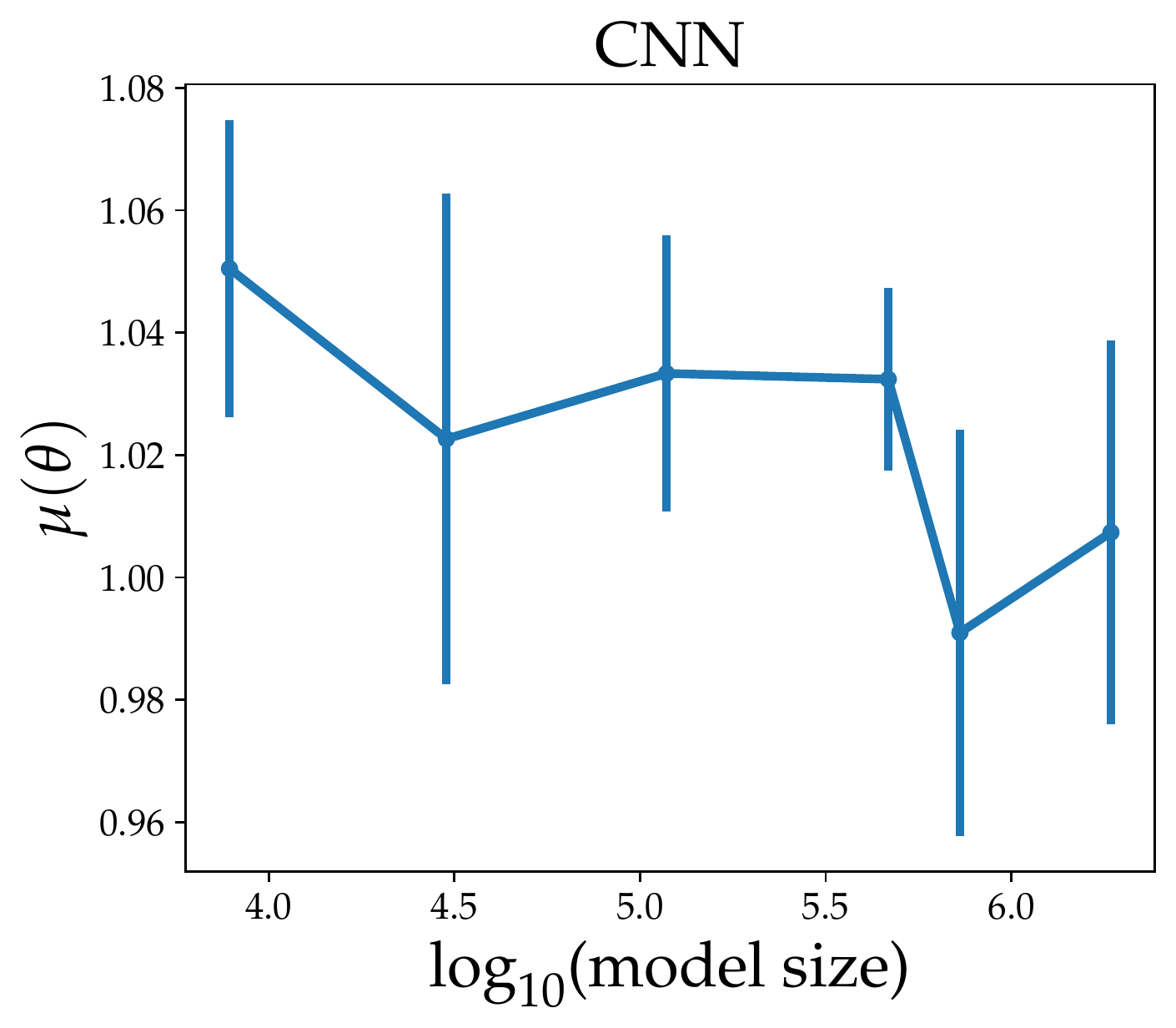}
}
\vspace*{-.5em}
\caption{\small \textbf{The alignment property of SGD noise}. Four types of models, including the RFM, linear network, fully-connected network (FCN), and convolutional neural network (CNN), are examined. We refer to Appendix \ref{sec: app-experiment-setup} for the experimental setup. Note that the linear network is trained in a low-sample regime ($n=100,d=50$).
(a) The alignment factors  during  training. (b) How the alignment strength of convergent solution changes with the over-parameterization. The error bar corresponds to the standard deviation over $5$ runs.
}
\label{fig: alignment-small-scale}
\vspace*{-1em}
\end{figure}

\paragraph*{The norm uniformity of model gradients.} Figure \ref{fig: feature-norms-small-scale} 
 shows the norm uniformity of model gradients, where we report the values of $\gamma(\theta)$ in the entire SGD trajectory. On can see that   $\gamma(\theta)$ is indeed bounded below, implying that the SGD noise  satisfies the $\mu_1$-alignment according to Lemma \ref{pro: alignment-linearized-nonlinear-model}. 

\begin{figure}[!h]
\centering

\includegraphics[width=0.24\textwidth]{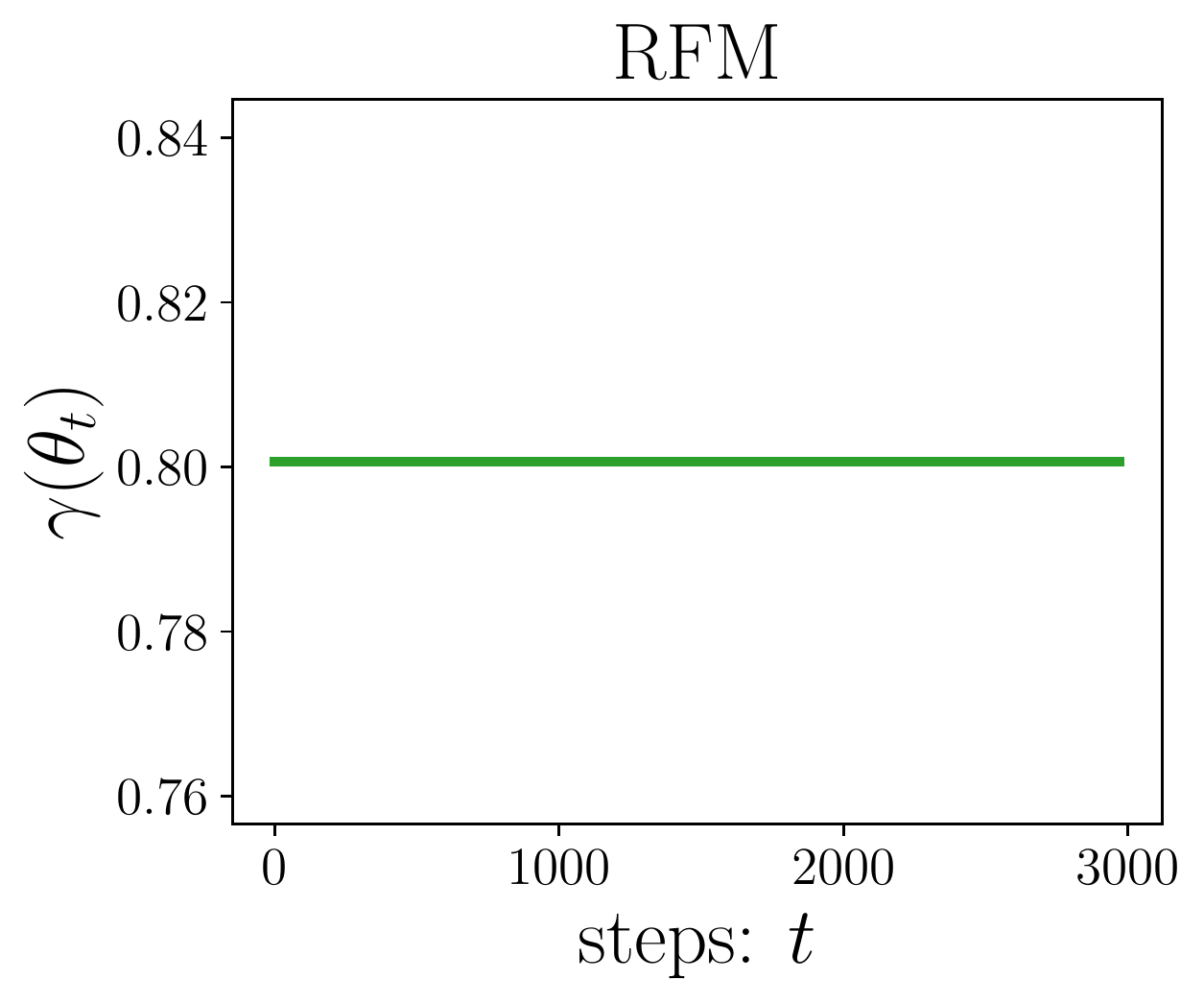}
\includegraphics[width=0.24\textwidth]{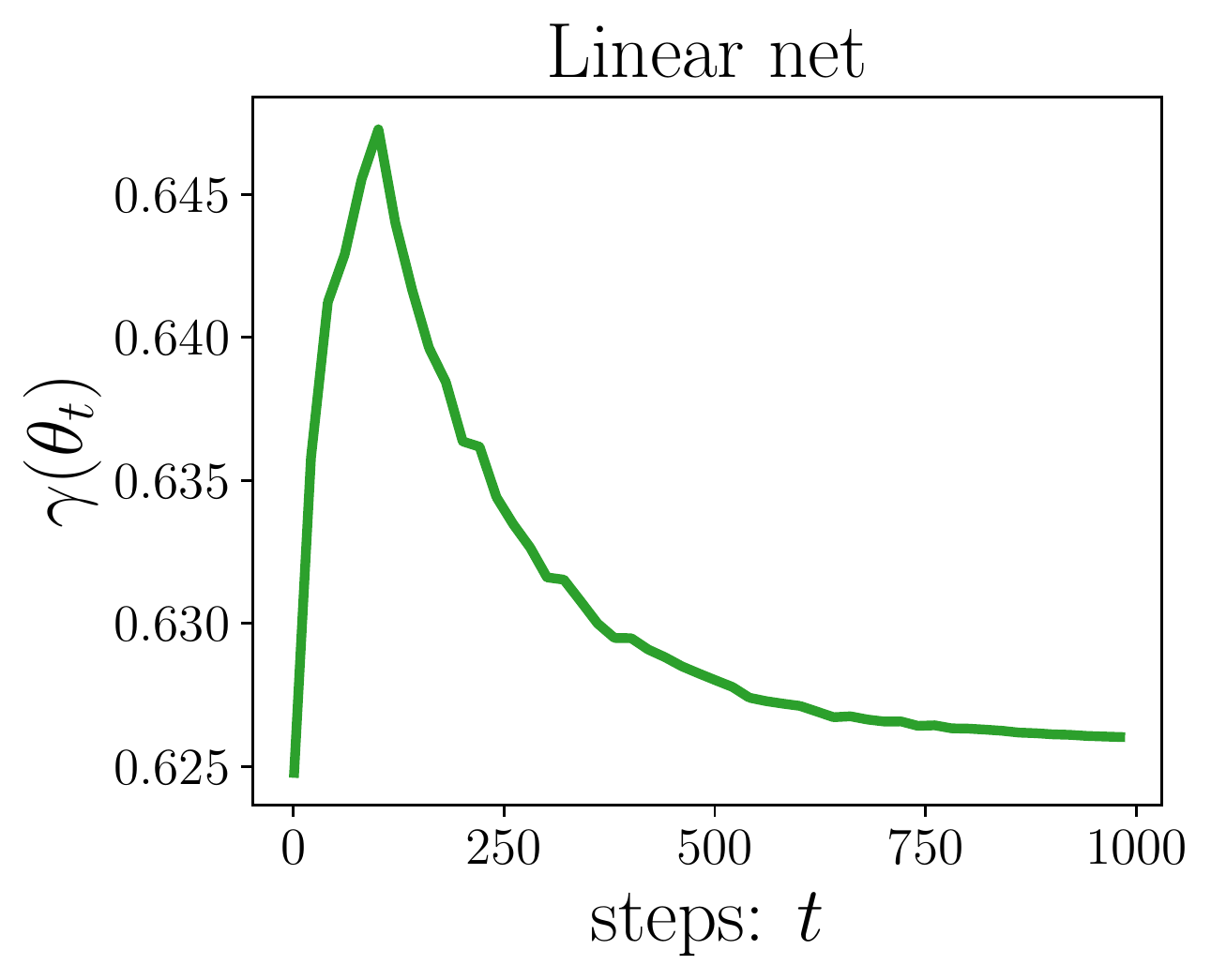}
\hspace*{-.4em}
\includegraphics[width=0.24\textwidth]{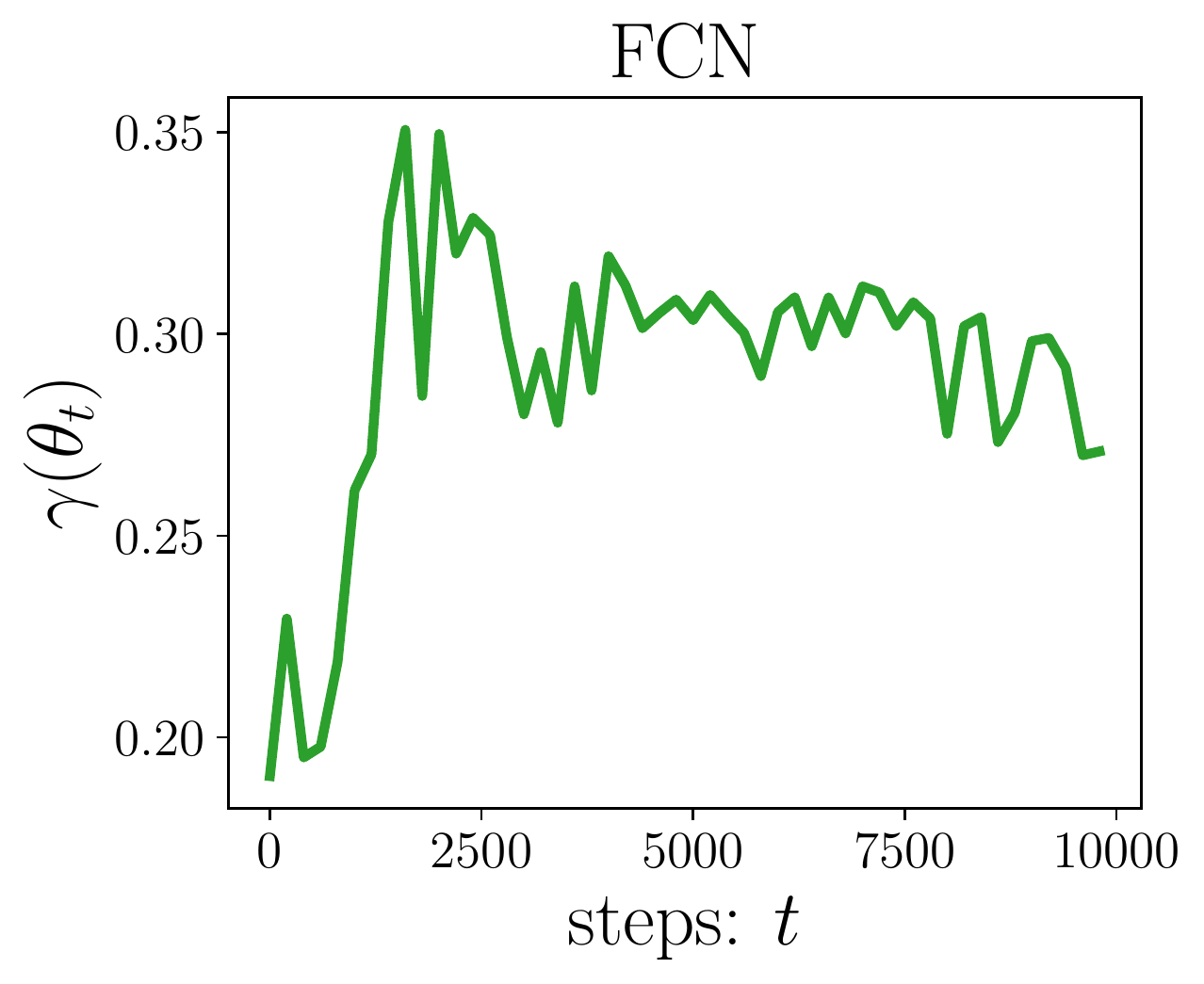}
\hspace*{-.4em}
\includegraphics[width=0.24\textwidth]{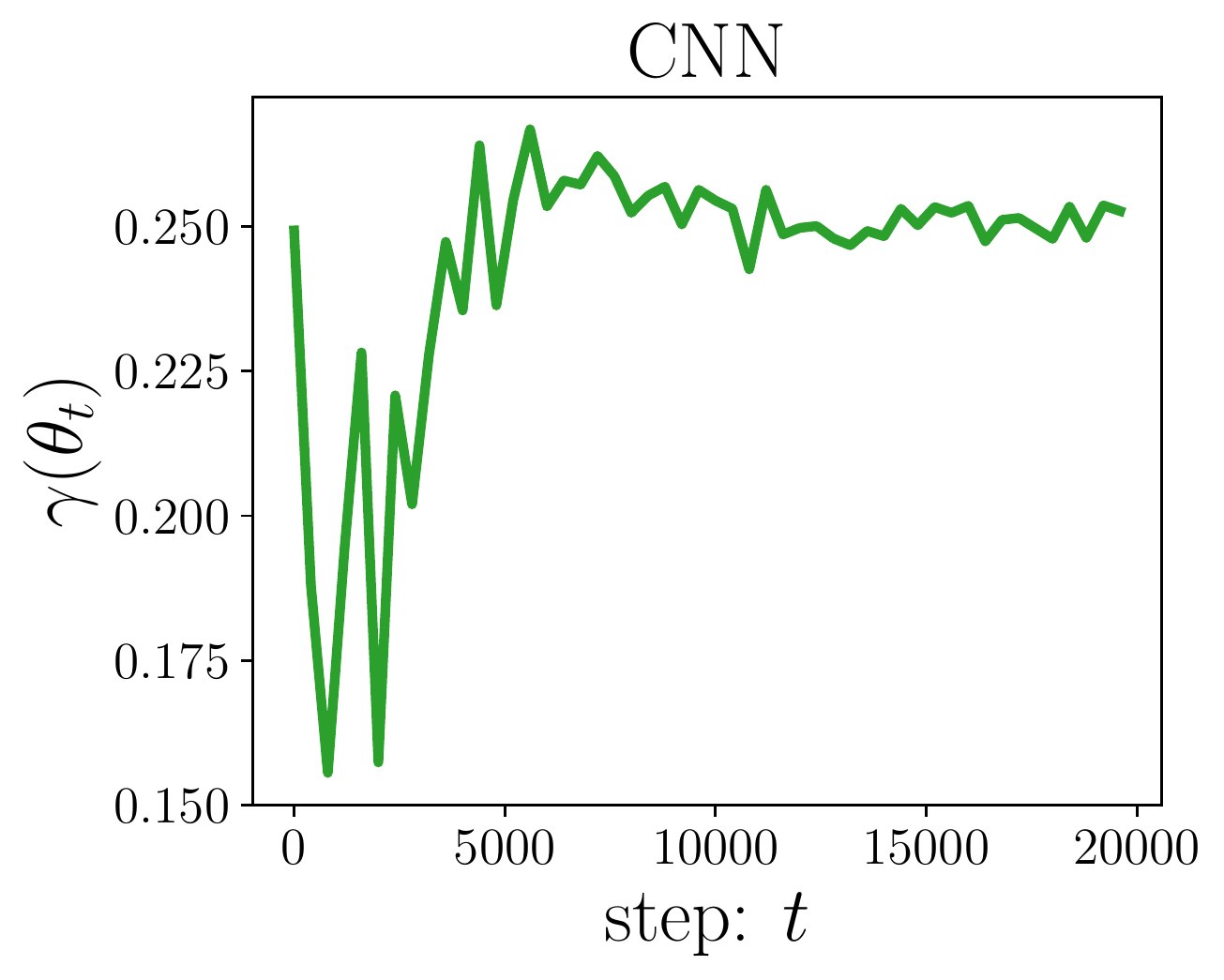}

\vspace*{-.5em}
\caption{\small \textbf{The norm uniformity of model gradients.} The values of $\gamma(\theta_t)$ in the entire SGD trajectory are reported. It is shown that the norms of model gradient norms are  uniform during the entire training process.
}
\label{fig: feature-norms-small-scale}
\vspace*{-1em}
\end{figure}

Note that in experiments, we only show that the alignment property  is satisfied in typical regions explored by SGD, including the random initialization and the convergent region. 
In contrast, 
for OLMs and RFMs, we in fact prove the alignment property  for the entire parameter space.  

\vspace*{-.5em}
\section{The linear stability analysis}
\vspace*{-.5em}
\label{sec: linear-stability}

Let $\theta^*$ be a global minimum of $L(\cdot)$. When $\theta_t$ is close to $\theta^*$, the local dynamical behavior of SGD  can be characterized by linearizing the dynamics around $\theta^*$:
\begin{equation}\label{eqn: linearized-SGD}
\begin{aligned}
\ttheta_{t+1} &= \ttheta_t - \frac{\eta}{B} \sum_{i\in I_t} \nabla^2 L_i (\theta^*) (\ttheta_t-\theta^*),
\end{aligned}
\end{equation}
where $\nabla^2 L_i (\theta^*) = \nabla f(x_i;\theta^*)\nabla f(x_i;\theta^*)^T$.  This corresponds to the local quadratic approximation of the loss $L(\cdot)$ or the local linearization of the model around $\theta^*$:
\begin{equation}\label{eqn: linearized-model}
  f_{\lin}(x;\theta) = f(x;\theta^*) + \langle \theta-\theta^*, \nabla f(x;\theta^*)\rangle.
\end{equation}
Specifically, \eqref{eqn: linearized-SGD} is exactly the SGD of training the linearized model \eqref{eqn: linearized-model}.

\begin{definition}[Linear stability]\label{def: linear-stability}
A global minimum $\theta^*$ is said to be \textit{linearly stable} if there exists a $C>0$ such that it holds for the linearized dynamics \eqref{eqn: linearized-SGD} that
$
\EE[\erisk(\ttheta_t)]\leq C \EE[\erisk(\ttheta_0)], \forall t\geq 0.
$
\end{definition}

\vspace*{-.5em}
It is well-known in dynamical system that the local behavior of the original nonlinear dynamics  can be characterized by the linearized one if the local quadratic approximation   is non-degenerate.
However, in over-parameterized case, the local quadratic approximation is degenerate in flat directions. Consequently, one may be concerned about the relevance of local quadratic approximation and the resulting linear stability analysis. Fortunately,  the stability in Definition \ref{def: linear-stability} is particularly measured with the change of loss value. Thus the instability mostly comes from noise perturbations in sharp directions and the alignment property guarantees that noise mostly   concentrates in sharp directions. Consequently, the flat directions contribute little to the instability. In sharp directions, the local quadratic approximation is always valid, thereby explaining the relevance of linear stability analysis. The rigorous formulation of this intuition is left to future work and  we instead resort to numerical experiments to demonstrate the validity in this paper.


For simplicity, we will use  $\theta_t$ to denote $\ttheta_t$; let $\theta^*=0$ and $g_i = \nabla f(x_i;\theta^*)$. For the linearized model $f_{\lin}(\cdot;\theta)$, we have 
$
  \erisk(\theta) = \frac{1}{2n}\sumin |\theta^Tg_i|^2=\half\theta^T H\theta,  G=H=\fn\sumin g_ig_i^T,
$
where we omit the dependence on $\theta^*$ for simplicity. Note that the Fisher and Hessian matrix  are constant but the noise covariance  
$
\Sigma(\theta) = \fn\sumin |g_i^T\theta|^2 g_ig_i^T- H\theta\theta^TH
$
is still state-dependent.

Before considering the specific linearized SGD \eqref{eqn: linearized-SGD}, we first have a general result.
\begin{lemma}\label{lemma: loss-update}
Consider a general SGD:
$
\theta_{t+1} = \theta_t - \eta (\nabla\erisk(\theta_t) + \xi_t)
$
for the linearized model \eqref{eqn: linearized-model}, where $(\xi_t)_{t\geq 1}$ are any noises satisfying $\EE[\xi_t]=0, \EE[\xi_t\xi_t^T]=S(\theta_t)$. Then we have 
\begin{equation}\label{eqn: loss-update}
\EE[\erisk(\theta_{t+1})]=\EE[r(\theta_t)\erisk(\theta_t)+\eta^2 v(\theta_t)],
\end{equation}
where 
$
  \nu(\theta)=\tr(H S(\theta))/2
$
and $r(\theta)\geq 0$. Moreover, if $\eta\leq 2/\lambda_{1}(H)$, then $r(\theta)\leq 1$.
\end{lemma}

\begin{proof}
Using the fact  $\EE[\xi_t]=0$ and $\EE[\xi_t\xi_t^T]=S(\theta_t)$, we have
\begin{equation}
\begin{aligned}\label{eqn: loss-update-1}
\EE[\erisk(\theta_{t+1})] &= \EE[\half (\theta_t-\eta \nabla \erisk(\theta_t) +\eta\xi_t)^T H (\theta_t -\eta \nabla \erisk(\theta_t)+ \eta \xi_t)]\\
&= \EE[\erisk(\theta_t)-\eta \nabla \erisk(\theta)^TH \theta_t +\frac{\eta^2}{2} \nabla \erisk(\theta_t)^TH\nabla \erisk(\theta_t)]+\frac{\eta^2}{2}\EE[\tr(H S(\theta_t))]\\
&=\EE[r(\theta_t)L(\theta_t)+\eta^2\nu(\theta_t)],
\end{aligned}
\end{equation}
where $r(\theta)=1-2\eta \frac{\theta^T H^2\theta}{\theta^T H\theta}+\eta^2 \frac{\theta^T H^3 \theta}{\theta^TH\theta}$ since $\nabla \erisk(\theta)=H\theta$. By Lemma \ref{lemma: GD-factor}, $r(\theta)\geq 0$ and if $\eta\leq 2/\lambda_1(H)$, then $r(\theta)\leq 1$.
\end{proof}

\vspace*{-.7em}
The two terms $r(\theta_t)L(\theta_t)$ and $\eta^2 v(\theta_t)$ denote the contributions from the full-batch gradient $\nabla \erisk(\theta_t)$ and the noise $\xi_t$, respectively. The stability is affected by both terms simultaneously.  It is well-known that if $\theta^*$ is linearly stable for GD, then
$
  \lambda_1(H(\theta^*))\leq 2/\eta
$ (see, e.g., \cite{wu2018sgd,mulayoff2021implicit}).
This also holds for SGD but SGD imposes a stricter  condition because of the extra $\eta^2 \nu(\theta_t)$ term. Specifically, 
\begin{equation}\label{eqn: loss-update-3}
\EE[\erisk(\theta_{t+1})] =\EE[r(\theta_t)\erisk(\theta_{t})+\eta^2\nu(\theta_t)]\geq \eta^2\EE[\nu(\theta_t)]=0.5\eta^2 \tr(HS(\theta)).
\end{equation}
Therefore, the more $S(\theta)$ aligns with  $H$, the more unstable  that minimum is. Specifically, let $S(\theta)=2\sigma^2 L(\theta)C(\theta)$. Then, 
$
	\EE[\erisk(\theta_{t+1})]\geq \eta^2 \sigma^2 \EE[L(\theta_t)\tr(HC(\theta))].
$
Thus to ensure a stable convergence, we should roughly have 
$
	\tr(HC(\theta))\leq \frac{1}{\sigma^2\eta^2}.
$
We next show that this can lead to a flatness control by utilizing the alignment between $C(\theta)$ and $H$.


\vspace*{-.2em}
\subsection{The linear stability imposes size-independent flatness constraints}
\vspace*{-.5em}

For the mini-batch SGD, the following theorem characterizes how the batch size and learning rate affect the flatness---as measured by the Frobenius norm of Hessian---of minima accessible to SGD.

\begin{theorem}\label{thm: Fro-norm-Hessian-bound}
Let $\theta^*$ be a global minimum that is linearly stable.  Denote by $\mu(\theta)$ the alignment factors for the linearized SGD \eqref{eqn: linearized-SGD} and model \eqref{eqn: linearized-model}. If $\mu(\theta)\geq \bmu_0$, then 
$
\|H(\theta^*)\|_F\leq \frac{1}{\eta}\sqrt{\frac{B}{\mu_0}}.
$
\end{theorem}
\vspace*{-1em}
\begin{proof}
By \eqref{eqn: loss-update-3} and the definition of $\mu(\theta)$, we have 
\begin{align}\label{eqn: loss-update-2}
\EE[\erisk(\theta_{t+1})]\geq \frac{\eta^2}{2B}\EE[\tr(H\Sigma(\theta_t))]\geq\frac{\eta^2\|H\|_F^2}{B}\EE[\mu(\theta_t)\erisk(\theta_t)]\geq \frac{\bmu_0 \eta^2\|H\|_F^2}{B}\EE[\erisk(\theta_t)].
\end{align}
To ensure the stability, we must have $\bmu_0\eta^2\|H\|_F^2/B\leq 1$, leading to $\|H\|_F\leq \sqrt{B/\bmu_0}/\eta$.
\end{proof}

We have shown in Section \ref{sec: alignment} that $\mu_0$ is (nearly) size-independent, and thus the obtained upper bound of flatness is also (nearly) size-independent.
As a comparison, for GD, the linear stability can only ensure $\lambda_{\max}(H(\theta^*))\leq 2/\eta$. This gives a bound of the Hessian's Frobenius norm: $\|H(\theta^*)\|_F\leq 2\sqrt{p}/\eta$, depending on the model size explicitly. The comparison of two bounds partially explains why SGD tends to select flatter minima than GD.

We show below that the $\mu$-alignment can be further relaxed to the $\mu_1$-alignment. The proof is  similar to the one of Theorem \ref{thm: Fro-norm-Hessian-bound} and deferred to   Appendix \ref{sec: proof-linear-stability}.
\begin{proposition}\label{pro: flatness-control-mu1-alignment}
Under the setting of Theorem \ref{thm: Fro-norm-Hessian-bound}, if the noise of linearized SGD satisfies
$\mu_1(\theta)\geq \bmu_1$, then 
$
  \|H(\theta^*)\|_F\leq \min\left(\frac{B}{\sqrt{(B-1)\bmu_1}},\frac{2B}{\bmu_1}\right)\frac{1}{\eta}.
$
\end{proposition}

When $B\gg 1$, the bound becomes
$
B/(\eta\sqrt{(B-1)\bmu_1})\approx \sqrt{B/\bmu_1}/\eta
$
, which is the same as the case of $\mu$-alignment.  Thus the influence of $\nabla L$ is indeed negligible compared with  $\{\nabla L_i\}_i$.

Note that  the linear stability is local in nature and hence our analysis essentially only needs  the $\mu_1$-alignment to hold locally around minima of interest. Experiments in Figure \ref{fig: feature-norms-small-scale} shows that $\gamma(\theta^*)$ is always bounded below, i.e., the norm uniformity of model gradients is satisfied. Combining with Proposition \ref{pro: rfm-alignment}, we can conclude that the alignment property assumed in Proposition \ref{pro: flatness-control-mu1-alignment}   holds.

\paragraph*{Tightness of our analysis.}
In the analysis above, we only inspect the instability caused by the noise, with the full-batch gradient completely ignored. Therefore, we anticipate that our bound is tighter in small-batch regime, where the noise term dominates the full-batch term. We will see that our numerical experiments indeed confirm this tightness in small-scale regime.
However, for obtaining the tightest bound,   one may need to consider both components simultaneously; this is much more complicated and left to future work.

\paragraph*{Numerical validations.}
Figure \ref{fig: size-independence-a} numerically shows how the Frobenius norm of Hessian (not only the upper bound) changes with extent of over-parameterization, where the trace of Hessian is also plotted for comparison. One can see that the Frobenius norm indeed keeps almost unchanged as increasing the model size but the Hessian trace increases significantly. Figure \ref{fig: size-independence-b} further shows the ratio between the Frobenius norm and our upper bound in the training process and two  batch sizes are examined. We have two observations. First, the correctness of our bound holds for the entire SGD trajectory, suggesting that the linear stability analysis is relevant for the entire training process. Second, as expected, the theoretical bound is indeed  tighter for the case with a smaller batch size.


\begin{figure}[!h]
\centering
\captionsetup[subfloat]{farskip=.3pt,captionskip=0.2pt}
\subfloat[]{\label{fig: size-independence-a}
\includegraphics[width=0.23\textwidth]{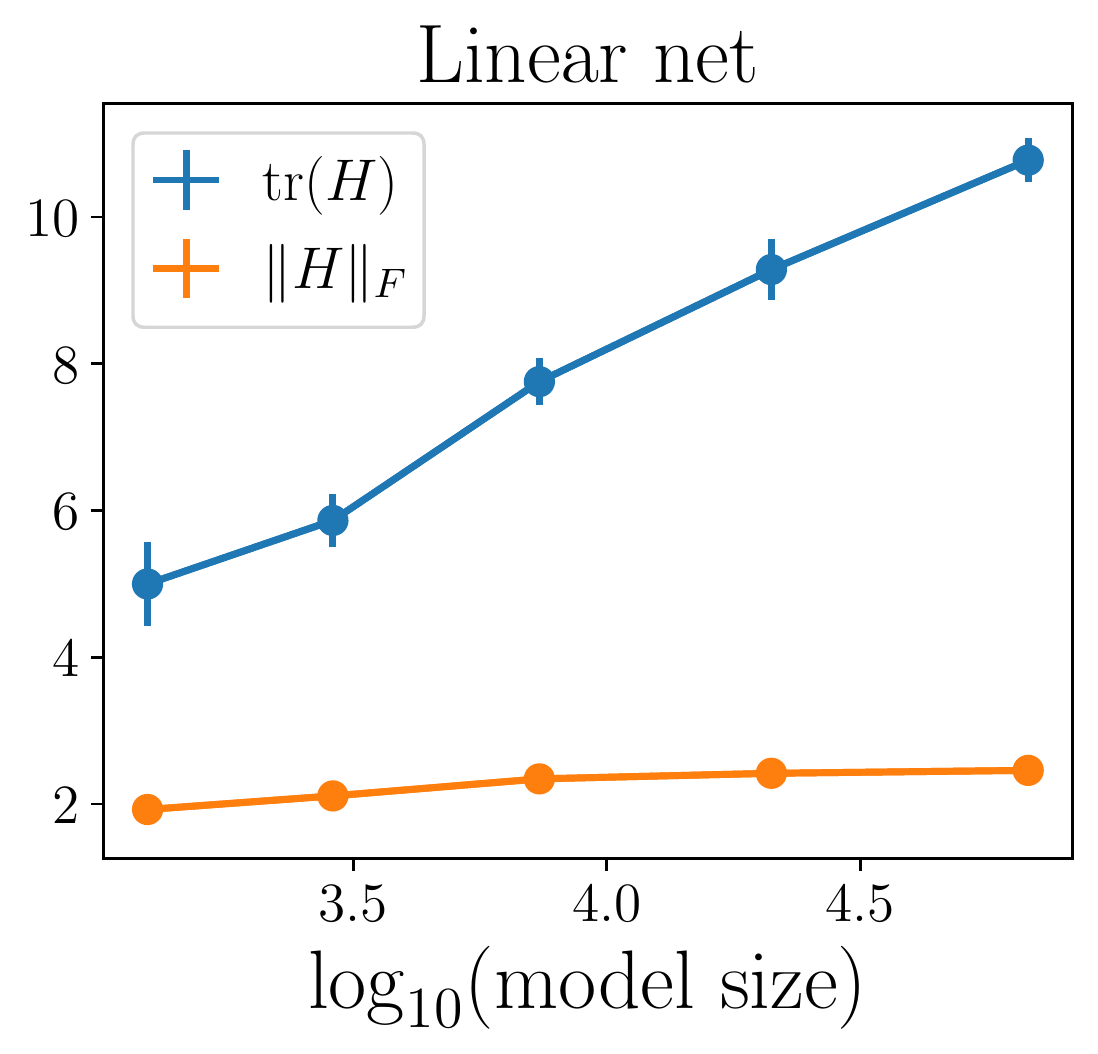}
\includegraphics[width=0.23\textwidth]{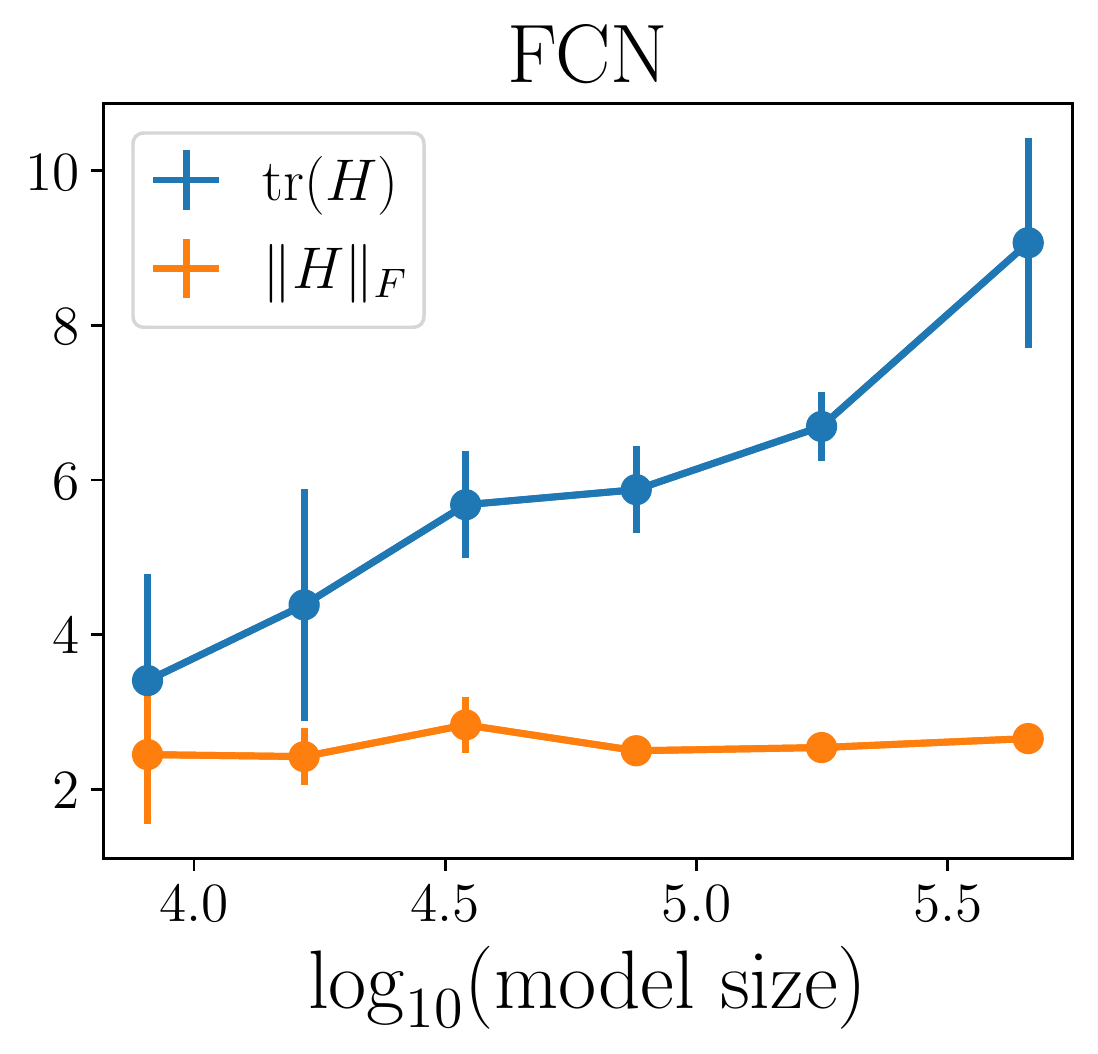}
\includegraphics[width=0.23\textwidth]{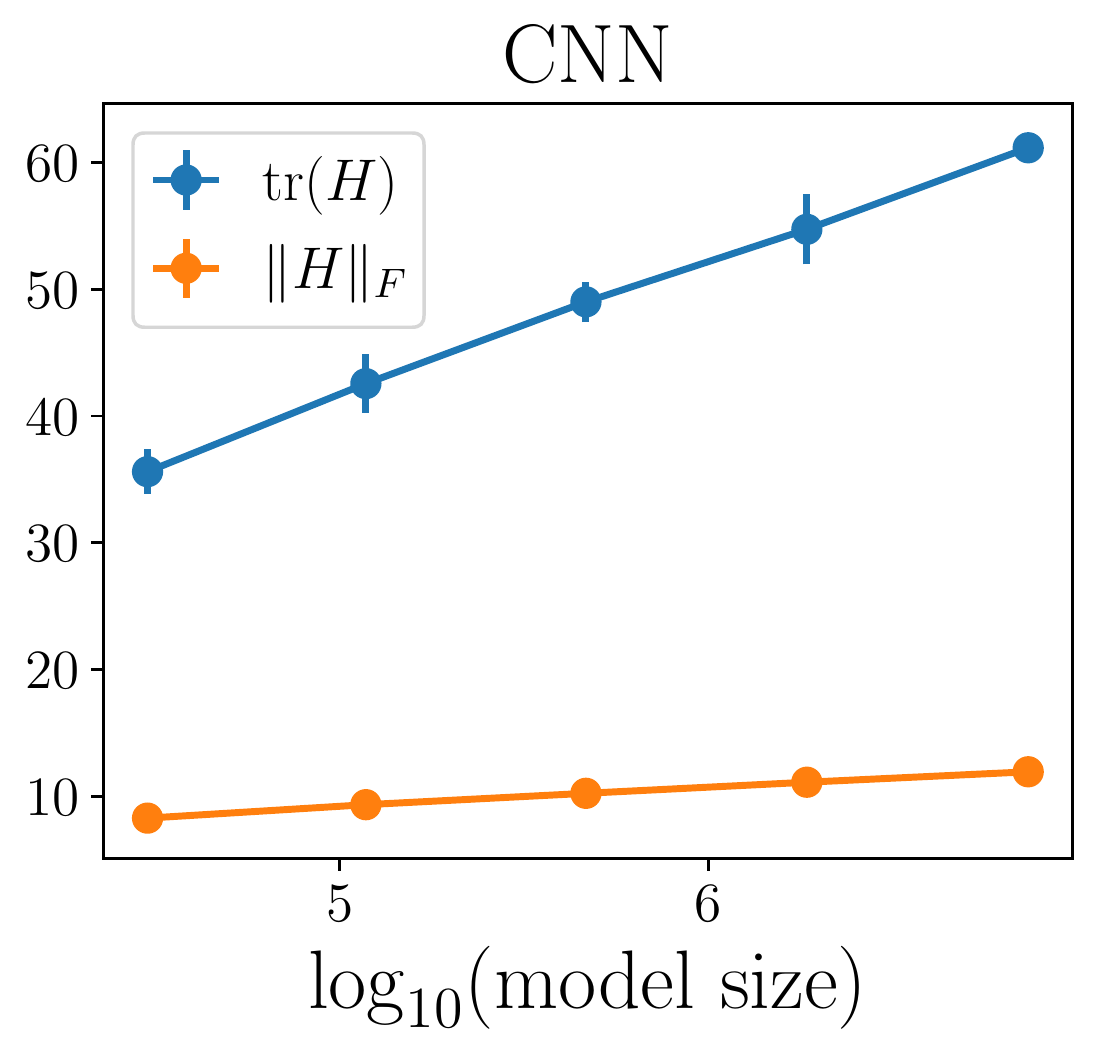}
}
\subfloat[]{\label{fig: size-independence-b}
  \includegraphics[width=0.27\textwidth]{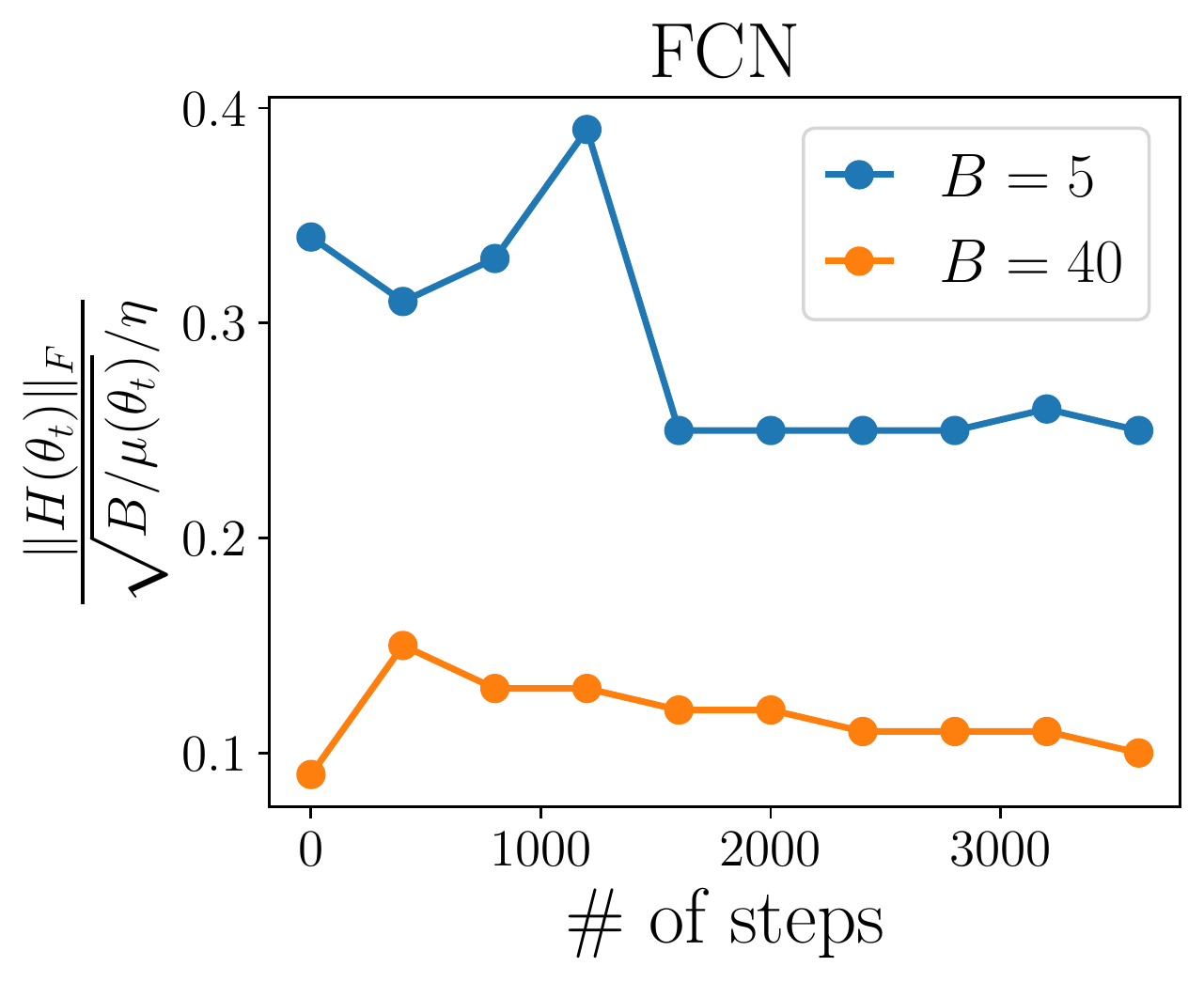}
}
\vspace*{-.4em}
\caption{\small (a) The Frobenius norm and trace of Hessian vs. model size. The error bar corresponds to the standard deviation estimated over $5$ runs. (b) The ratio between the Frobenius norm and our upper bound in the training process. Here $B=5$ and $B=40$ are examined.  For more experiment details, we refer to Appendix \ref{sec: app-experiment-setup}.}
\label{fig: size-independence}
\vspace*{-1em}
\end{figure}

\subsection{SGD escapes from sharp minima exponentially fast}
\label{sec: escape}
 The following theorem shows that the pure noise-driven escape from a sharp minimum is \emph{exponentially fast}, whose proof  follows trivially from \eqref{eqn: loss-update-2}.
\begin{theorem}\label{thm: escape}
Under the setting of Theorem \ref{thm: Fro-norm-Hessian-bound},
if $\|H(\theta^*)\|_F>\frac{1}{\eta}\sqrt{\frac{B}{\bmu_0}}$, then the linearized SGD satisfies $\EE[\erisk(\theta_t)]\geq \gamma_0^t\EE[\erisk(\theta_0)]$ with $\gamma_0=\frac{\eta^2\bmu_0}{B}\|H(\theta^*)\|_F^2>1$.
\end{theorem}
Hence, linearized SGD takes roughly $\log_{\gamma_0}(1/\varepsilon)$ steps to escape from a $O(\varepsilon)$-loss region to a $O(1)$-loss region. The escape time depends on the loss barrier only logarithmically and is independent of the parameter space dimension. Due to the local closeness between linearized SGD and the original SGD, this partially explains the unreasonable escape efficiency of SGD for training big models. In contrast, the escape rates of existing works \cite{xie2020diffusion,zhu2019anisotropic,zhou2020towards,mori2021logarithmic} are either exponentially slow with respect to the loss barrier or suffer from the curse of dimensionality.

Figure \ref{fig: escape} shows the trajectories of  SGD escaping from sharp minima. It is demonstrated that the escape is indeed exponentially fast and specifically, $10$ steps are enough for SGD  escaping to a high-loss region for all the models examined. In addition, we observe that the escape is still exponentially fast in the high-loss region, although our analysis only applies to a local region. How can we explain this nonlocal escape behavior? We leave the study of this interesting phenomenon to future work.


\begin{figure}[!h]
\centering
\includegraphics[width=0.23\textwidth]{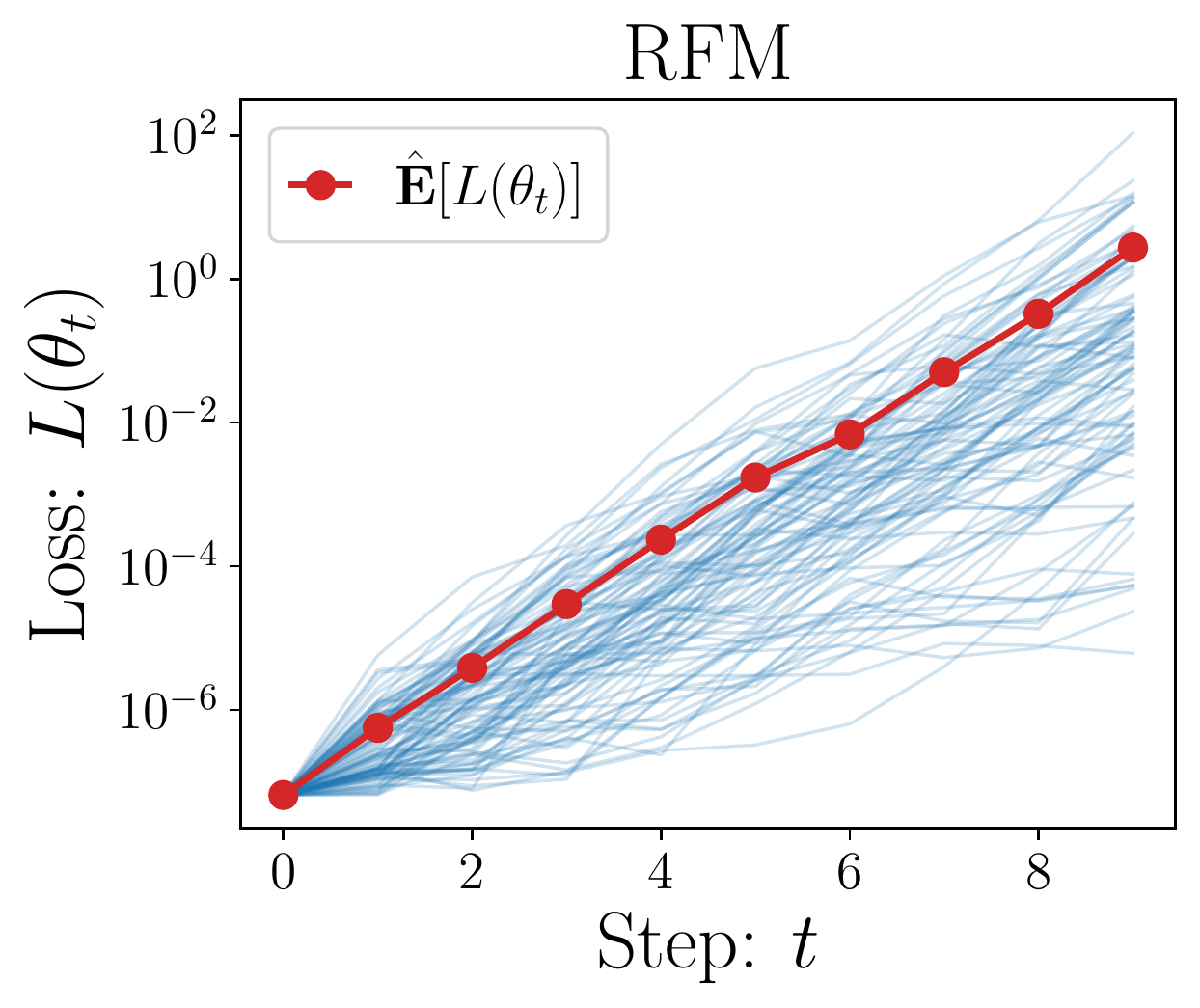}
\includegraphics[width=0.23\textwidth]{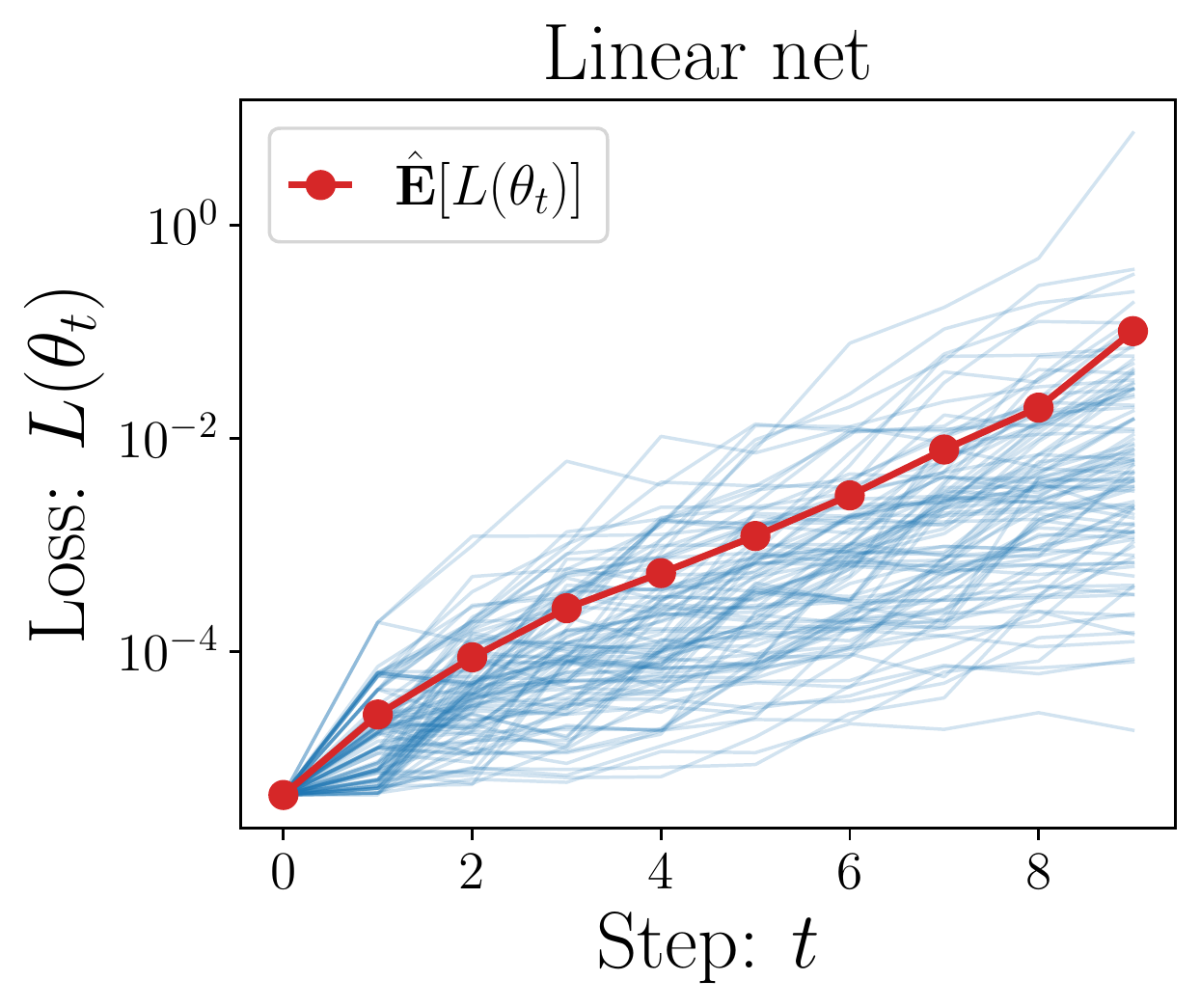}
\includegraphics[width=0.23\textwidth]{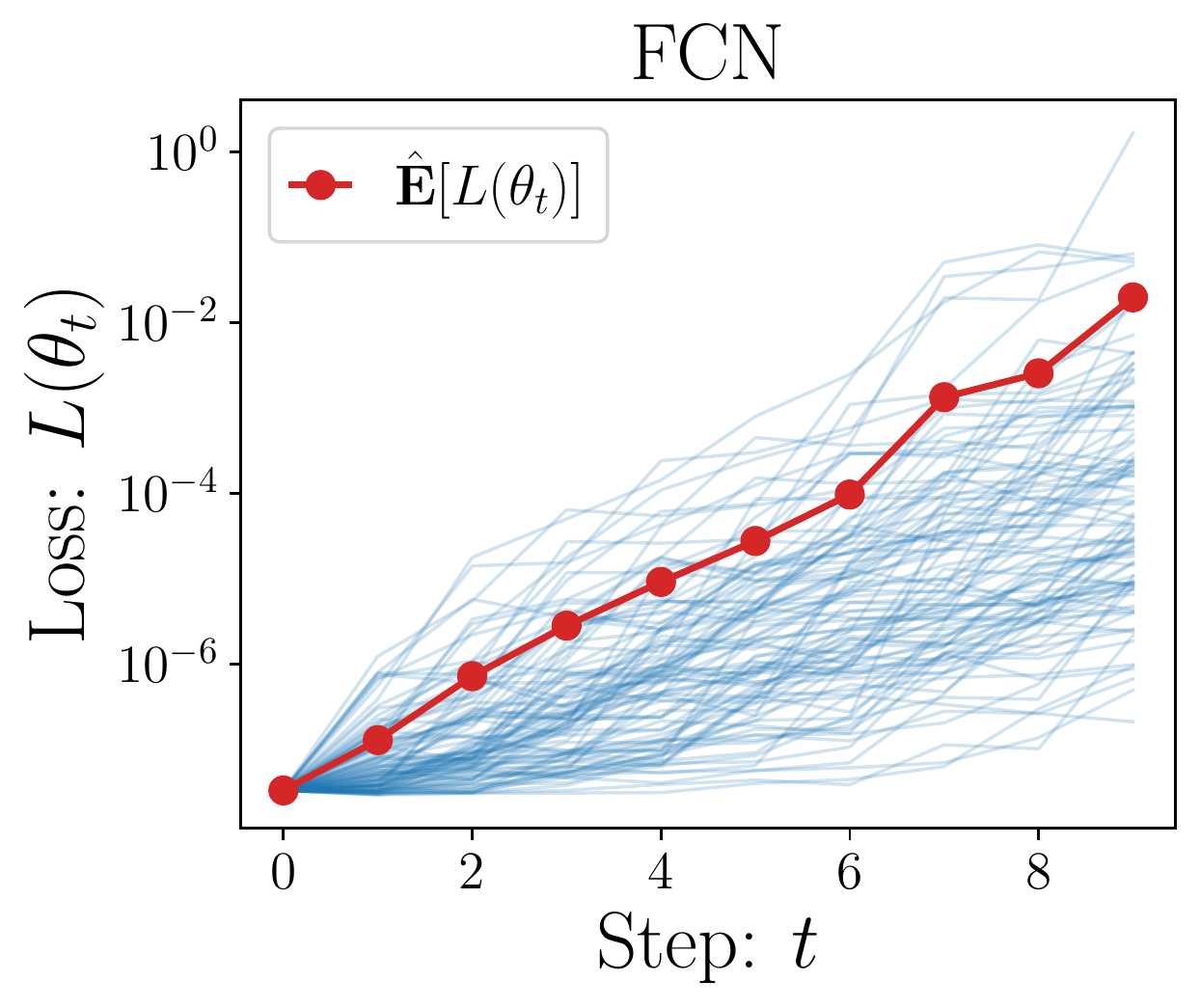}
\includegraphics[width=0.23\textwidth]{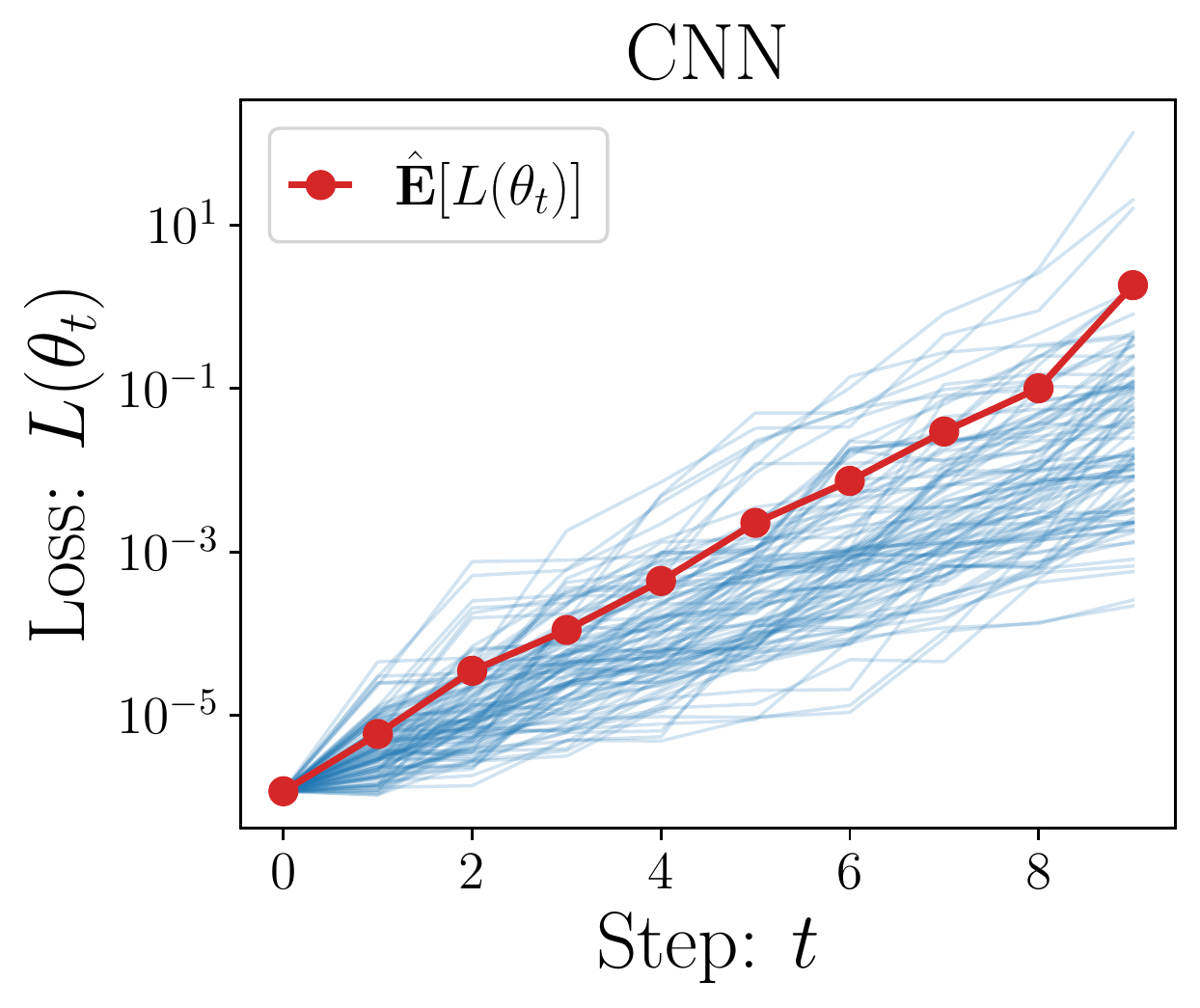}
\vspace*{-.5em}
\caption{\small \textbf{The exponentially fast escape from sharp minima.} The blue curves are  $200$ trajectories of SGD; the red curve  corresponds to the average. The sharp minimum is found by GD. When GD nearly converge,  we switch to SGD with the same learning rate. This choice ensures that the minimum is stable for GD, and thus the escape is purely driven by SGD noise. For more experimental details, we refer to Appendix \ref{sec: app-experiment-setup}.
}
\label{fig: escape}
\vspace*{-1em}
\end{figure}

\subsection{The importance of the noise structure}
\vspace*{-.4em}

\paragraph*{The magnitude structure.}
Theorem \ref{thm: escape} together with its proof suggests that the loss dependence of noise magnitude  is critical for obtaining the exponentially fast escape. The intuition is as follows. When $\theta_t$ is perturbed by noise to $\theta_{t+1}$ where $\erisk(\theta_{t+1})>\erisk(\theta_t)$, the noise magnitude  becomes larger there and thus $\theta_{t+1}$ is easier to be perturbed to a larger-loss region. This positive feedback drives SGD to leave exponentially fast. On the contrary, the  following lemma shows that if the noise is uniformly bounded, the noise-driven escape is at most linear in time. 
\vspace*{-.1em}
\begin{lemma}
Under the setting of Lemma \ref{lemma: loss-update}, assume $\eta\leq 2/\lambda_1(H)$ and $\EE[HS(\theta)]\leq 2\sigma^2$. Then
$
  \EE[\erisk(\theta_{t})-\erisk(\theta_0)]\leq \eta^2\sigma^2t.
$
\end{lemma}
We set $\eta\leq 2/\lambda_1(H)$ to avoid the exponential escape caused by the full-batch component.
\vspace*{-.5em}

\begin{proof}
By Lemma \ref{lemma: GD-factor}, when $\eta\leq 2/\lambda_1(H)$, $r(\theta)\leq 1$. Thus Lemma \ref{lemma: loss-update} implies $\EE[\erisk(\theta_{t+1})]\leq \EE[\erisk(\theta_t)]+\eta^2 \sigma^2$, which implies $\EE[\erisk(\theta_t)]\leq \EE[\erisk(\theta_0)]+\eta^2\sigma^2t$.
\end{proof}
\vspace*{-.5em}

\paragraph*{The direction structure.}
We now turn to consider the impact of direction structure. Consider general SGDs: $\theta_{t+1}=\theta_t-\eta (\nabla \erisk(\theta_t)+\xi_t)$ with $\EE[\xi_t\xi_t]=S(\theta_t)/B$ for the linearized model \eqref{eqn: linearized-model}. We compare  two type of (unrealistic) noises: 
\begin{itemize}
\item \textit{Geometry-aware noise:} $S_1(\theta)=2\erisk(\theta)H$. 
\item \textit{Isotropic noise:} $S_2(\theta)=2\sigma^2 \erisk(\theta) I_p$ with $\sigma^2=\tr(H)/p$. Here, the value of $\sigma^2$ is chosen to ensure two types of noises have the same total variance  for a fair comparison \cite{zhu2019anisotropic}. 
\end{itemize}
Note that $p$ denotes the model size.
Analogous to Theorem \ref{thm: Fro-norm-Hessian-bound}, for the second isotropic SGD, 
\[
\EE[\erisk(\theta_{t+1})]\geq \frac{\eta^2}{2B}\tr(HS(\theta))\geq \EE[\erisk(\theta_t)] \frac{\sigma^2\eta^2}{B} \tr(H)=  \EE[\erisk(\theta_t)] \frac{\eta^2 }{pB} \tr(H)^2.
\]
Hence, the instability decreases with the parameter-space dimension and the resulting flatness constraint is  $\tr(H(\theta^*))\leq \sqrt{pB}/\eta$, depending on the model size explicitly.  In contrast, for the fist noise, Theorem \ref{thm: Fro-norm-Hessian-bound} implies $\|H(\theta^*)\|_F\leq \sqrt{B}/\eta$, independent of the model size. This difference can be intuitively explained as follows. The isotropic noise wastes most energy in perturbing SGD along flat directions, which barely affects the instability. 
In contrast, the geometry-aware noise focuses  most energy on perturbing SGD along sharp directions, causing much more instability.

\section{Larger-scale experiments}
\label{sec: large-scale-experiment}
\vspace*{-.5em}

We have provided small-scale experiments to justify the validity of our theoretical findings  for a variety of ML models. 
Here we turn to demonstrate the practical relevance  by consider the classification of the  CIFAR-10  dataset  \cite{krizhevsky2009learning} with  VGG nets \cite{vgg} and ResNets \cite{he2016deep}.
In training, all explicit regularizations  are removed to keep  consistent with our theoretical analysis. 
More details of the experimental setup can be found in  Appendix \ref{sec: app-experiment-setup}.

 \paragraph*{The alignment property and escaping behavior.}
  Figure \ref{fig: large-scale-1} reports the alignment strength of SGD noise during training for VGG-19 and ResNet-110. One can see that the alignment factors are significantly positive and similar results are also observed for a variety of VGG nets and ResNets of different depths and can be found in Figure \ref{fig: cifar10-large-app} of Appendix \ref{sec: app-experiment-result}. On can see that the alignment strength is nearly independent of the model size. Figure \ref{fig: large-scale-2} shows the behavior of escaping from sharp minima for VGG-19 and ResNet-110. One can still observe that the escape is exponentially fast and similar observation for other ResNets and VGG nets can be found in Figure \ref{fig: large-scale-escape} in the appendix.  These observations suggest that our theoretical findings also hold for this practical setting.
\begin{figure}[!h]
\captionsetup[subfloat]{farskip=0pt,captionskip=0pt}
\centering
\subfloat[]{\label{fig: large-scale-1}
  \hspace*{-1em}
  \includegraphics[width=0.22\textwidth]{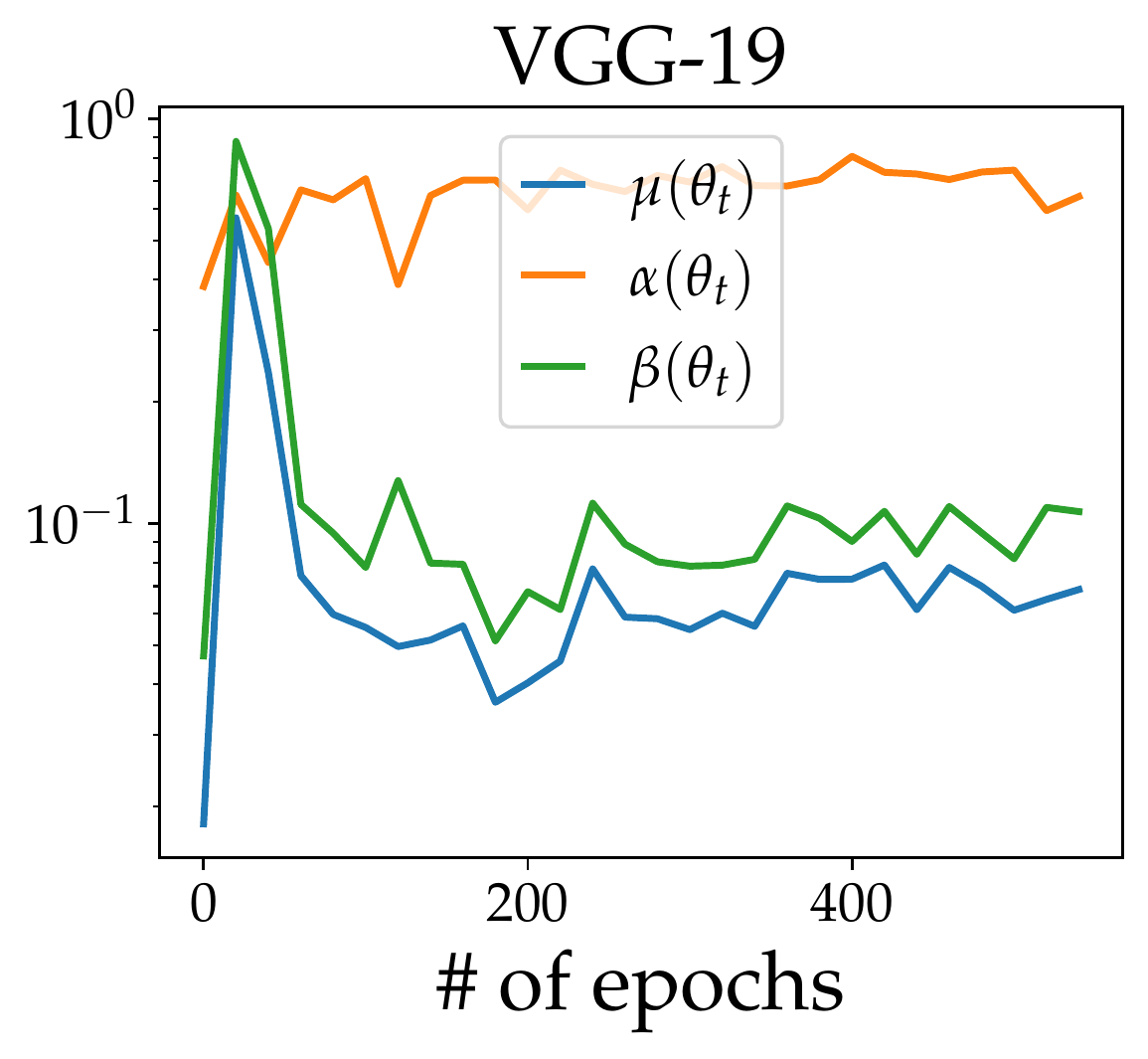}
  \includegraphics[width=0.22\textwidth]{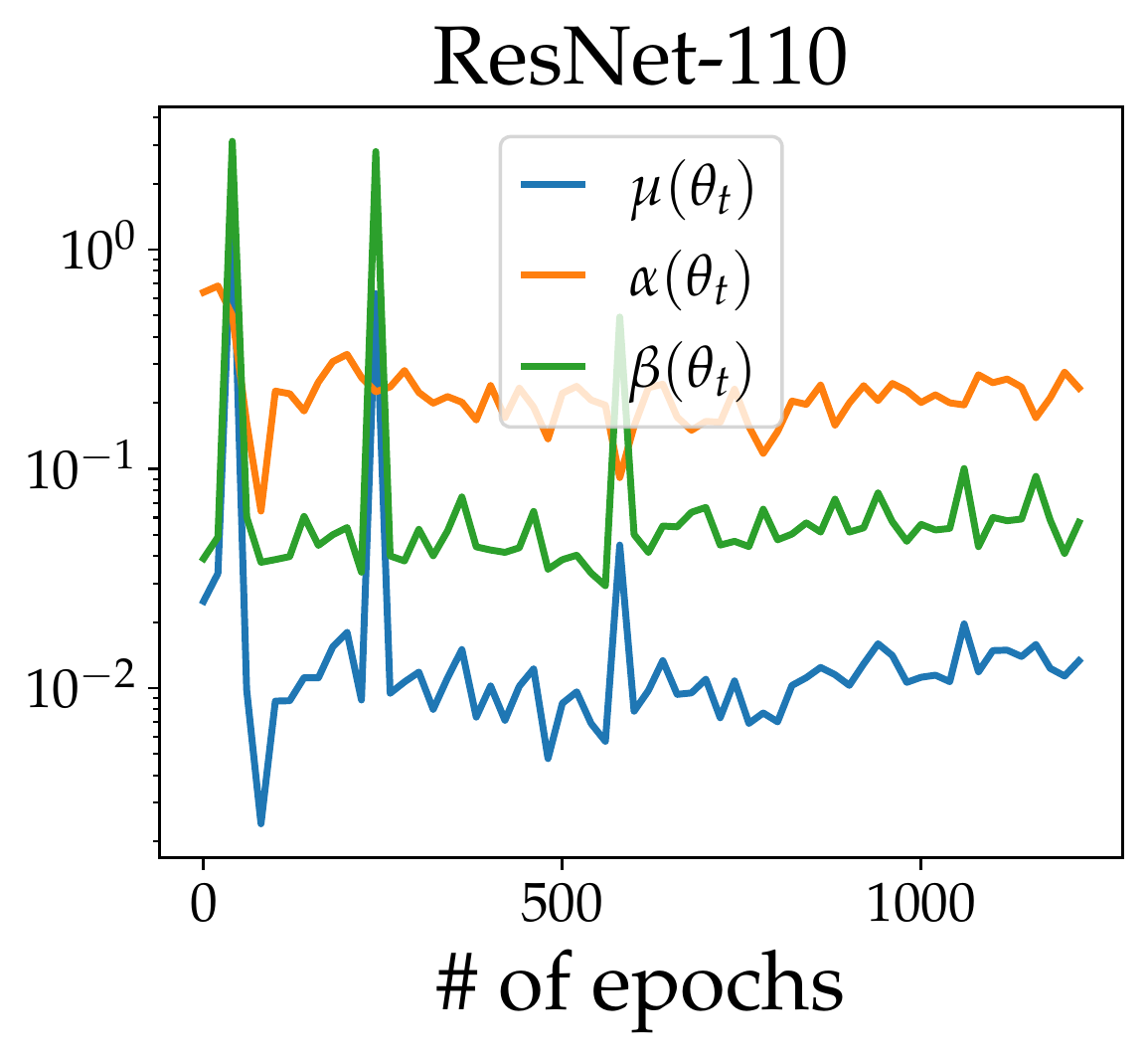}
}
\hspace*{.2em}
\subfloat[]{\label{fig: large-scale-2}
  \hspace*{-1em}
    \includegraphics[width=0.236\textwidth]{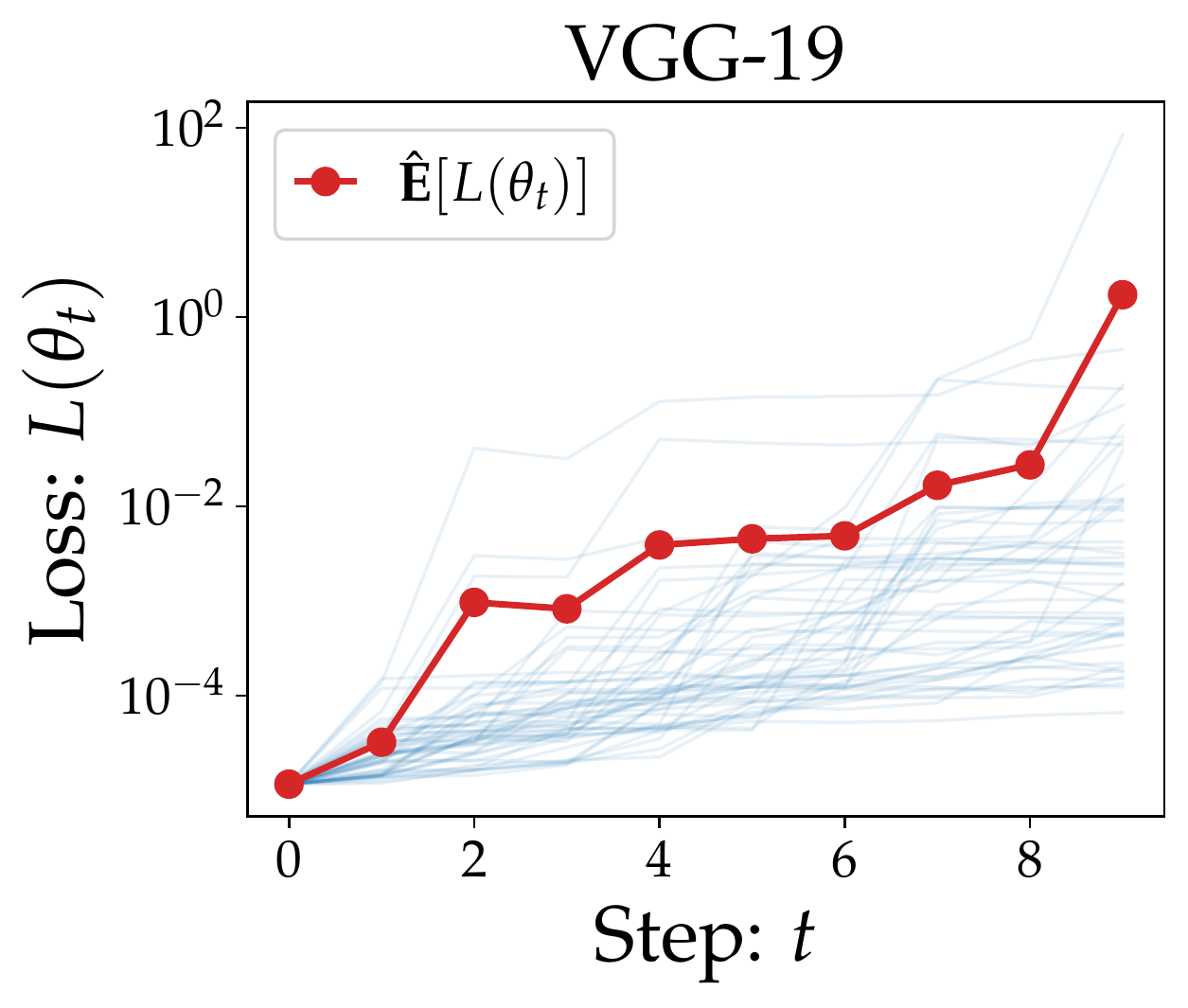}
  \includegraphics[width=0.23\textwidth]{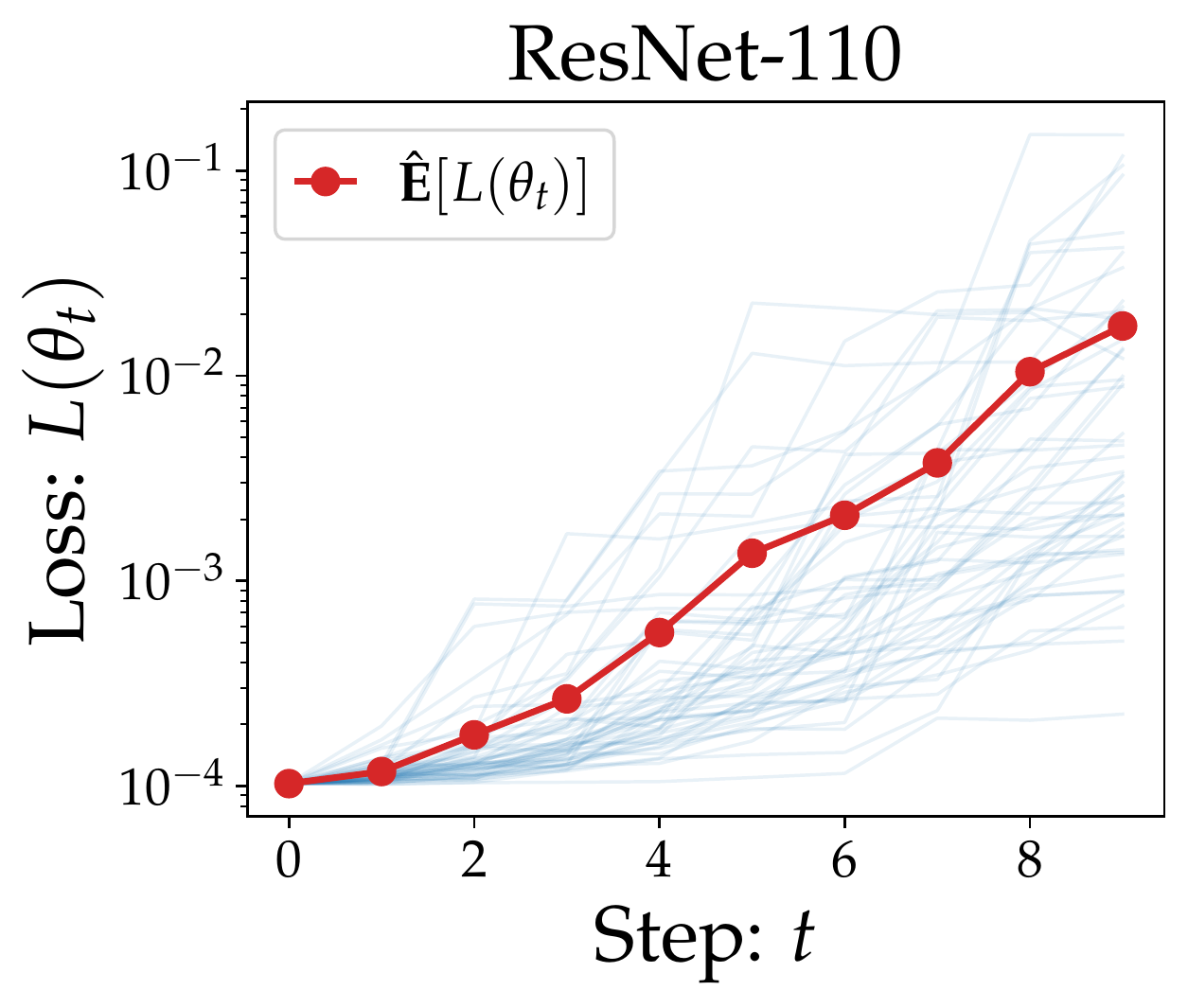}
}
\vspace*{-.7em}
	\caption{\small (a) The alignment factors  and the escaping behavior for VGG-19. Similar results are also observed for all examined  ResNets and VGG nets, which can be found in  Appendix. (b) The actual flatness of  SGD solution and the corresponding theoretical upper bound (Theorem \ref{thm: Fro-norm-Hessian-bound}).  (c) The upper bound becomes tighter as decreasing the batch size.
}
\label{fig: cifar10-large}
\vspace*{-1em}
\end{figure}

\paragraph*{The flatness and upper bound.}
Figure \ref{fig: large-scale-3} reports the  flatness of convergent solution and the corresponding  upper bound predicted by our theory for  ResNets and VGG nets. It is again observed that the flatness is nearly independent of the model size.  A surprising observation in Figure \ref{fig: large-scale-3} is that our upper bounds are rather tight, see, e.g., VGG-16 and VGG-19. This tightness suggests that SGD runs (nearly) at the edge of stability \cite{wu2018sgd,cohen2021gradient}.
Moreover, as expected, Figure \ref{fig: xx} shows that the our bound becomes tighter as decreasing batch size. This is consistent with what we observe in small-scale experiment in Figure \ref{fig: size-independence-b}.
\begin{figure}[!h]
\captionsetup[subfloat]{farskip=0pt,captionskip=0pt}
\centering
\subfloat[]{\label{fig: large-scale-3}
  \includegraphics[width=0.222\textwidth]{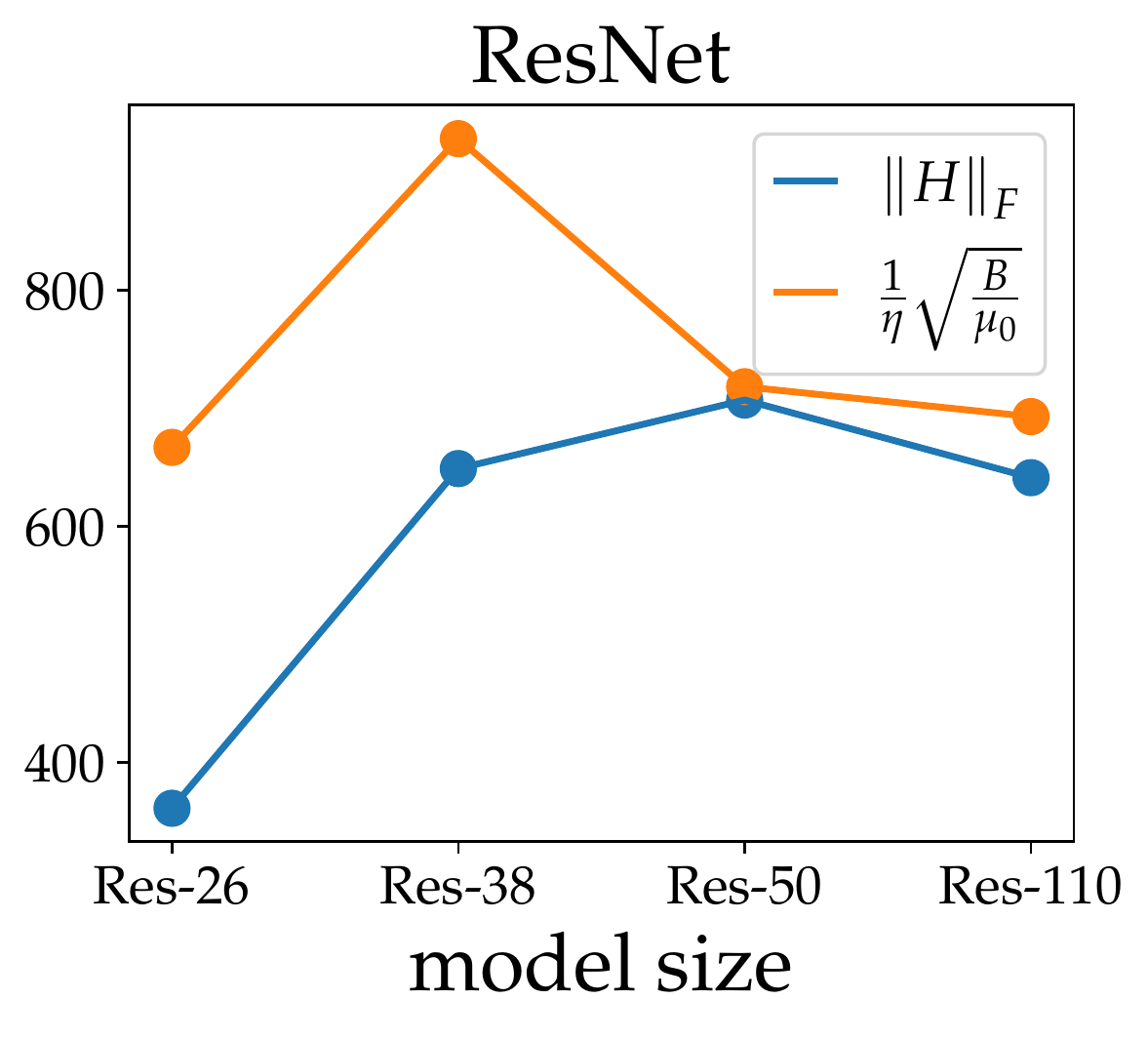}
  \includegraphics[width=0.222\textwidth]{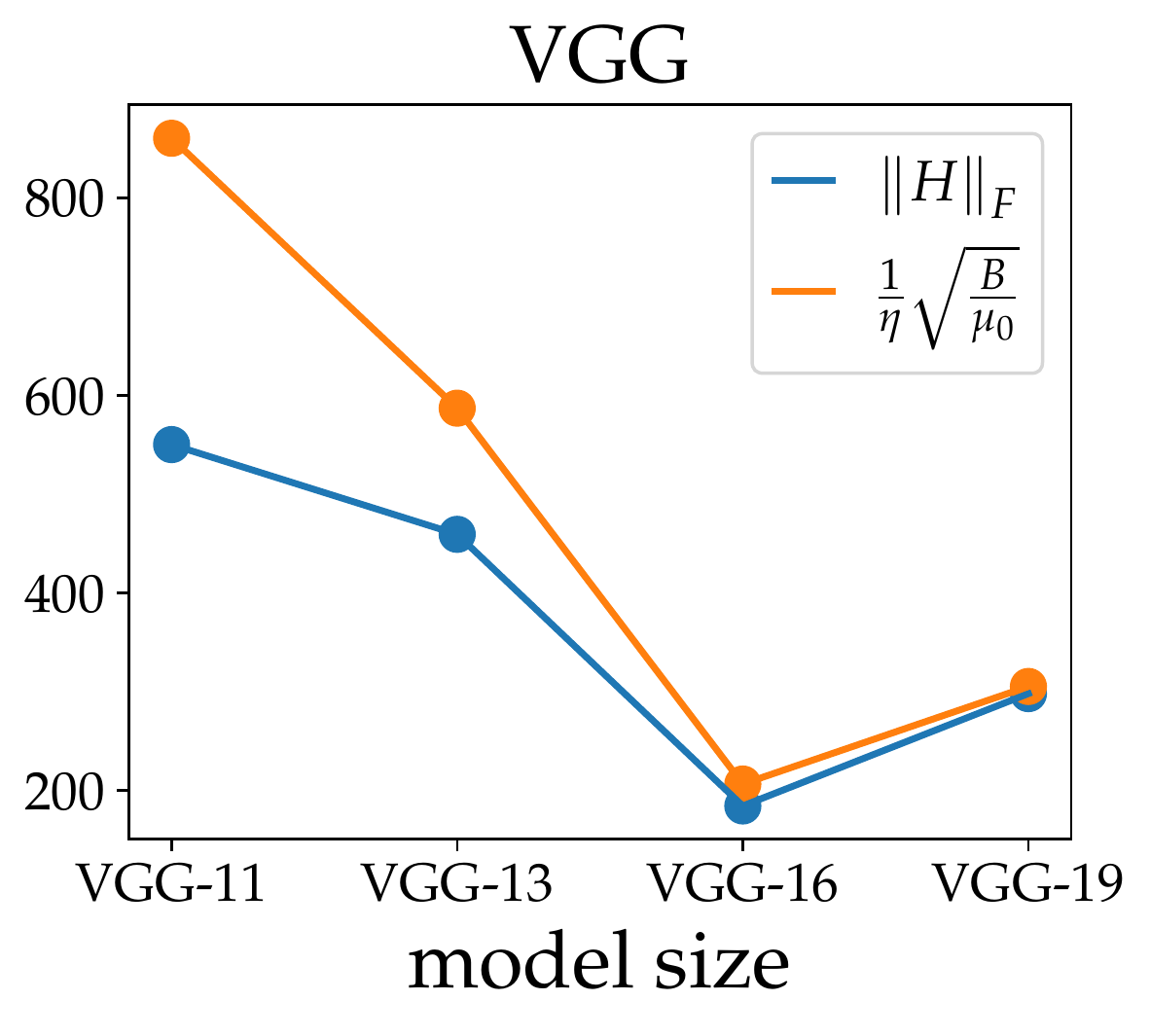}
}
\subfloat[]{\label{fig: xx}
  \includegraphics[width=.212\textwidth]{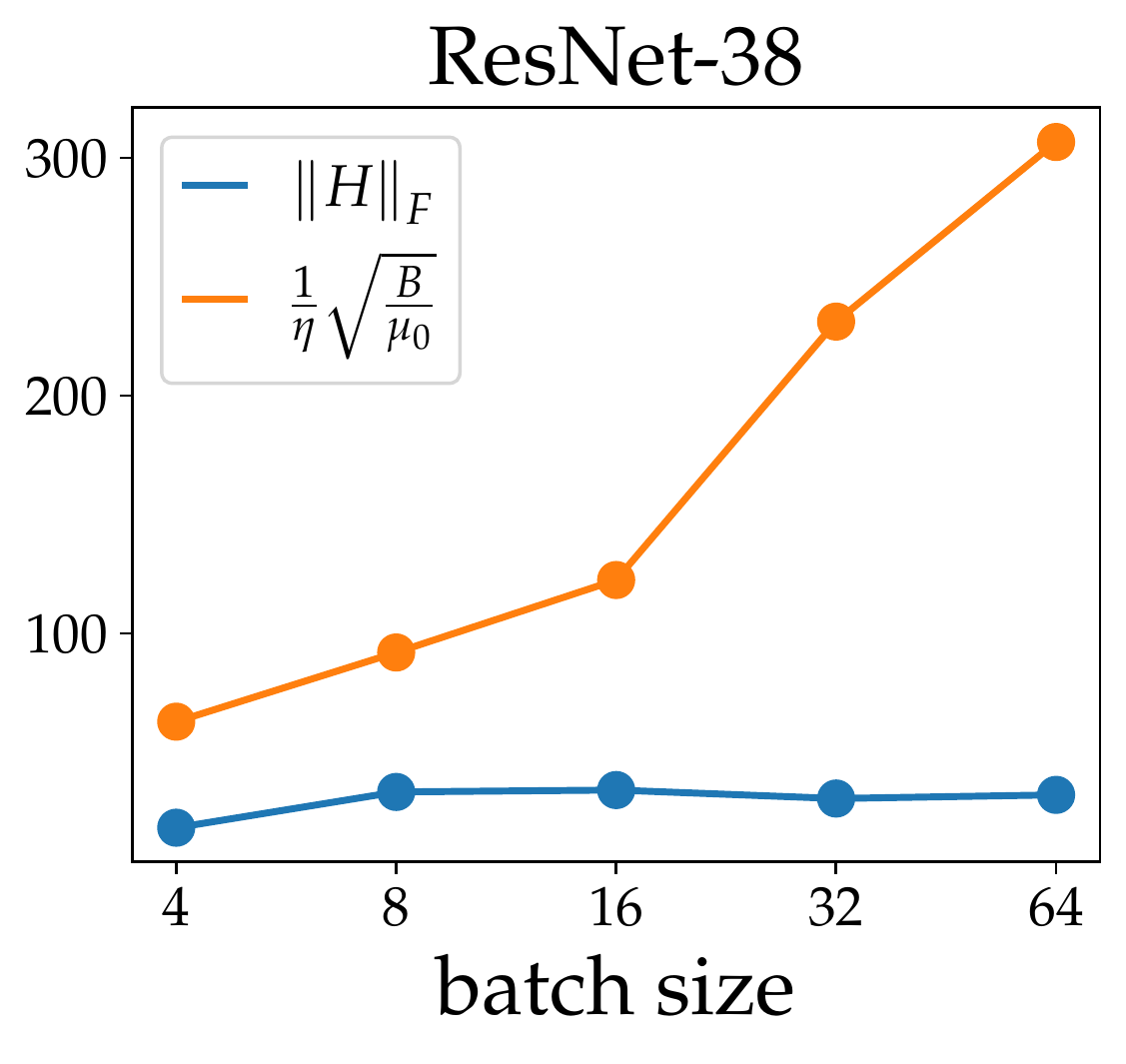}
  \includegraphics[width=.22\textwidth]{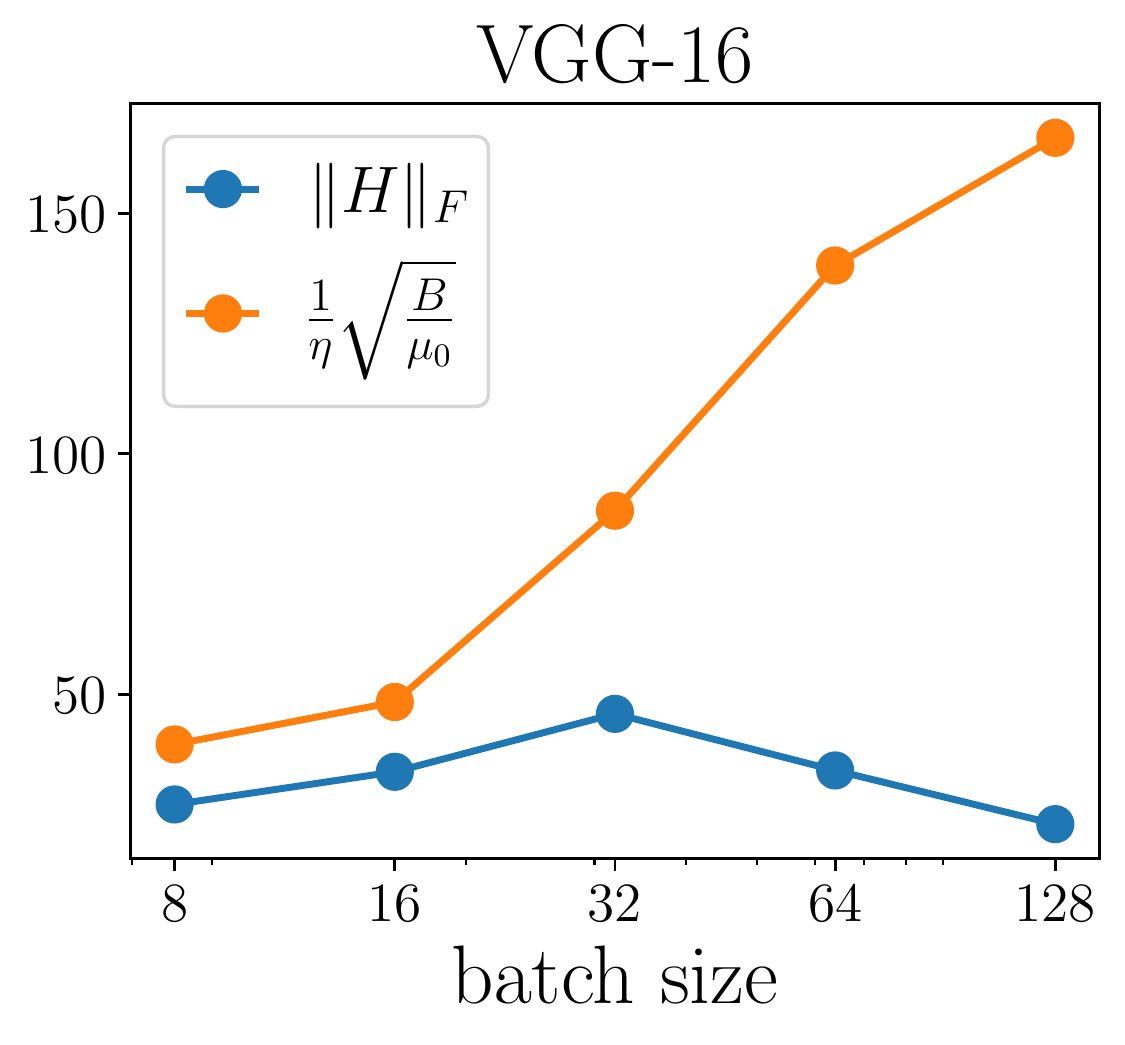}
}
\vspace*{-.3em}
\caption{\small
(a) The flatness and upper bound for various ResNets and VGG nets. 
(b) How the tightness of upper bound changes with decreasing batch size.
}
\vspace*{-1em}
\end{figure}

\section{Conclusion}
\vspace*{-.5em}

We provide a stability-based explanation of why SGD selects flat minima. Our current understanding is as follows. 1) For popular ML models, the SGD noise  aligns very well with local landscape. 2) This alignment property ensures that the flatness of  stable minima  must be size-independent. This understanding is made rigorous and quantitative by introducing a loss-scaled alignment factor to characterize the alignment strength and analyze the linear stability.
Obviously, many questions remains open. For example, can we understand what roles the stability plays in the whole dynamic process of SGD instead of only around global minima?   Can we establish the connection between the Hessian's Frobenius norm and generalization? 
Can we provide a fine-grained characterization of the noise structure and how the structure is related to implicit regularization of SGD? 
We leave the discussion of these important questions to future work.

\subsubsection*{Acknowledgements}
\vspace*{-0.4em}
We thank Zhiqin Xu and the anonymous reviewers for helpful suggestions. The work of Lei Wu is supported by a startup fund from Peking University. The work of Weijie Su is 
supported in part by NSF Grants CAREER DMS-1847415 and an Alfred Sloan Research Fellowship.

\bibliographystyle{plain}
\bibliography{ref}

\begin{thebibliography}{10}

\bibitem{arnold2012geometrical}
Vladimir~Igorevich Arnold.
\newblock {\em Geometrical methods in the theory of ordinary differential
  equations}, volume 250.
\newblock Springer Science \& Business Media, 2012.

\bibitem{aronszajn1950theory}
Nachman Aronszajn.
\newblock Theory of reproducing kernels.
\newblock {\em Transactions of the American mathematical society},
  68(3):337--404, 1950.

\bibitem{arora2019implicit}
Sanjeev Arora, Nadav Cohen, Wei Hu, and Yuping Luo.
\newblock Implicit regularization in deep matrix factorization.
\newblock {\em Advances in Neural Information Processing Systems}, 32, 2019.

\bibitem{azulay2021implicit}
Shahar Azulay, Edward Moroshko, Mor~Shpigel Nacson, Blake~E Woodworth, Nathan
  Srebro, Amir Globerson, and Daniel Soudry.
\newblock On the implicit bias of initialization shape: Beyond infinitesimal
  mirror descent.
\newblock In {\em International Conference on Machine Learning}, pages
  468--477. PMLR, 2021.

\bibitem{bach2017breaking}
Francis Bach.
\newblock Breaking the curse of dimensionality with convex neural networks.
\newblock {\em The Journal of Machine Learning Research}, 18(1):629--681, 2017.

\bibitem{blanc2020implicit}
Guy Blanc, Neha Gupta, Gregory Valiant, and Paul Valiant.
\newblock Implicit regularization for deep neural networks driven by an
  {Ornstein-Uhlenbeck} like process.
\newblock In {\em Conference on learning theory}, pages 483--513. PMLR, 2020.

\bibitem{cho2009kernel}
Youngmin Cho and Lawrence~K Saul.
\newblock Kernel methods for deep learning.
\newblock In {\em Proceedings of the 22nd International Conference on Neural
  Information Processing Systems}, pages 342--350, 2009.

\bibitem{cohen2021gradient}
Jeremy Cohen, Simran Kaur, Yuanzhi Li, J~Zico Kolter, and Ameet Talwalkar.
\newblock Gradient descent on neural networks typically occurs at the edge of
  stability.
\newblock In {\em International Conference on Learning Representations}, 2020.

\bibitem{damian2021label}
Alex Damian, Tengyu Ma, and Jason~D Lee.
\newblock Label noise {SGD} provably prefers flat global minimizers.
\newblock {\em Advances in Neural Information Processing Systems},
  34:27449--27461, 2021.

\bibitem{feng2021inverse}
Yu~Feng and Yuhai Tu.
\newblock The inverse variance--flatness relation in stochastic gradient
  descent is critical for finding flat minima.
\newblock {\em Proceedings of the National Academy of Sciences}, 118(9), 2021.

\bibitem{foret2020sharpness}
Pierre Foret, Ariel Kleiner, Hossein Mobahi, and Behnam Neyshabur.
\newblock Sharpness-aware minimization for efficiently improving
  generalization.
\newblock In {\em International Conference on Learning Representations}, 2020.

\bibitem{gardiner2009stochastic}
Crispin Gardiner.
\newblock {\em Stochastic methods}, volume~4.
\newblock springer Berlin, 2009.

\bibitem{haochen2021shape}
Jeff~Z HaoChen, Colin Wei, Jason Lee, and Tengyu Ma.
\newblock Shape matters: Understanding the implicit bias of the noise
  covariance.
\newblock In {\em Conference on Learning Theory}, pages 2315--2357. PMLR, 2021.

\bibitem{he2019local}
Hangfeng He and Weijie Su.
\newblock The local elasticity of neural networks.
\newblock In {\em International Conference on Learning Representations}, 2020.

\bibitem{he2016deep}
Kaiming He, Xiangyu Zhang, Shaoqing Ren, and Jian Sun.
\newblock Deep residual learning for image recognition.
\newblock In {\em Proceedings of the IEEE conference on computer vision and
  pattern recognition}, pages 770--778, 2016.

\bibitem{hochreiter1997flat}
S.~Hochreiter and J.~Schmidhuber.
\newblock Flat minima.
\newblock {\em Neural Computation}, 9(1):1--42, 1997.

\bibitem{izmailov2018averaging}
P~Izmailov, AG~Wilson, D~Podoprikhin, D~Vetrov, and T~Garipov.
\newblock Averaging weights leads to wider optima and better generalization.
\newblock In {\em 34th Conference on Uncertainty in Artificial Intelligence
  2018, UAI 2018}, pages 876--885, 2018.

\bibitem{jacot2018neural}
Arthur Jacot, Franck Gabriel, and Cl{\'e}ment Hongler.
\newblock Neural tangent kernel: {C}onvergence and generalization in neural
  networks.
\newblock In {\em Advances in neural information processing systems}, pages
  8571--8580, 2018.

\bibitem{jastrzkebski2017three}
Stanis{\l}aw Jastrz{\k{e}}bski, Zachary Kenton, Devansh Arpit, Nicolas Ballas,
  Asja Fischer, Yoshua Bengio, and Amos Storkey.
\newblock Three factors influencing minima in {SGD}.
\newblock {\em arXiv preprint arXiv:1711.04623}, 2017.

\bibitem{jastrzebski2019break}
Stanislaw Jastrzebski, Maciej Szymczak, Stanislav Fort, Devansh Arpit, Jacek
  Tabor, Kyunghyun Cho, and Krzysztof Geras.
\newblock The break-even point on optimization trajectories of deep neural
  networks.
\newblock In {\em International Conference on Learning Representations}, 2019.

\bibitem{keskar2016large}
N.~S. Keskar, D.~Mudigere, J.~Nocedal, M.~Smelyanskiy, and P.~T.~P. Tang.
\newblock On large-batch training for deep learning: Generalization gap and
  sharp minima.
\newblock In {\em In International Conference on Learning Representations
  (ICLR)}, 2017.

\bibitem{krizhevsky2009learning}
Alex Krizhevsky and Geoffrey Hinton.
\newblock Learning multiple layers of features from tiny images, 2009.

\bibitem{lecun1998gradient}
Yann LeCun, L{\'e}on Bottou, Yoshua Bengio, and Patrick Haffner.
\newblock Gradient-based learning applied to document recognition.
\newblock {\em Proceedings of the IEEE}, 86(11):2278--2324, 1998.

\bibitem{ledoux2013probability}
Michel Ledoux and Michel Talagrand.
\newblock {\em Probability in {Banach} Spaces: {Isoperimetry} and processes}.
\newblock Springer Science \& Business Media, 2013.

\bibitem{pmlr-v70-li17f}
Qianxiao Li, Cheng Tai, and Weinan E.
\newblock Stochastic modified equations and adaptive stochastic gradient
  algorithms.
\newblock In {\em Proceedings of the 34th International Conference on Machine
  Learning}, volume~70, pages 2101--2110. PMLR, Aug 2017.

\bibitem{li2021validity}
Zhiyuan Li, Sadhika Malladi, and Sanjeev Arora.
\newblock On the validity of modeling {SGD} with stochastic differential
  equations ({SDE}s).
\newblock In {\em Advances in Neural Information Processing Systems},
  volume~34, 2021.

\bibitem{li2021happens}
Zhiyuan Li, Tianhao Wang, and Sanjeev Arora.
\newblock What happens after {SGD} reaches zero loss? --a mathematical
  framework.
\newblock In {\em International Conference on Learning Representations}, 2022.

\bibitem{pmlr-v89-liang19a}
Tengyuan Liang, Tomaso Poggio, Alexander Rakhlin, and James Stokes.
\newblock {Fisher-Rao} metric, geometry, and complexity of neural networks.
\newblock In {\em Proceedings of the Twenty-Second International Conference on
  Artificial Intelligence and Statistics}, volume~89, pages 888--896. PMLR,
  2019.

\bibitem{ma2021linear}
Chao Ma and Lexing Ying.
\newblock On linear stability of {SGD} and input-smoothness of neural networks.
\newblock {\em Advances in Neural Information Processing Systems}, 34, 2021.

\bibitem{mori2021logarithmic}
Takashi Mori, Liu Ziyin, Kangqiao Liu, and Masahito Ueda.
\newblock Logarithmic landscape and power-law escape rate of {SGD}.
\newblock {\em arXiv preprint arXiv:2105.09557}, 2021.

\bibitem{mulayoff2021implicit}
Rotem Mulayoff, Tomer Michaeli, and Daniel Soudry.
\newblock The implicit bias of minima stability: A view from function space.
\newblock {\em Advances in Neural Information Processing Systems}, 34, 2021.

\bibitem{neyshabur2014search}
Behnam Neyshabur, Ryota Tomioka, and Nathan Srebro.
\newblock In search of the real inductive bias: On the role of implicit
  regularization in deep learning.
\newblock {\em arXiv preprint arXiv:1412.6614}, 2014.

\bibitem{pesme2021implicit}
Scott Pesme, Loucas Pillaud-Vivien, and Nicolas Flammarion.
\newblock Implicit bias of {SGD} for diagonal linear networks: a provable
  benefit of stochasticity.
\newblock {\em Advances in Neural Information Processing Systems}, 34, 2021.

\bibitem{sejdinovic2012rkhs}
Dino Sejdinovic and Arthur Gretton.
\newblock What is an {RKHS}?
\newblock {\em Lecture Notes}, 2012.

\bibitem{shalev2014understanding}
Shai Shalev-Shwartz and Shai Ben-David.
\newblock {\em Understanding machine learning: From theory to algorithms}.
\newblock Cambridge university press, 2014.

\bibitem{vgg}
Karen Simonyan and Andrew Zisserman.
\newblock Very deep convolutional networks for large-scale image recognition.
\newblock In {\em 3rd International Conference on Learning Representations,
  {ICLR} 2015, San Diego, CA, USA, May 7-9, 2015, Conference Track
  Proceedings}, 2015.

\bibitem{simsekli2019tail}
Umut Simsekli, Levent Sagun, and Mert Gurbuzbalaban.
\newblock A tail-index analysis of stochastic gradient noise in deep neural
  networks.
\newblock In {\em International Conference on Machine Learning}, pages
  5827--5837. PMLR, 2019.

\bibitem{su2021neurashed}
Weijie Su.
\newblock Neurashed: A phenomenological model for imitating deep learning
  training.
\newblock {\em arXiv preprint arXiv:2112.09741}, 2021.

\bibitem{tsuzuku2020normalized}
Yusuke Tsuzuku, Issei Sato, and Masashi Sugiyama.
\newblock Normalized flat minima: {E}xploring scale invariant definition of
  flat minima for neural networks using {PAC}-{B}ayesian analysis.
\newblock In {\em International Conference on Machine Learning}, pages
  9636--9647. PMLR, 2020.

\bibitem{wang2021eliminating}
Xingyu Wang, Sewoong Oh, and Chang-Han Rhee.
\newblock Eliminating sharp minima from {SGD} with truncated heavy-tailed
  noise.
\newblock In {\em International Conference on Learning Representations}, 2022.

\bibitem{wojtowytsch2021stochastic}
Stephan Wojtowytsch.
\newblock Stochastic gradient descent with noise of machine learning type. part
  {II}: Continuous time analysis.
\newblock {\em arXiv preprint arXiv:2106.02588}, 2021.

\bibitem{woodworth2020kernel}
Blake Woodworth, Suriya Gunasekar, Jason~D Lee, Edward Moroshko, Pedro
  Savarese, Itay Golan, Daniel Soudry, and Nathan Srebro.
\newblock Kernel and rich regimes in overparametrized models.
\newblock In {\em Conference on Learning Theory}, pages 3635--3673. PMLR, 2020.

\bibitem{wu2020adversarial}
Dongxian Wu, Shu-Tao Xia, and Yisen Wang.
\newblock Adversarial weight perturbation helps robust generalization.
\newblock {\em Advances in Neural Information Processing Systems}, 33, 2020.

\bibitem{wu2020noisy}
Jingfeng Wu, Wenqing Hu, Haoyi Xiong, Jun Huan, Vladimir Braverman, and
  Zhanxing Zhu.
\newblock On the noisy gradient descent that generalizes as {SGD}.
\newblock In {\em International Conference on Machine Learning}, pages
  10367--10376. PMLR, 2020.

\bibitem{wu2022spectral}
Lei Wu and Jihao Long.
\newblock A spectral-based analysis of the separation between two-layer neural
  networks and linear methods.
\newblock {\em Journal of Machine Learning Research}, 23(119):1--34, 2022.

\bibitem{wu2018sgd}
Lei Wu, Chao Ma, and Weinan E.
\newblock How {SGD} selects the global minima in over-parameterized learning: A
  dynamical stability perspective.
\newblock {\em Advances in Neural Information Processing Systems},
  31:8279--8288, 2018.

\bibitem{wu2017towards}
Lei Wu, Zhanxing Zhu, and Weinan E.
\newblock Towards understanding generalization of deep learning: Perspective of
  loss landscapes.
\newblock {\em arXiv preprint arXiv:1706.10239}, 2017.

\bibitem{xie2020diffusion}
Zeke Xie, Issei Sato, and Masashi Sugiyama.
\newblock A diffusion theory for deep learning dynamics: Stochastic gradient
  descent exponentially favors flat minima.
\newblock In {\em International Conference on Learning Representations}, 2020.

\bibitem{zhang2017understanding}
Chiyuan Zhang, Samy Bengio, Moritz Hardt, Benjamin Recht, and Oriol Vinyals.
\newblock Understanding deep learning requires rethinking generalization.
\newblock In {\em International Conference on Learning Representations}, 2017.

\bibitem{fixup}
Hongyi Zhang, Yann~N. Dauphin, and Tengyu Ma.
\newblock Residual learning without normalization via better initialization.
\newblock In {\em International Conference on Learning Representations}, 2019.

\bibitem{zhou2020towards}
Pan Zhou, Jiashi Feng, Chao Ma, Caiming Xiong, Steven Chu~Hong Hoi, et~al.
\newblock Towards theoretically understanding why {SGD} generalizes better than
  {Adam} in deep learning.
\newblock {\em Advances in Neural Information Processing Systems}, 33, 2020.

\bibitem{zhu2019anisotropic}
Zhanxing Zhu, Jingfeng Wu, Bing Yu, Lei Wu, and Jinwen Ma.
\newblock The anisotropic noise in stochastic gradient descent: Its behavior of
  escaping from sharp minima and regularization effects.
\newblock In {\em International Conference on Machine Learning}, pages
  7654--7663. PMLR, 2019.

\bibitem{ziyin2021minibatch}
Liu Ziyin, Kangqiao Liu, Takashi Mori, and Masahito Ueda.
\newblock Strength of minibatch noise in {SGD}.
\newblock In {\em International Conference on Learning Representations}, 2022.

\end{thebibliography}


\appendix

\section{Experiment setup}
\label{sec: appendix-experiment}

\label{sec: app-experiment-setup}
We consider training the following models in the over-parameterized regime. In training, all explicit regularizations (including weight decay, dropout, data augmentation, batch normalization, learning rate decay) are removed, and a simple constant-LR SGD is used to train our models.

\textbf{Small-scale models:} In this case, we set the sample size to be particularly small on purpose, which allows interested readers to quickly reproduce these experiments on their own computer. Note that this choice does not impact our conclusions since we also discuss in details the influence of changing the extent of over-parameterization and provide larger-scale experiments.
\begin{itemize}
\item \underline{\text{Random feature model} (RFM)}. The inputs $\{x_i\}_{i=1}^n$ are drawn from $\cN(0,I_d)$ with $d=10, n=200$. The labels are generated by 
\[
f^*(x)=0.2x_1+\left(\sum_{i=2}^d x_i-1\right)^2/3 + \sin\left(\sum_{i=1}^{d-2}x_ix_{i+2}/4\right).
\]
The model is $f(x;\theta)=\sum_{j=1}^m \theta_j\relu(w_j^Tx)$ with $m=2000$ and $\{w_j\}_{j=1}^m$ independently drawn from $N(0,I_d)$ at initialization and fixed during the training. This model is trained by SGD with learning rate $\eta=0.003$ and batch size $B=5$.

\item \underline{\text{Linear networks}}. The inputs $\{x_i\}_{i=1}^n$ drawn are drawn from $\cN(0,I_d)$ with $d=100, n=50$. Here we set $n<d$ to examine the low-sample regime. The labels are generated by $f^*(x)=\sum_{i=1}^dx_i/d$. The model is a four-layer linear network: $d\to m\to m\to m\to 1$ with $m=50$. This model is trained by SGD with learning rate $\eta=0.1$ and batch size $B=5$. The default LeCun initialization is used.

\item \underline{\text{Fully-connected networks} (FCN)}.  We randomly sample $n$ data from the MNIST training set and label $\{1, 2, 3, 4, 5\}$ to $0$ and $\{5, 6, 7, 8, 9, 10\}$ to $1$ to form our new training set. The model is a ReLU-activated fully-connected network with the architecture: 
$784\to m\to m \to 1$. 
Except for studying the influence of over-parameterization, we always  set $m=30, n=1000$, and the model size is $p=25441$. This model is trained by SGD with $\eta=0.05, B=5$ and the LeCun initialization is used. 

\item \underline{\text{Convolutional neural networks} (CNN)}. This is   a small LeNet-type   CNN \cite{lecun1998gradient} whose architecture is given in Table \ref{tab: cnn-arch}. The training set is the same as the one constructed above for FCN.  The LeCun initialization is used and the model is trained by SGD with $\eta=0.1,B=5$.
\begin{table}[!h]
\renewcommand{\arraystretch}{1.1}
\centering
\caption{\small The architecture of CNN with $m$ controlling the network width. We always set $m=20$ except for studying the impact of over-parameterization, where we vary the value of $m$.}
\begin{tabular}{c|c}
 \specialrule{1pt}{1pt}{1pt} 
 \hline 
Layer & Output size \\\hline 
input &  $28\times 28\times 1$ \\
$3\times 3\times$ $m$, conv & $28\times 28\times m$ \\
$3\times 3\times$ $2m$, conv & $28\times 28\times 2m$ \\
$2\times 2$, avgpool & $14\times 14\times 2m$ \\
$3\times 3\times 2m$, conv & $14\times 14\times 2m$ \\
$3\times 3\times m$, conv & $14\times 14\times m$ \\
$2\times 2$, avgpool & $7\times 7\times m$ \\
\hline
flatten & $49m$\\
$49m\to 1$, linear & $1$ \\ \hline
 \specialrule{1pt}{1pt}{1pt}
\end{tabular}
\label{tab: cnn-arch}
\end{table}
\end{itemize}

\textbf{Larger-scale models:} For these experiments, the full CIFAR-10 dataset are used to train our model. 
\begin{itemize}
\item \underline{\text{VGG nets.}} The models are the standard VGG nets (including VGG-11, VGG-13, VGG-16, and VGG-19) for classifying CIFAR-10 proposed in \cite{vgg}. 
\item \underline{\text{ResNets}}.  The residual networks for CIFAR-10 proposed in \cite{he2016deep} are considered and the models include ResNets-26, ResNet-38, ResNet-50, and ResNet-100.  For ResNets, we follow \cite{fixup} to use the fixup initialization in order to ensure that the model can be trained without batch normalization. 
\end{itemize}
Both VGG nets and ResNets are trained by SGD with learning rate $\eta=0.1$ and batch size $B=64$ until the training loss becomes smaller than $10^{-4}$. 

\paragraph*{Hyperparameter choices.} In all experiments, the default model size, sample size, learning rate, and batch size described above are used unless explicitly specified, e.g., studying the influence of  over-parameterization.

\paragraph*{Efficient computations of the alignment factors and flatness.}
Let $g_i = \nabla f(x_i;\theta), e_i = f(x_i;\theta)-y_i$. Let $Q=(g_1,\dots,g_n)\in\RR^{p\times n}$ and $S=(e_1g_1,\dots,e_ng_n)\in\RR^{p\times n}$.
Then, 
\begin{align*}
  G &= \fn \sumin g_i g_i^T=\fn QQ^T\in \RR^{p\times p}\\
  \Sigma &= \fn \sumin e_i^2 g_i g_i^T - (\fn\sumin e_ig_i)(\fn\sumin e_ig_i)^T=\frac{1}{n}SPS^T\in \RR^{p\times p},
\end{align*}
where $P=I_{n\times n}-\frac{1}{n}\bm{1}\bm{1}^T\in\RR^{n\times n}$. Here $\bm{1}\in\RR^n$ denotes  the all-one vector.

Notice that the computation of $\alpha(\theta),\beta(\theta),\mu(\theta), \|H(\theta)\|_F$ can be reduced to computing $\|G\|_F, \|\Sigma\|_F,$ and $\tr(G\Sigma)$.  The time complexity of naively computing  them is on the order of $O(p^2n)$, which is prohibitive for large-scale models where $p\gg n$. A more efficient way is to use the following equations
\begin{equation}\label{eqn: computation-alignment}
\begin{aligned}
\|G\|_F^2 &= \frac{1}{n^2}\tr(QQ^TQQ^T) = \frac{1}{n^2}\|Q^TQ\|_F^2\\
\|\Sigma\|_F^2 &= \frac{1}{n^2}\tr(SPS^TSPS^T)=\frac{1}{n^2}\|P S^TS\|_F^2\\
\tr(G\Sigma) &= \frac{1}{n^2}\tr(QQ^TSPS^T)=\frac{1}{n^2}\|PS^TQ\|_F^2,
\end{aligned}
\end{equation}
where all the matrices are  $n\times n$. Hence, using these equations, the computation complexity becomes $O(n^2p)$. This is much smaller than $O(p^2n)$ when $p\gg n$.

\begin{itemize}
\item For small-scale experiments, the equations in \eqref{eqn: computation-alignment} are directly used. 
\item For the large-scale models, we need further approximations since the computation complexity $O(n^2p)$ is still prohibitive in this case. Notice that  the  formulations in Eq.~\eqref{eqn: computation-alignment} are all in the form of sample average, which allows us to perform Monte-Carlo approximation. Specifically, we
randomly choose $B$ samples from $x_1,\dots,x_n$ and still use \eqref{eqn: computation-alignment} to estimate these quantities. But now the computation complexity becomes $O(B^2 p)$. For the experiments on CIFAR-10, we test $B$'s with different values and find that $B=50$ is sufficient to obtain a reliable approximation of the original full-data quantity. Hence, for all large-scale experiments in this paper, we use $B=50$ to speed up the computation of alignment factors and flatness. We clarify that the models are still trained on the full dataset. 
\end{itemize}

\section{Extra experiment results}
\label{sec: app-experiment-result}
\subsection{Small-scale experiments} \label{sec: app-extral-experiment-result}

Figure \ref{fig: feature-norms-app} shows the relative norm of model gradient, i.e., $\chi_i(\theta)/\bar{\chi}(\theta)$ of each samples when the model convergences. It is shown that $\gamma(\theta)=\min_i \chi_i(\theta)/\bchi(\theta)$ is indeed bounded below. This explains why the SGD noise satisfies the alignment property locally according to Lemma \ref{lemma: gen-feature-based-model}.
\begin{figure}[!h]
\includegraphics[width=0.23\textwidth]{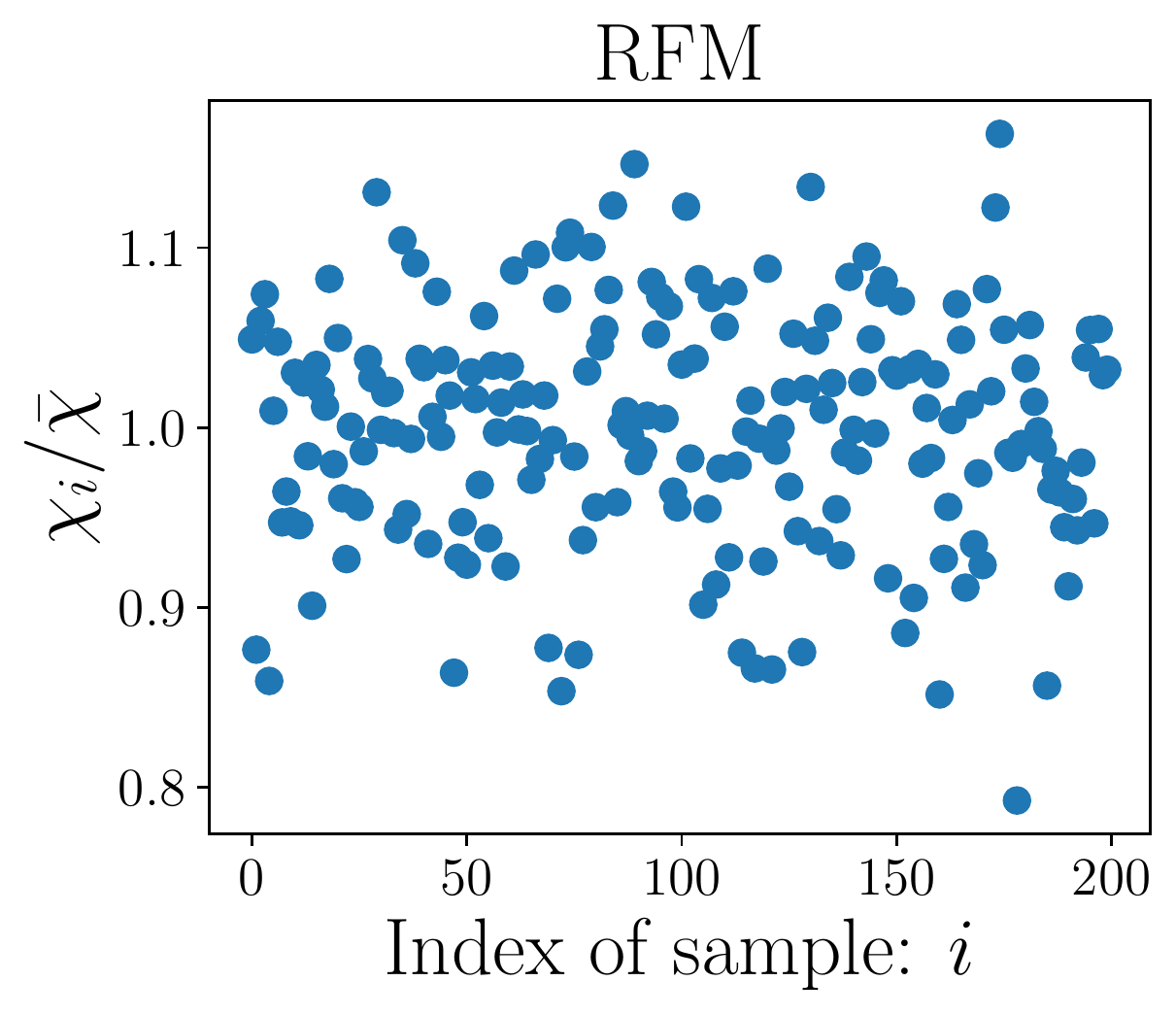}
\includegraphics[width=0.23\textwidth]{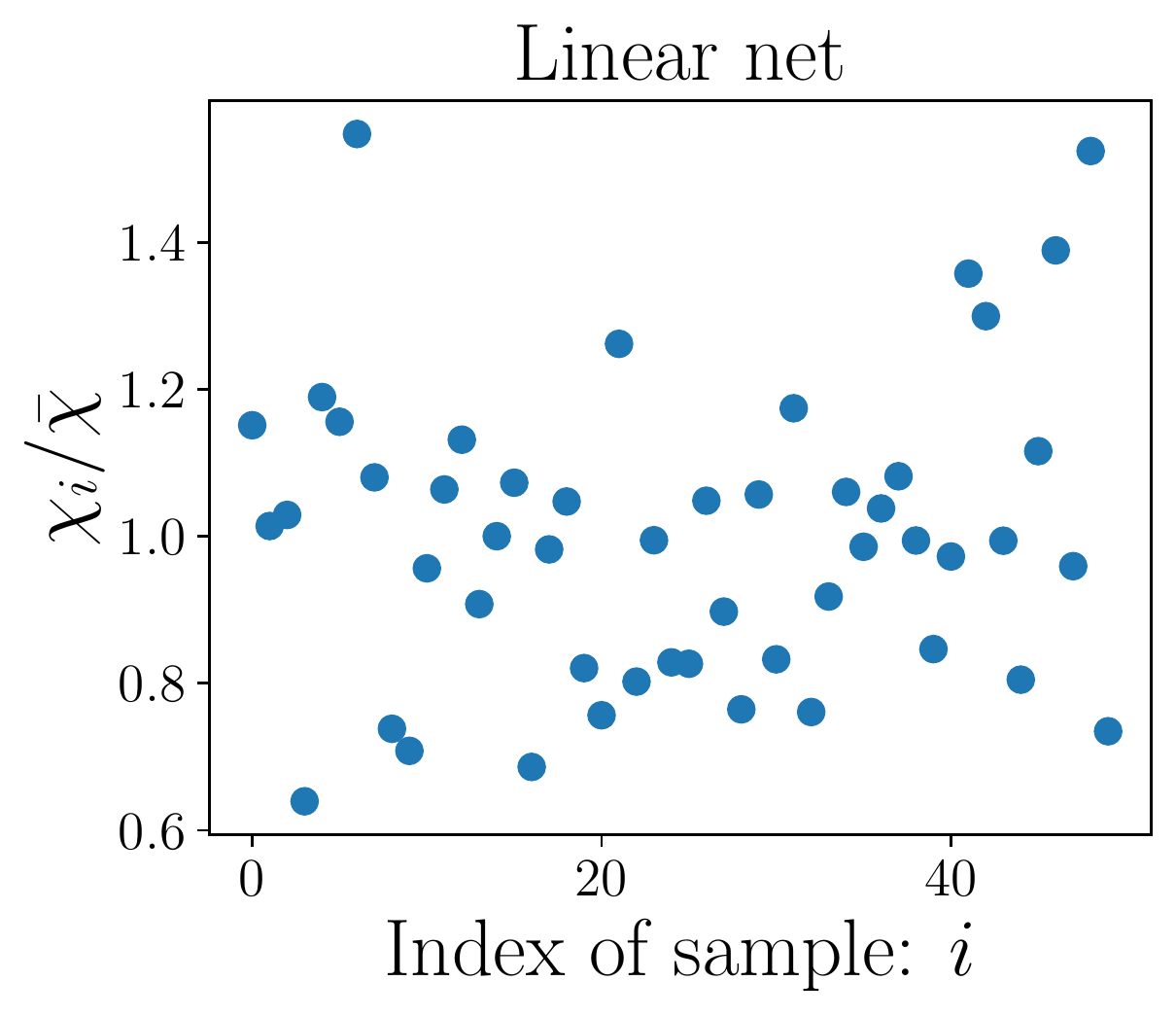}
\includegraphics[width=0.23\textwidth]{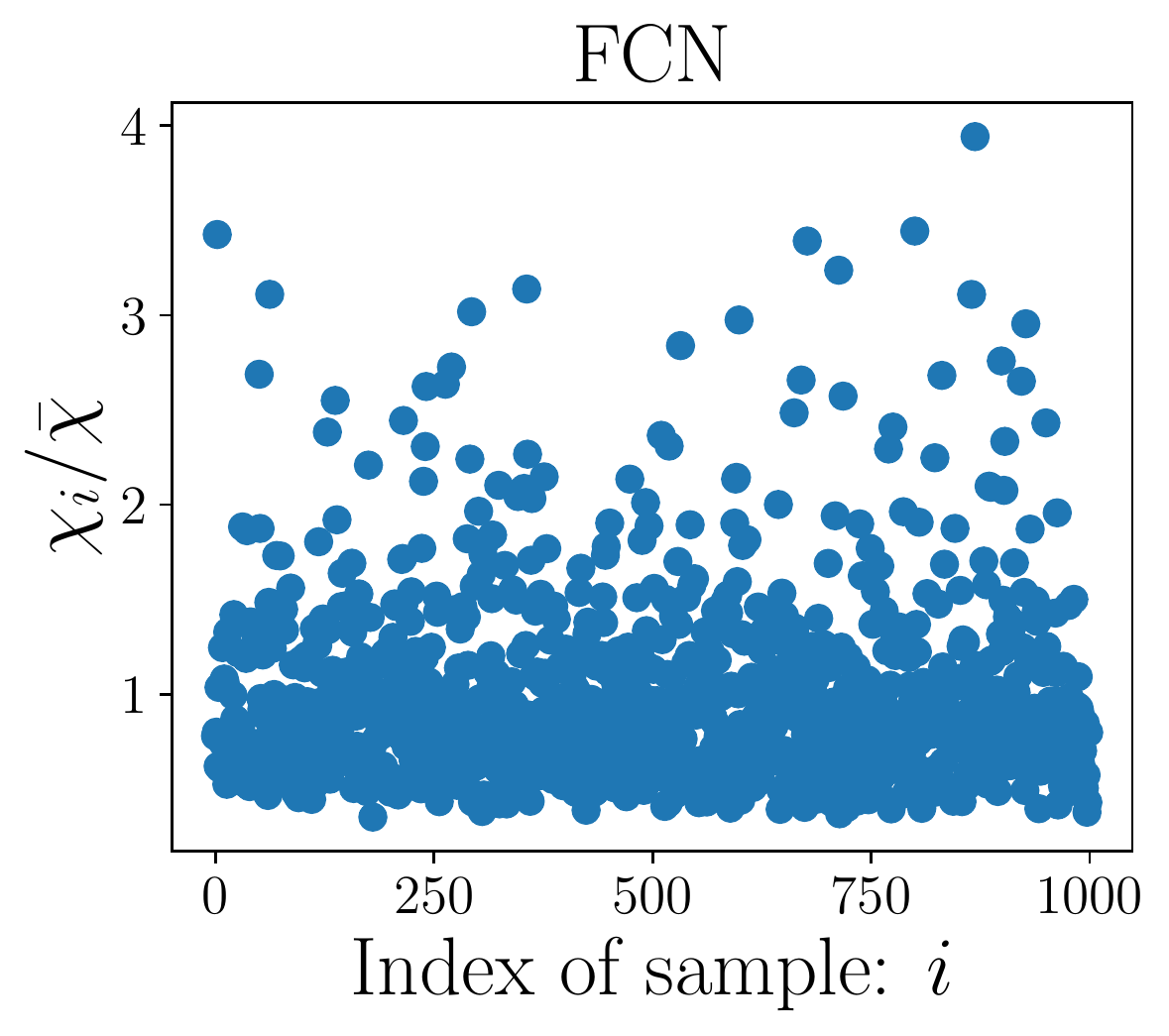}
\includegraphics[width=0.23\textwidth]{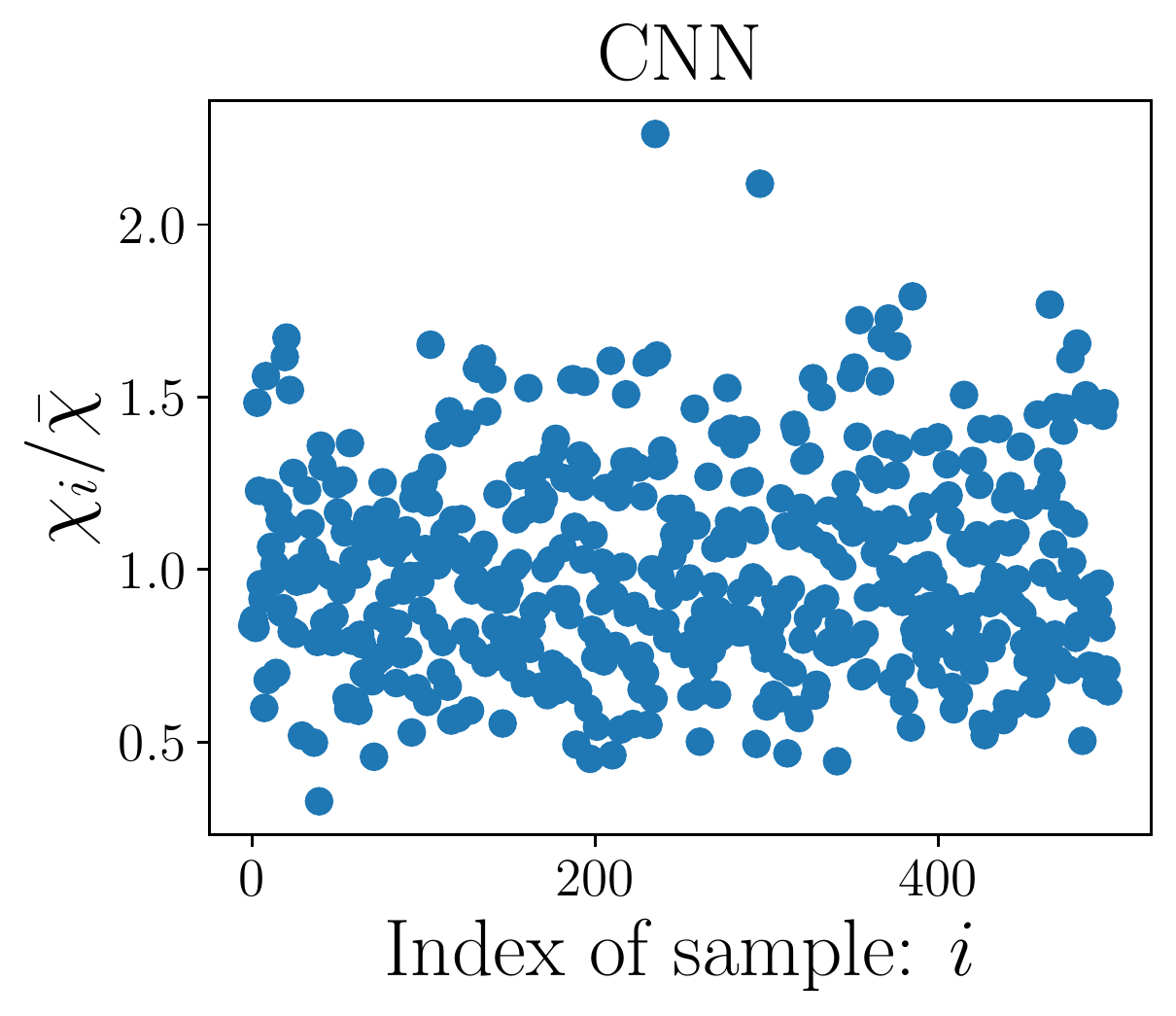}
\caption{\small The model gradient norms of SGD solutions for the RFM, linear network, FCN, and CNN.}
\label{fig: feature-norms-app}
\end{figure}

\subsection{Larger-scale experiments}
\label{sec: app-large-scale-exp}

Figure \ref{fig: cifar10-large-app} reports the values of alignment factors during the training for ResNet-26, ResNet-38, ResNet-50, VGG-13, and VGG-19. We see that the alignment factors are always significantly positive. In particular, for VGG-16, $\alpha(\theta)$ is close to $1$ and $\mu(\theta)$ is roughly on the order of $0.1$.  In addition, Figure \ref{fig: large-scale-escape} shows that the noise-driven escape is indeed exponentially fast as suggested by our theoretical analysis.

\begin{figure}[!h]
\centering
  \includegraphics[width=0.25\textwidth]{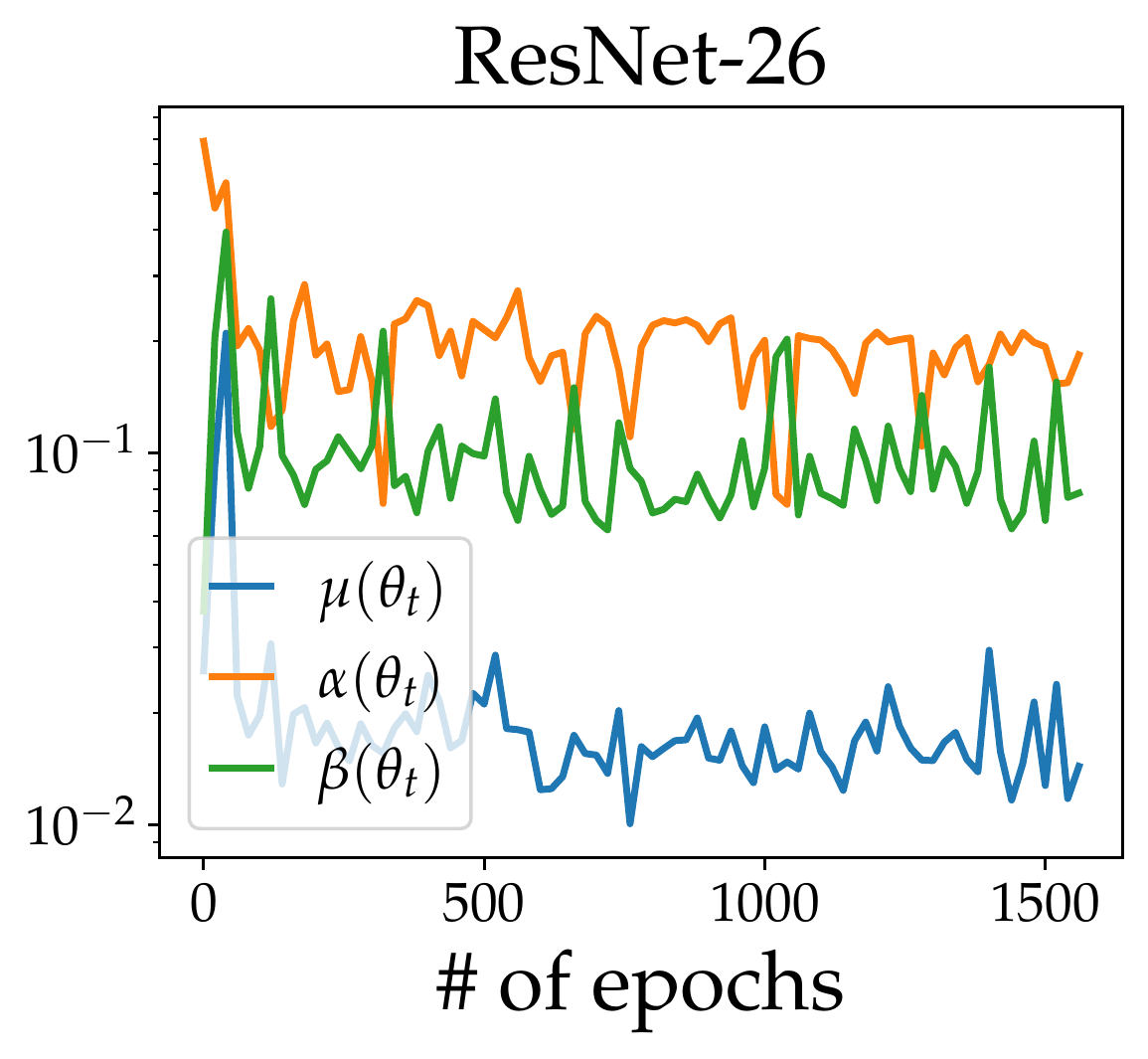}
  \includegraphics[width=0.25\textwidth]{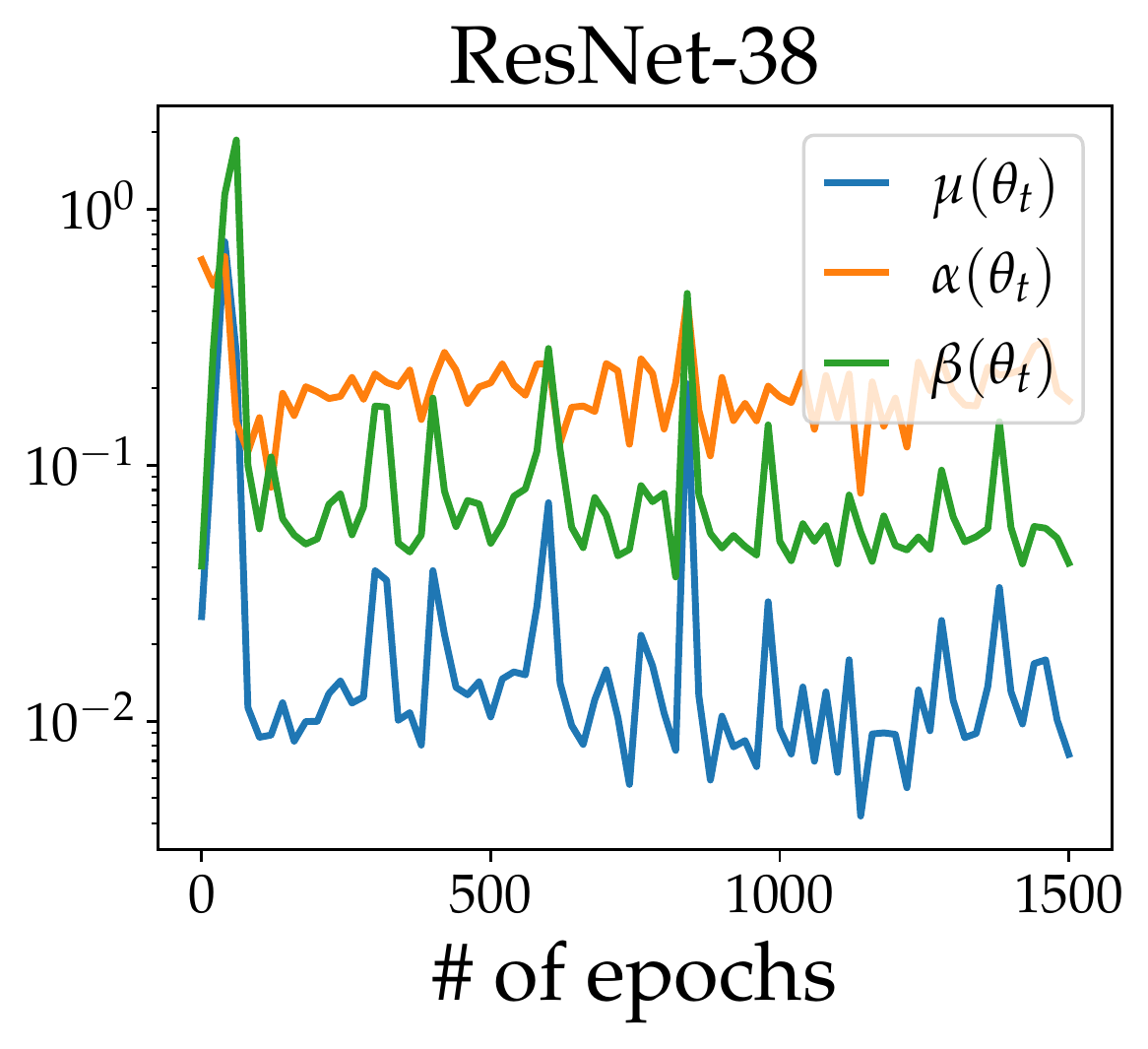}
  \includegraphics[width=0.25\textwidth]{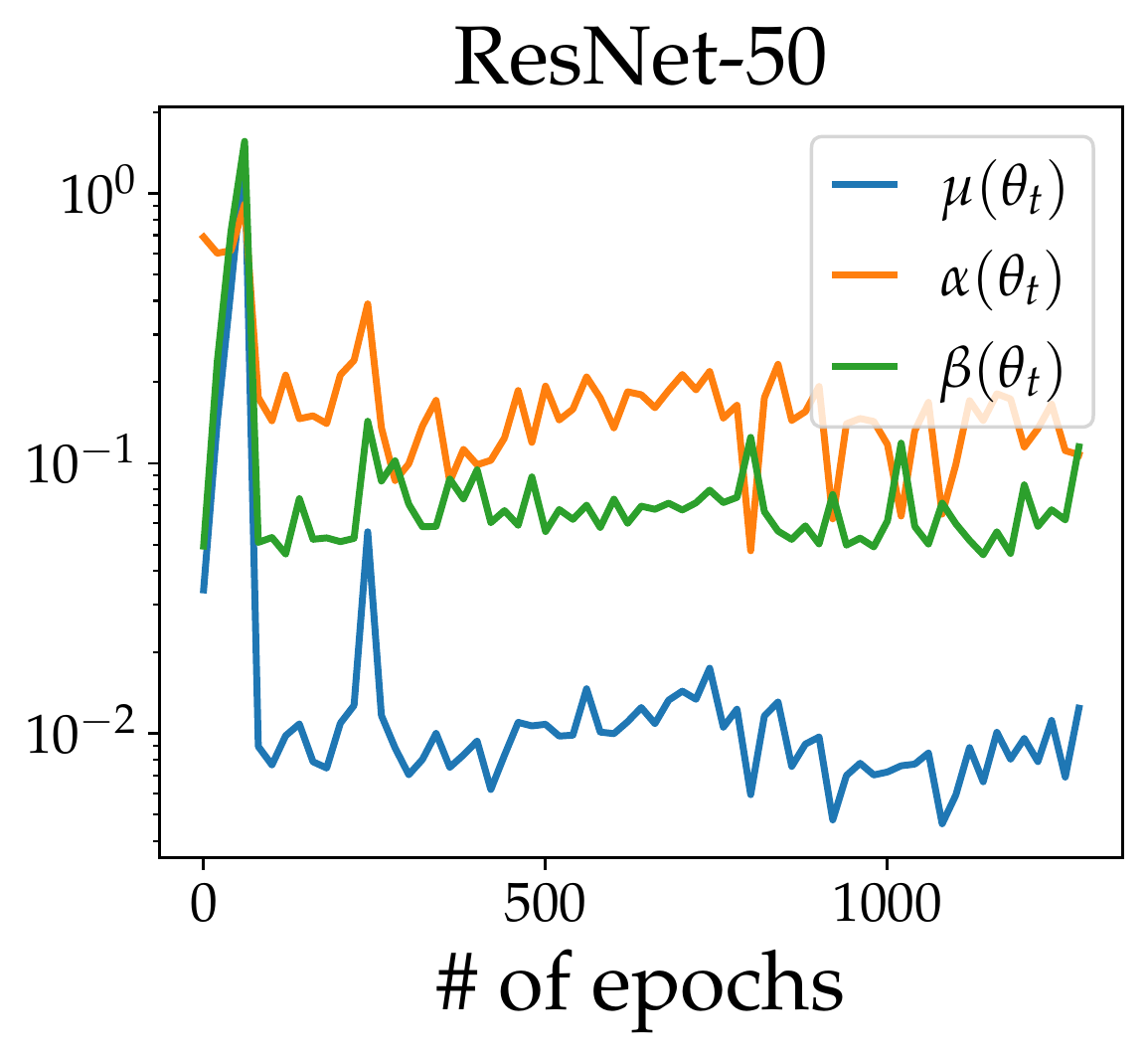}

  \includegraphics[width=0.25\textwidth]{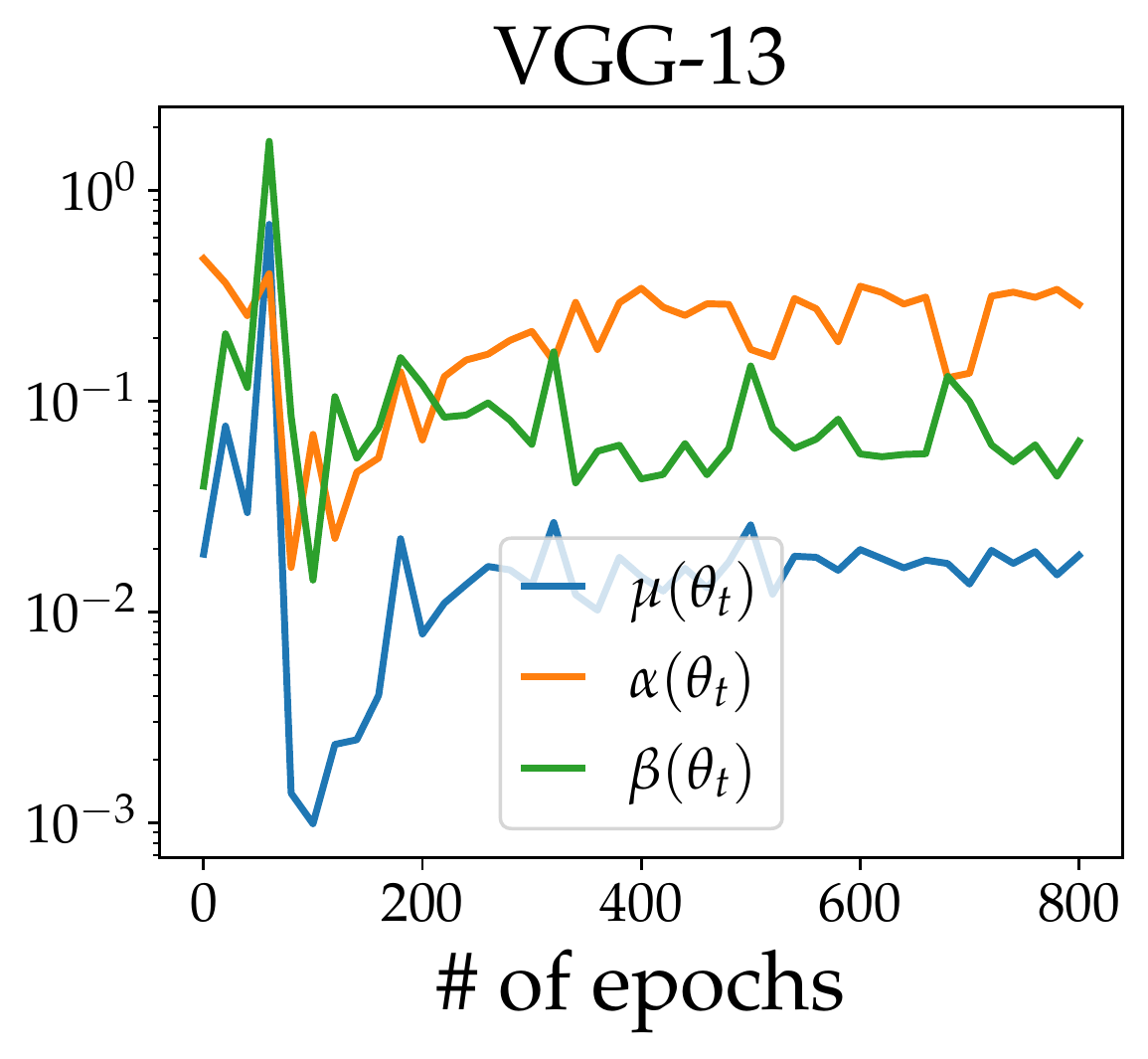}
  \includegraphics[width=0.25\textwidth]{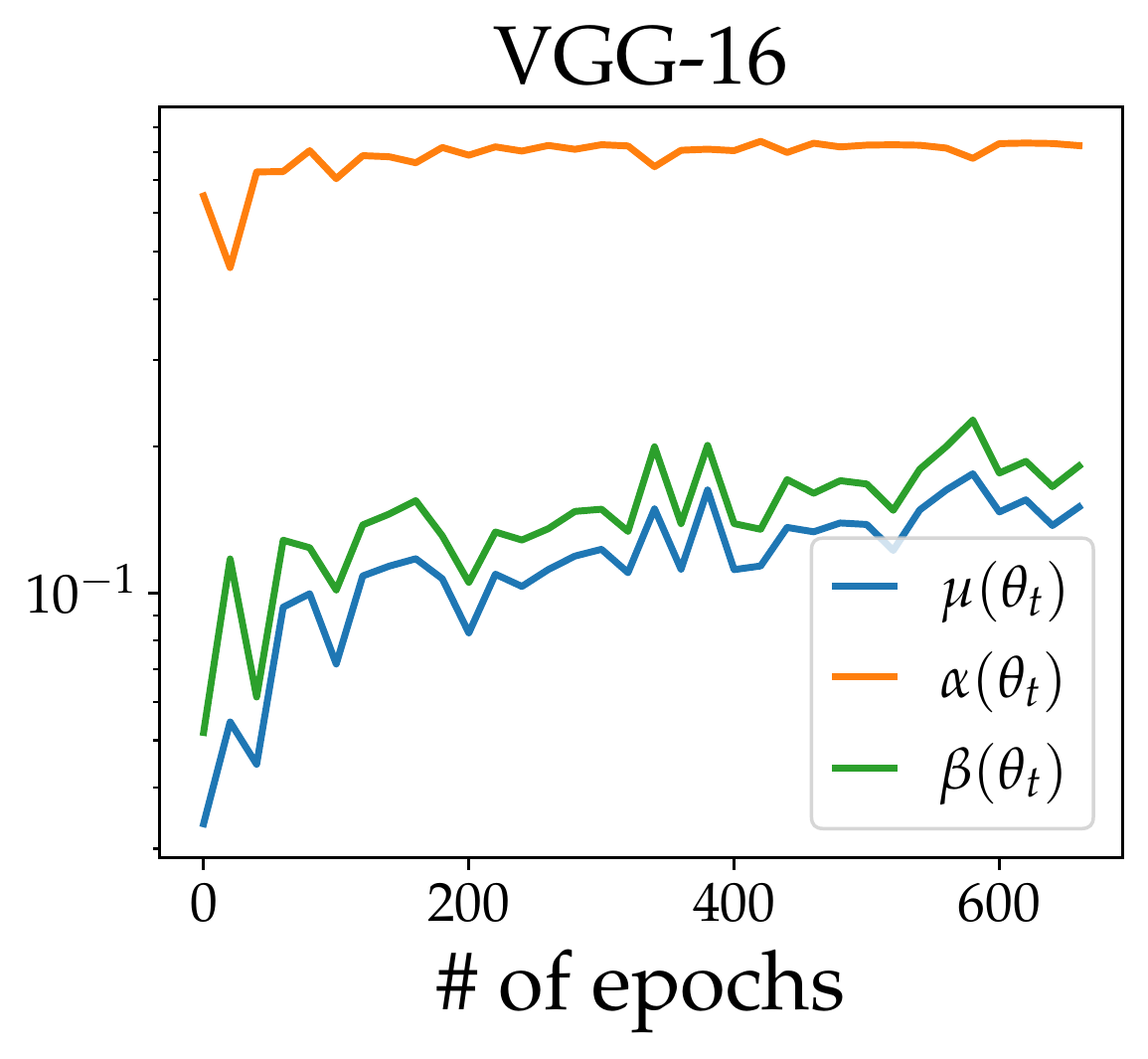}
\vspace*{-.5em}
\caption{\small \text{The alignment factors during the SGD training for classifying full CIFAR-10 dataset with VGG nets and ResNets .} }
\label{fig: cifar10-large-app}
\end{figure}
\begin{figure}[!h]
\centering
\includegraphics[width=0.25\textwidth]{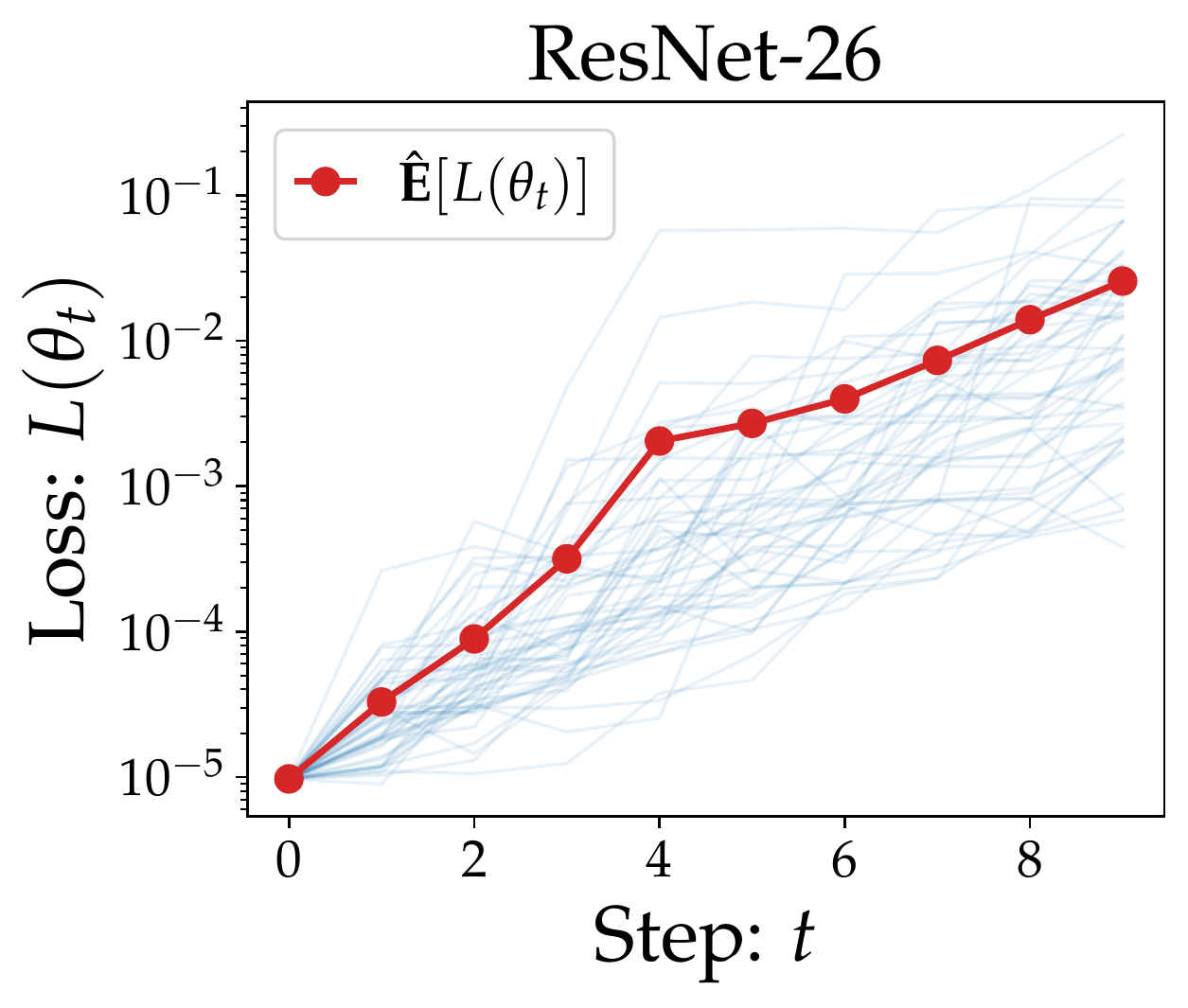}
\includegraphics[width=0.25\textwidth]{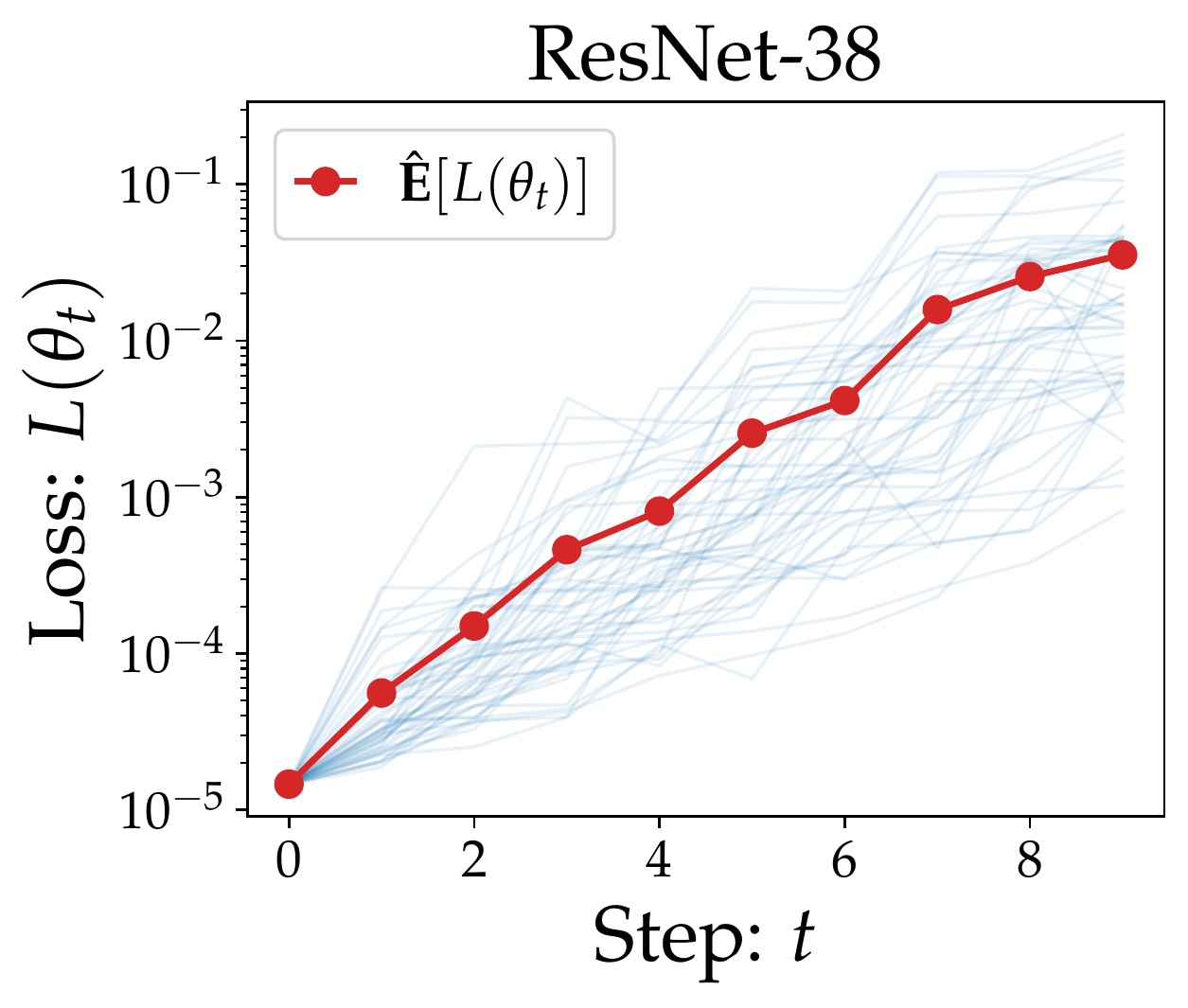}
\includegraphics[width=0.25\textwidth]{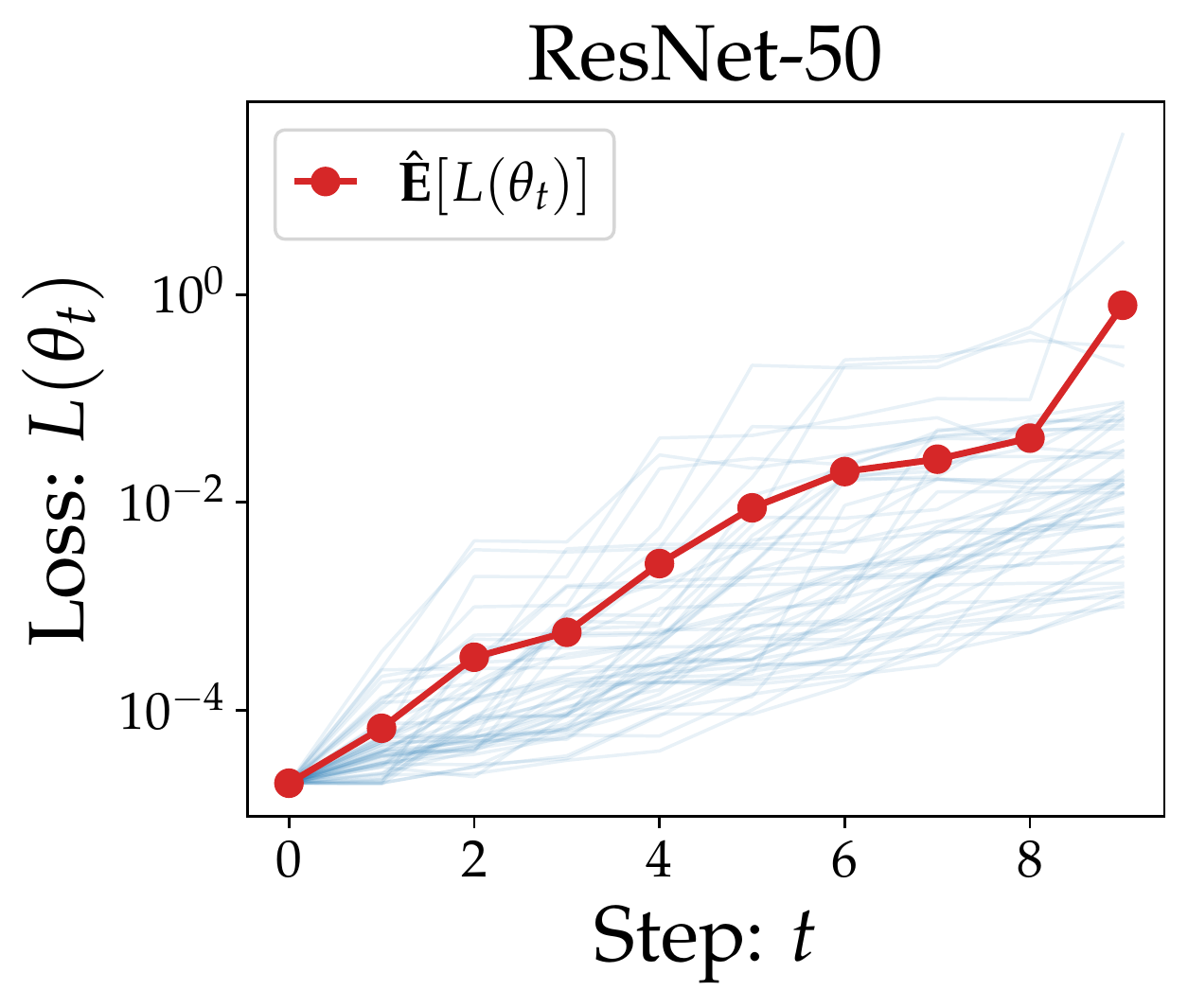}

\includegraphics[width=0.25\textwidth]{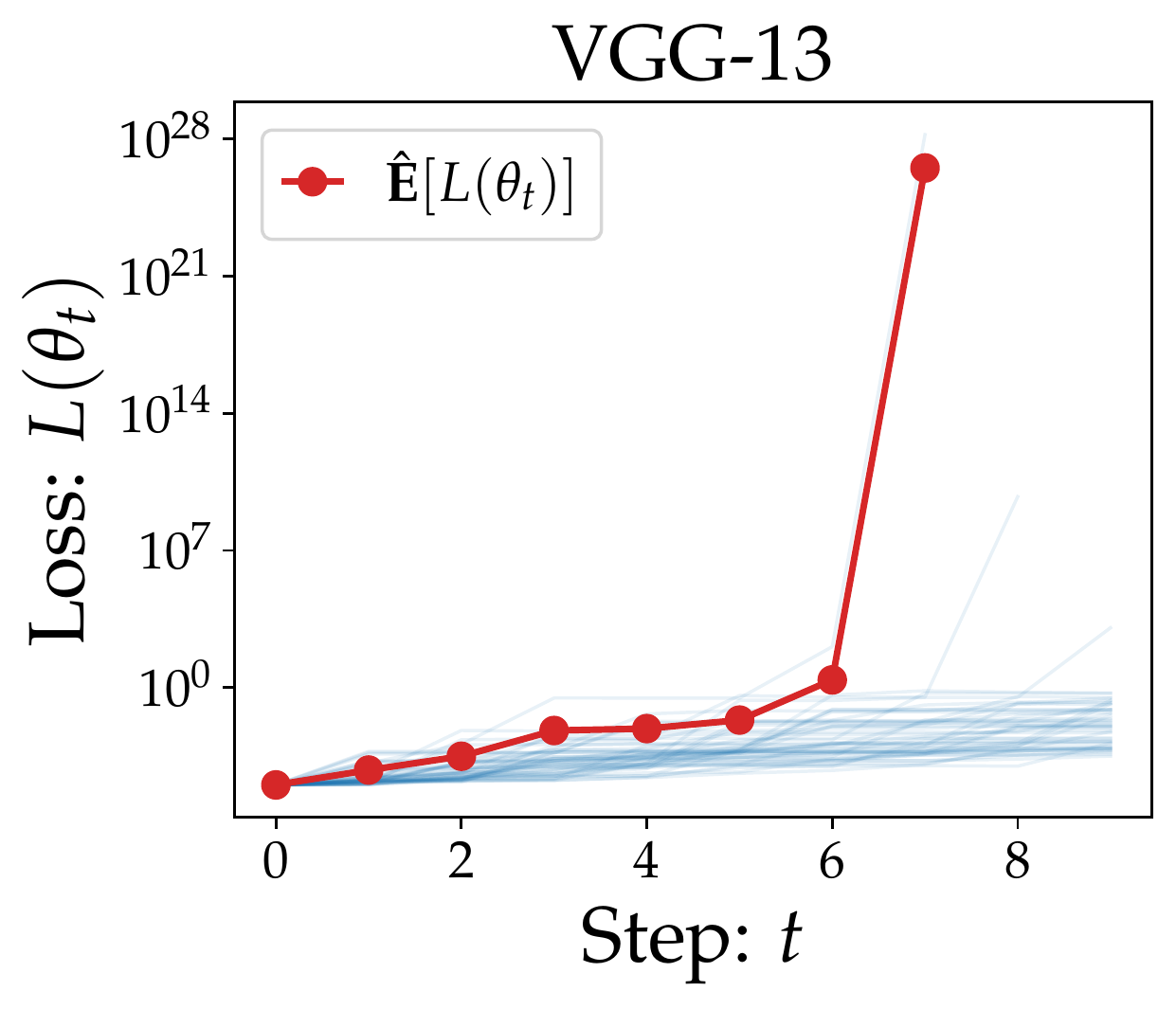}
\includegraphics[width=0.25\textwidth]{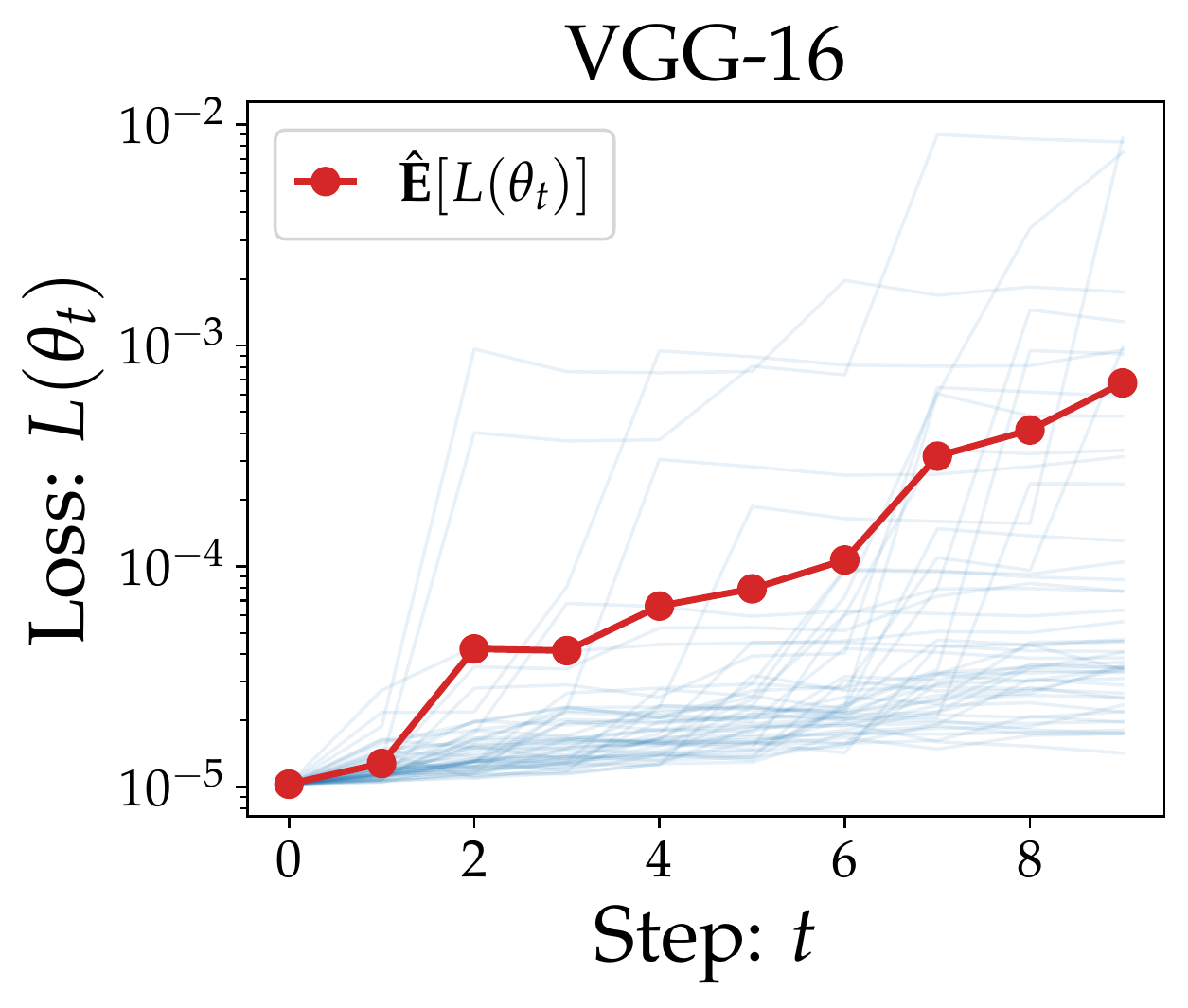}
\vspace*{-.6em}
\caption{\textbf{The exponentially fast escape from sharp minima for training CIFAR-10 dataset.}
The blue curves are  $50$ trajectories of SGD; the red curve  corresponds to the average. The sharp minimum is found by SGD with $B=64$ and $\eta=0.1$. When it nearly converges,  we switch to SGD with the same learning rate but a smaller batch size $B=4$. This choice ensures that  the escape is purely driven by SGD noise. 
}
\label{fig: large-scale-escape}
\end{figure}

However, it should be stressed that the loss-scaled alignment factors for these larger-scale models are clearly smaller than the ones reported in small-scale experiments. In particular, one can see that the values of $\mu(\theta)$ for various ResNets are only the order of $0.01$. To understand the underlying reason, 
we additionally examine the same models but for  classifying a two-class subset of CIFAR-10 with the results reported in Figure \ref{fig: cifar10-large-app-binary}. It is shown that in this case, the loss-scaled alignment factors are roughly on the order of $0.1$ for all the models, significantly larger than the case of classifying the full CIFAR-10 dataset. It is also very surprising to observe that for VGG nets, the standard alignment factor $\alpha(\theta)$'s  are nearly  $1$, suggesting that the noise covariance completely aligns with the geometry of local landscape.

\begin{figure}[!h]
\centering
  \includegraphics[width=0.24\textwidth]{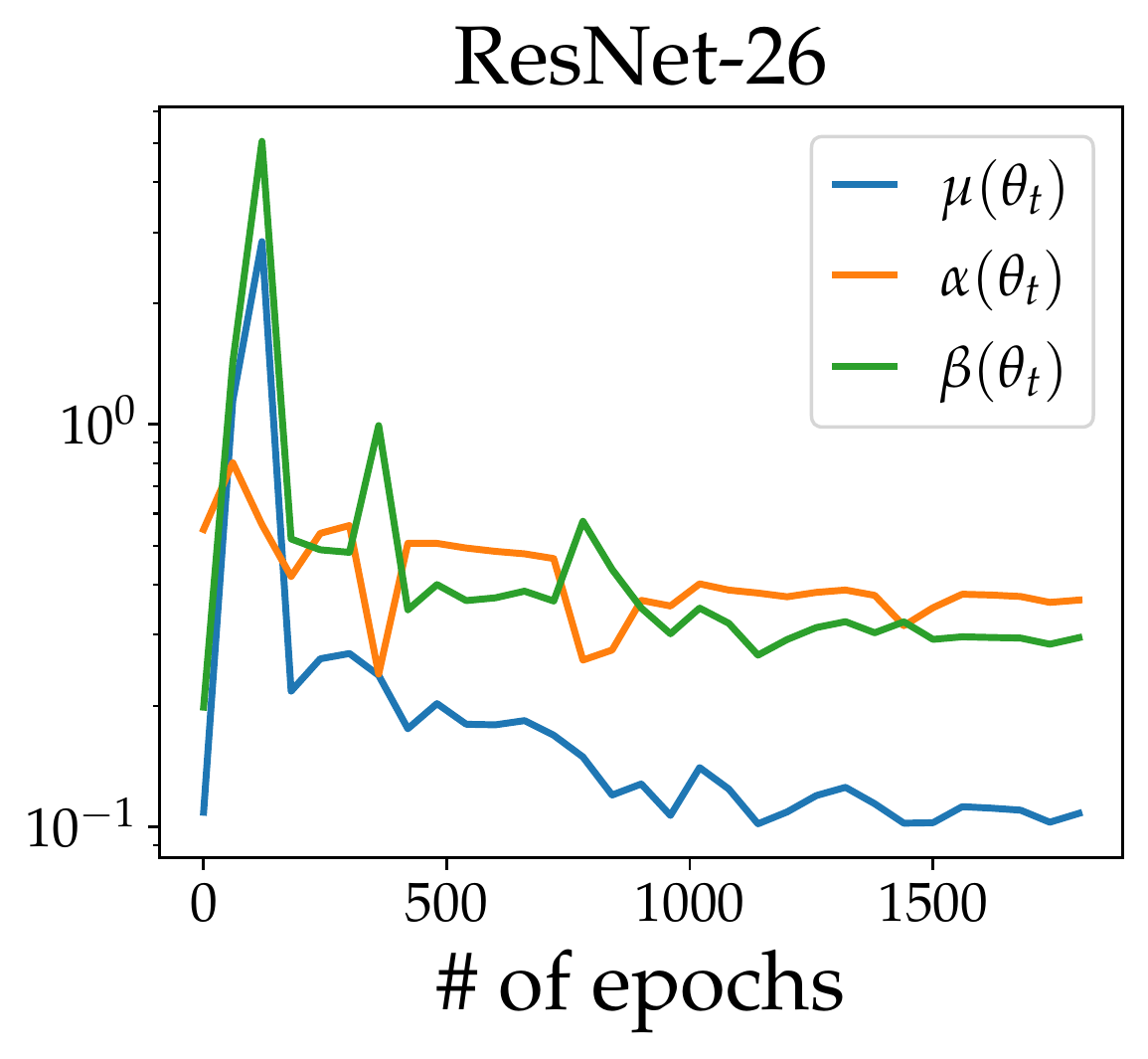}
  \includegraphics[width=0.24\textwidth]{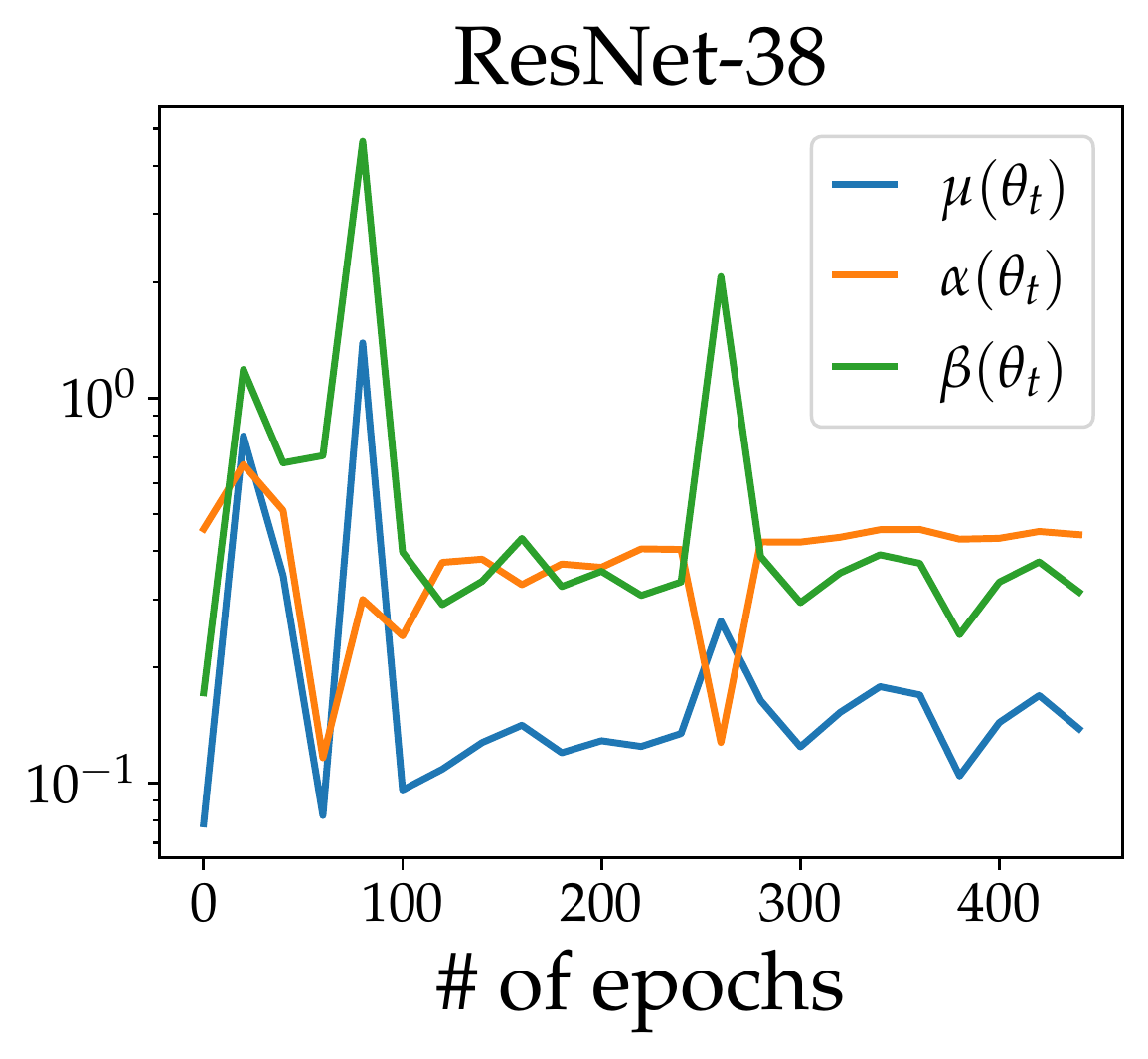}
  \includegraphics[width=0.24\textwidth]{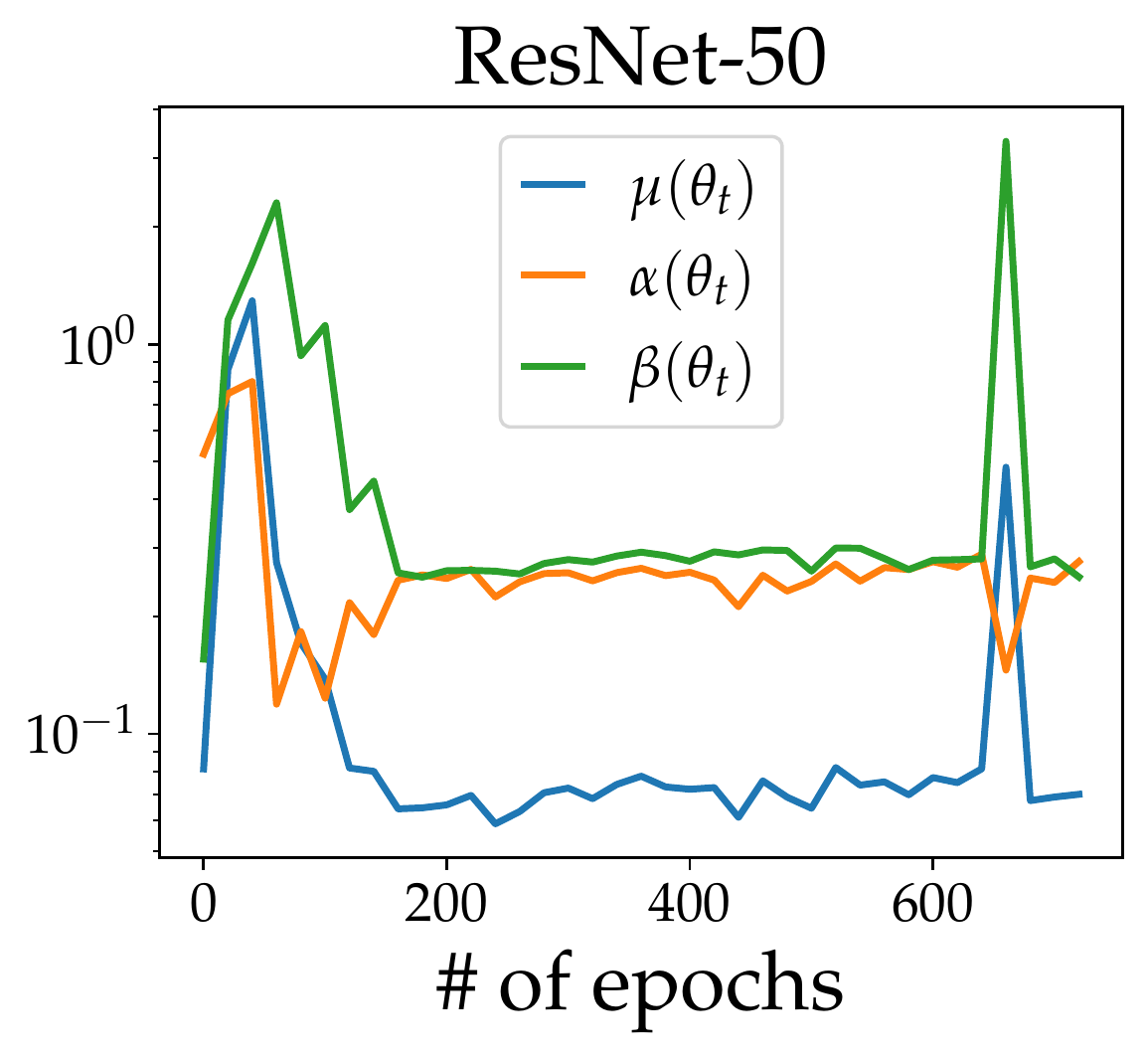}
  \includegraphics[width=0.24\textwidth]{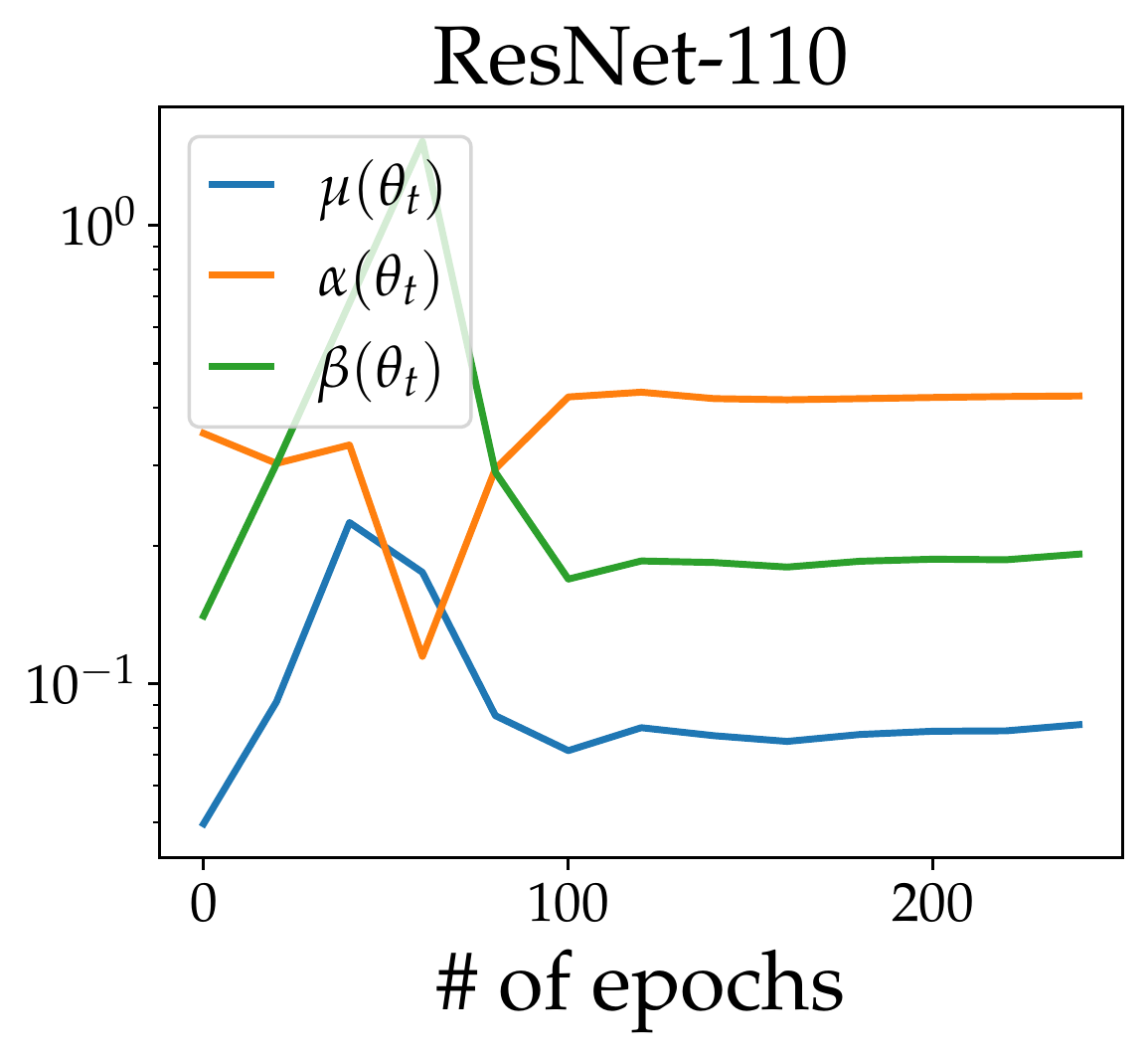}

  \includegraphics[width=0.25\textwidth]{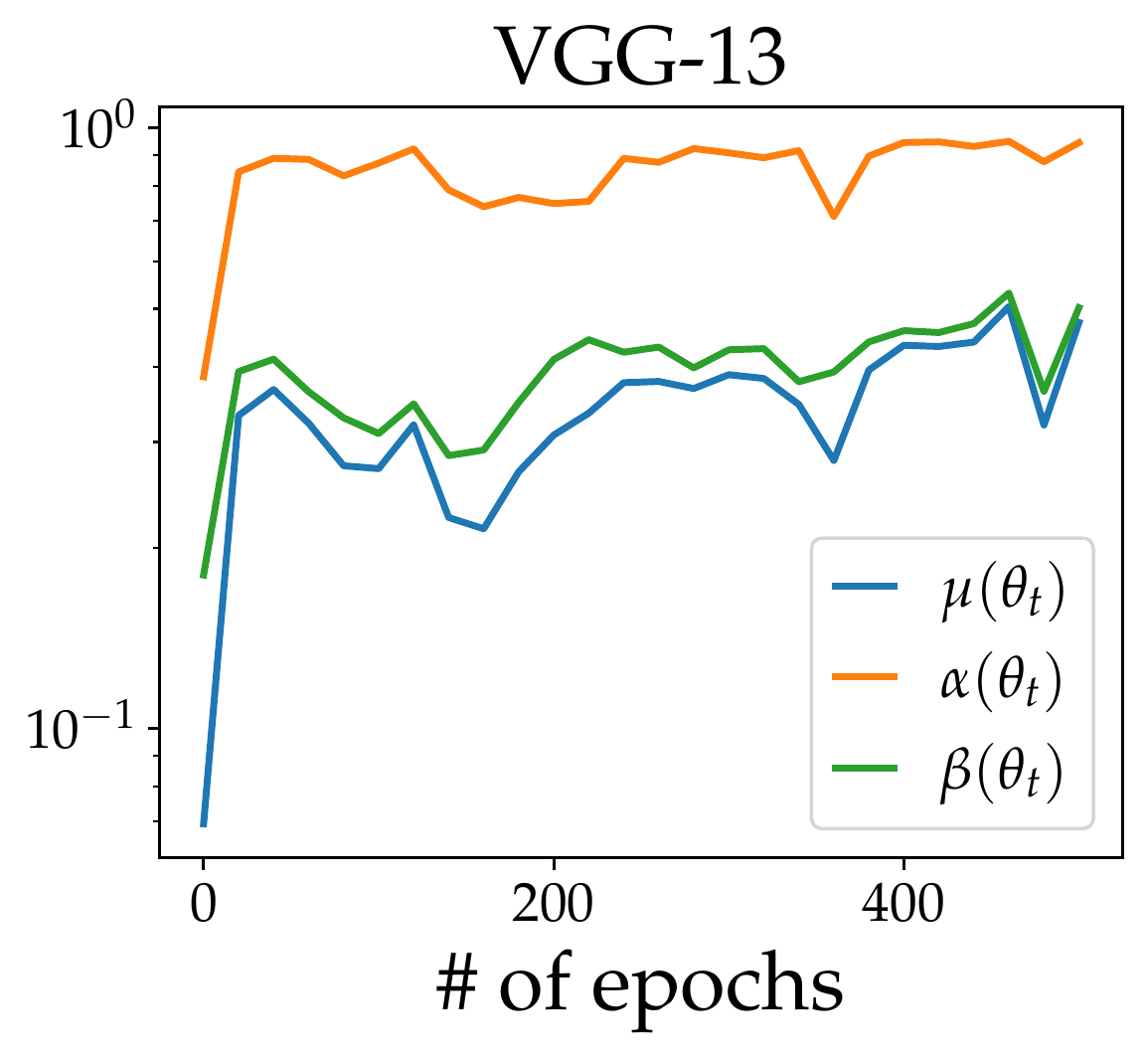}
  \includegraphics[width=0.25\textwidth]{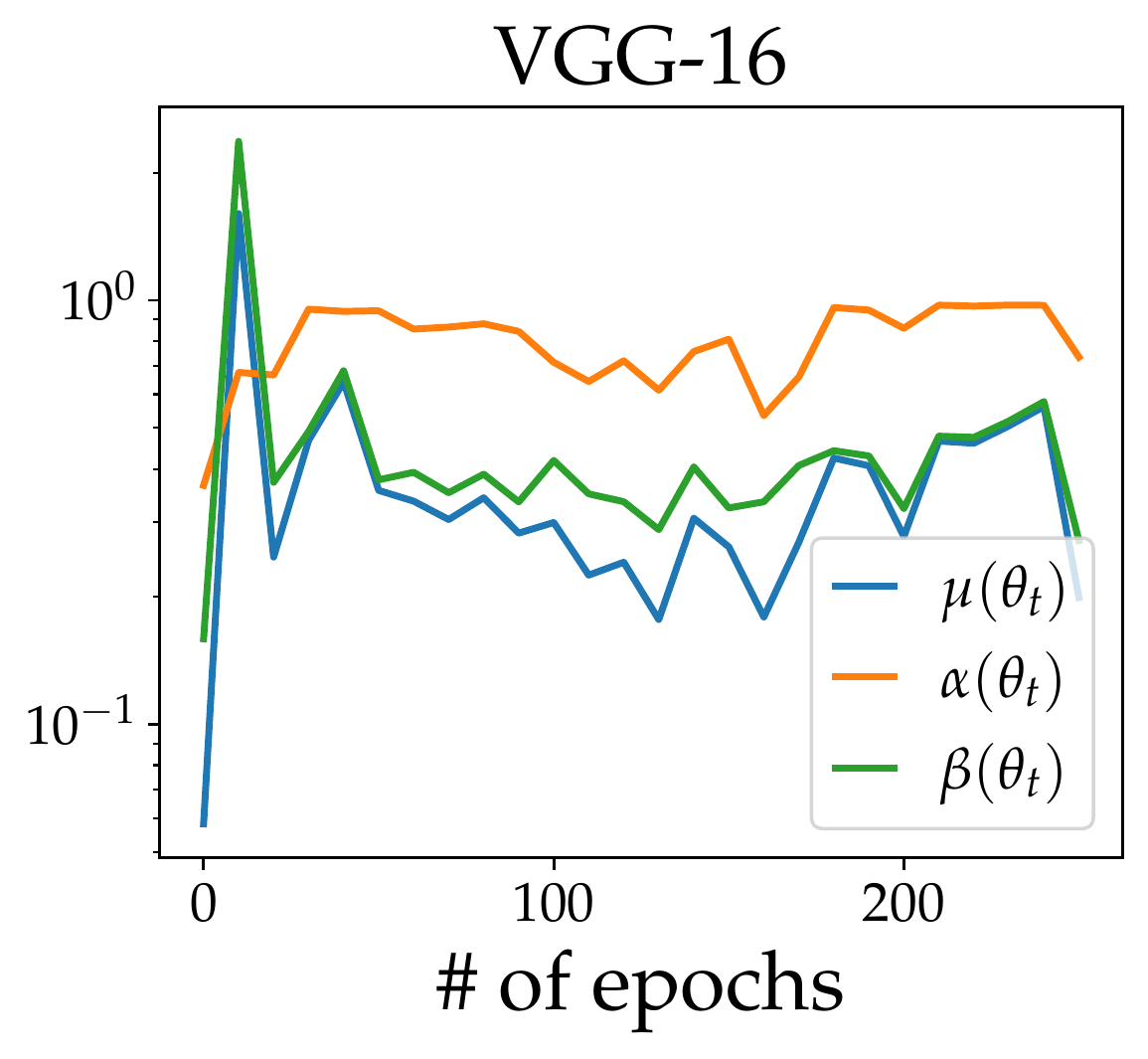}
  \includegraphics[width=0.25\textwidth]{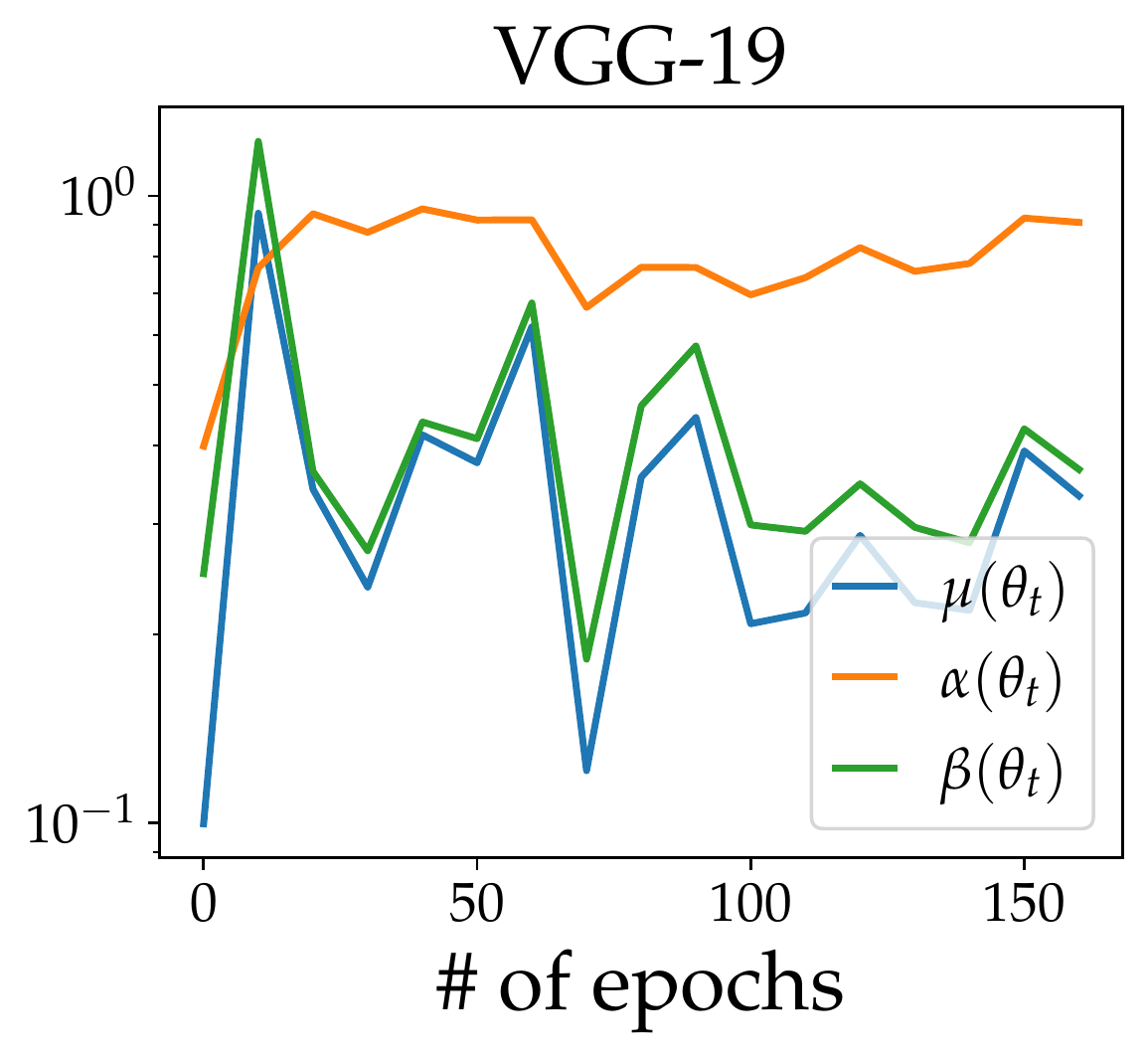}
\vspace*{-.5em}
\caption{\small \textbf{The alignment factors during the SGD training for classifying a two-class subset of CIFAR-10 with VGG nets and 
ResNets .}
In this experiment, we only pick the data of the class ``0'' and ``1'' from CIFAR-10 to train the model. It is shown that for this simpler problem, the alignment factors are significantly larger than the full-data case (see Figure  \ref{fig: large-scale-1} and \ref{fig: cifar10-large-app}) for all the models examined, regardless how over-parameterized the model is. In particular, we surprisingly observe that $\alpha(\theta)$'s are pretty close to $1$ for all VGG nets, suggesting a complete alignment between the noise covariance and Fisher matrix. We leave to the detailed analysis of this unreasonable alignment to future work. 
}
\label{fig: cifar10-large-app-binary}
\end{figure}

\section{Notations.} \label{sec: notation}
For a vector $v$, let $\|v\|_{p}=(\sum_i v_i^p)^{1/p}$ and $\hv=v/\|v\|_2$. When $p=2$, we omit the subscript for simplicity. For a matrix $A=(a_{i,j})$, denote by $\|A\|_F=(\sum_{i,j} a_{i,j}^2)^{1/2}$  the Frobenius norm. For a matrix or linear operator $A$, denote by $\{\lambda_i(A)\}_{i\geq 1}$ the eigenvalues of $A$ in a non-increasing order. Let $\SS^{d-1}=\{x\in\RR^d\,|\,\|x\|=1\}, r\SS^{d-1}=\{x\in \RR^d \,|\, \|x\|=r\}$, and  $\tau_{d-1}=\unif(\SS^{d-1})$. For any distribution $\rho$, let $\|f\|^2_{L_2(\rho)}=\EE_{x\sim\rho}[f^2(x)]$.
We use $X\lesssim Y$ to indicate $X\leq CY$ for an absolute  constant $C>0$ and $X\gtrsim Y$ is defined similarly. Moreover, we use $X\sim Y$ to mean $c Y\leq X\geq C Y$ for some absolute constants $c, C>0$.
We also use $C$ to denote an absolute constant, whose value may change from line to line. 
For a vector $a$ and a SPD matrix $W$, let $\|a\|_W=\sqrt{a^TWa}$.

\section{Proof of Proposition \ref{thm: over-para-linear-noise-covariance}}
\label{sec: proof-OLM}

We first need the following lemma.
\begin{lemma}\label{lemma: apd-linear}
Suppose $z\sim\cN(0, S)$. Then, $\EE_{z}[|v^Tz|^2zz^T]=2Svv^TS+\|S^{1/2}v\|^2 S$.
\end{lemma}
\begin{proof}
First assume  $S=I_d$. 
Notice that 
\begin{align*}
\EE[|v^Tz|^2 z_iz_j] &= \EE_z[\sum_{k,l=1}^d v_k v_l z_k z_l z_i z_j] =\begin{cases}
2v_i v_j &  \text{if } i\neq j\\
2 v_i^2 + \sum_{k=1}^d v_k^2 & \text{if } i=j.
\end{cases}
\end{align*}
Therefore, 
\begin{equation}\label{eqn: apd-100}
\EE[|v^Tz|^2 zz^T]=2vv^T + \|v\|^2 I_d
\end{equation}

For general $S$, let $x=S^{-1/2}z$. Then $x\sim\cN(0,I_d)$.  Then we have 
\begin{align*}
\EE[|v^Tz|^2 zz^T] &= \EE[|v^TS^{1/2}x|S^{1/2}zz^{T}S^{1/2}]\\
&=S^{1/2}\EE[|v^TS^{1/2}x|zz^{T}]S^{1/2}\\
&=S^{1/2}\left(2S^{1/2}vv^TS^{1/2}+\|S^{1/2}v\|^2I_d\right)S^{1/2}\\
&=2Svv^TS + \|S^{1/2}v\|^2 S,
\end{align*}
where the third step follows from \eqref{eqn: apd-100}.

\end{proof}
Now we start to prove Proposition \ref{thm: over-para-linear-noise-covariance}.
Let $\nabla F: \RR^p \mapsto \RR^d$ be the Jacobian matrix of $F$. Then $\nabla f(x;\theta) = \nabla F(\theta)^T x$. 
Recall that we assume $y=f(x;\theta^*)$, i.e., there is no label noise. Thus, we have 
\begin{align*}
    \prisk(\theta) &= \half \EE_{x}[|(F(\theta)-F(\theta^*))^Tx|^2]=\half \|u(\theta)\|_S^2\\
    G(\theta) &= \EE_{x}[\nabla F(\theta)^T xx^T\nabla F(\theta)] = \nabla F(\theta)^T S\nabla F(\theta)\\
    H(\theta) &= G(\theta) + (F(\theta)-F(\theta^*))\nabla^2 F(\theta),
\end{align*}
where $u(\theta) = F(\theta) -  F(\theta^*)$. 
For the noise covariance, we  have 
\[
\Sigma_2(\theta) = \nabla \prisk(\theta) \nabla \prisk(\theta)^T = \nabla F(\theta)^T Su(\theta) u(\theta)^TS\nabla F(\theta),
\]
and
\begin{align*}
\Sigma_1(\theta)  &=\EE_{x}[ |F(\theta)^Tx-F(\theta^*)^Tx|^2 \nabla F(\theta)^Tx x^T\nabla F(\theta)]\\
&= \nabla F(\theta)^T \EE_{x}[ |(F(\theta)-F(\theta^*))^Tx|^2 x x^T]\nabla F(\theta)\\
&= \nabla F(\theta)^T (2Su(\theta)u(\theta)^TS+\|u(\theta)\|_S^2 S) \nabla F(\theta)\\
&= 2\nabla F(\theta)^T Su(\theta) u(\theta)^TS \nabla F(\theta) + \|u(\theta)\|_S^2\nabla F(\theta)^T S\nabla F(\theta)\\
&= 2 \nabla \prisk(\theta) \nabla \prisk(\theta)^T + 2\prisk(\theta) G(\theta),
\end{align*}
where the third step follows from Lemma \ref{lemma: apd-linear}. Consequently, $\mu_1(\theta)\geq \mu(\theta)\geq 1$.
\qed

\section{Proof of Lemma \ref{lemma: gen-feature-based-model}}
\label{sec: proof-lemma-feature-based-model}

Recall the definition of  feature-based model: $f(x;\theta)=\sum_{j=1}^m \theta_j\varphi_j(x)=\langle \theta, \Phi(x)\rangle$. Here $\Phi(x)\in\RR^m$ denotes the feature of $x$. In this case, the Hessian and Fisher matrix are both constant but the SGD noise  is still state-dependent. Let $g_i=\Phi(x_i)$ be the  sample feature of $x_i$. Then,
\[ 
\Sigma(\theta) = \frac{2}{n}\sum_{i=1}^n L_i(\theta) g_i g_i^T-\nabla L(\theta) \nabla L(\theta)^T, \quad G(\theta)=H(\theta) =\fn\sumin g_ig_i^T.
\]

\begin{proof}
\vspace*{-.5em}
Noticing $\bchi=\fn\sumin g_i^TGg_i=\|G\|_F^2$, then
\begin{align*}
\tr(\Sigma_1(\theta)G)&=\fn\sumin |g_i^T\theta|^2 \tr(g_ig_i^TG)\\
&=\fn\sumin |g_i^T\theta|^2 \chi_i\geq \frac{\gamma \bchi}{n} \sumin |g_i^T\theta|^2=2\gamma L(\theta)\|G\|_F^2. 
\end{align*}
Thus  $\mu_1(\theta)\geq \gamma$.
In addition, 
\begin{align*}
\mu_2(\theta) &= \frac{\nabla \erisk(\theta)^T G\nabla \erisk(\theta)}{2\erisk(\theta)\|G\|_F^2} \\
&= \frac{\theta^TG^3\theta}{\theta^TG\theta \|G\|_F^2}\leq \frac{\lambda_1^2(G)}{\|G\|_F^2},
\end{align*}
where the last step follows from Lemma \ref{lemma: appendix-matrix-lemma}.
\vspace*{-.5em}
\end{proof}

\section{Proof of Proposition \ref{pro: random-relu}}
\label{sec: proof-rfm-alignemt}

\subsection{Rademacher complexity}
\label{sec: app-radmacher}
We shall use the Rademacher complexity to bound the difference between empirical quantities and  the corresponding population ones. 
\begin{definition}[Rademacher complexity]
Let $\cF$ be a set of functions. The (empirical) Rademacher complexity of $\cF$  is defined as 
$
\erad(\cF) = \EE_{\xi}[\sup_{f\in \cF}\frac{1}{n} \sum_{i=1}^n\xi_i f(z_i)],
$
where $\{\xi_i\}_{i=1}^n$ are \iid random variables satisfying   $\PP(\xi_i=+1)=\PP(\xi_i=-1)=\frac{1}{2}$.  
\end{definition}
In particular, the following classic result will be frequently used, which is a restatement of \cite[Theorem 26.5]{shalev2014understanding}.
\begin{theorem}\label{thm: gen-err-rademacher-complexity}
Consider a function class $\cF$  and assume $|f|\leq B$. Then~for any $\delta\in (0,1)$,  \wp at least $1-\delta$ over the choice of $(z_1,z_2,\dots,z_n)$, we have,
\[
    \sup_{f\in\cF}|\frac{1}{n}\sum_{i=1}^n f(z_i) - \EE_{z}[f(z)]| \leq 4 \erad(\cF) + B\sqrt{\frac{2\ln(2/\delta)}{n}}.
\]
\end{theorem}

\begin{lemma}[Contraction property \cite{ledoux2013probability}]\label{lemma: contraction}
Let $\phi:\RR\mapsto\RR$  be $\beta$-Lispchitz continuous and $\phi\circ\cF=\{\phi\circ f: f\in\cF\}$. Then, 
$
    \erad(\phi\circ\cF)\leq \beta \, \erad(\cF).
$
\end{lemma}

\begin{lemma}\label{lemma: rad-linear-class}
Let $\cF=\{u^Tx: u\in\SS^{d-1}\}$ be the linear class.  Then 
$
    \erad(\cF)\leq \sqrt{\frac{\sum_{i=1}^n \|x_i\|^2}{n^2}}.
$
\end{lemma}

\begin{lemma}\label{lemma: rkhs}
Let $\phi: \cX\mapsto H$ be a feature map and $H$ be a Hilbert space. Define $k(x,y)=\<\phi_x,\phi_y\>_H$ and $\cH=\{f(x)=\langle \phi_x, h\rangle_H \, |\, \|h\|_H\leq 1\}$. Then, $\erad(\cH)\leq \sqrt{\sumin k(x_i,x_i)}/n$.
\end{lemma}
\begin{proof}
By the definition, we have 
\begin{align*}
n \erad(\cH) &= \EE_{\xi}\sup_{\|h\|_H\leq 1} \sumin \xi_i\langle \phi_{x_i}, h\rangle_H\\
&= \EE_{\xi}\sup_{\|h\|_H\leq 1}\langle  \sumin \xi_i\phi_{x_i}, h\rangle_H= \EE_{\xi}\|\sumin \xi_i \phi_{x_i}\|\\
&\stackrel{(i)}{\leq}\sqrt{\EE_{\xi}\|\sumin \xi_i \phi_{x_i}\|^2} = \sqrt{\sumin \<\phi_{x_i},\phi_{x_i}\>_{H}} = \sqrt{\sumin k(x_i,x_i)},
\end{align*}
where $(i)$ follows from the Jensen's inequality.
\end{proof}

For the proof of the above classic lemmas, we refer to \cite{shalev2014understanding}. The following lemma concerns the Rademacher complexity of the product of two function classes.

\begin{lemma}\label{lemma: Rademacher-multiply-space}
Let $\cF$ and $\cG$ be two function classes. Suppose that $\sup_{f\in \cF}\|f\|_{\infty}\leq A$ and $\sup_{g\in \cG}\|g\|_\infty\leq B$. Define $\cF*\cG=\{f(x)g(x): \cX\mapsto\RR \,:\, f\in \cF, g\in \cG\}$. Then, 
$$
\erad(\cF*\cG)\leq (A+B)(\erad(\cF)+\erad(\cG)).
$$
\end{lemma}
\begin{proof}
By the definition of Rademacher complexity, 
\begin{align*}
n \erad(\cF*\cG) &= \EE_{\xi}[\sup_{f\in \cF, g\in \cG} \sumin f(x_i)g(x_i)\xi_i]\\
&=\EE_{\xi}[\sup_{f\in \cF, g\in \cG} \sumin \frac{(f(x_i)+g(x_i))^2}{4}\xi_i - \sumin \frac{(f(x_i)-g(x_i))^2}{4}\xi_i]\\
&\leq \EE_{\xi}[\sup_{f\in \cF, g\in \cG} \sumin \frac{(f(x_i)+g(x_i))^2}{4}\xi_i +\EE_{\xi}[\sup_{f\in \cF, g\in \cG} \sumin \frac{(f(x_i)-g(x_i))^2}{4}\xi_i]\\
&\stackrel{(i)}{\leq} \frac{A+B}{2}\left( \EE_{\xi}[\sup_{f\in \cF, g\in \cG} \sumin (f(x_i)+g(x_i))\xi_i] +\EE_{\xi}[\sup_{f\in \cF, g\in \cG} \sumin (f(x_i)-g(x_i))\xi_i]\right)\\
&\leq (A+B)n(\erad(\cF)+\erad(\cG)),
\end{align*}
where $(i)$ follows from the Lemma \ref{lemma: contraction} and the fact that $t^2/4$ is $(A+B)/2$ Lipschitz continuous since $|f|\leq A, |g|\leq B$.
\end{proof}

\subsection{Results for general random feature models}

Consider a random feature model (RFM):
\[
  f(x;\theta) = \frac{1}{\sqrt{m}}\sumim \theta_j \varphi(x;w_j),
\]
with $\{w_j\}_{j=1}^m\stackrel{iid}{\sim}\pi$ and $x\sim\rho$. Here $\varphi: \cX\times \Omega\mapsto\RR$ is an arbitrary parametric feature function.
 The scaling factor $m^{-1/2}$ is  added only to ease taking the limit, which does not affect the noise structure. 
The associated kernel and kernel operator are given by   $k(x,x')=\EE_{w\sim\pi}[\sigma(w^Tx)\sigma(w^Tx')]$  and $\cK: L_2(\rho)\mapsto L_2(\rho),
    \cK u = \EE_{x'\sim\rho}[k(\cdot,x') u(x')].
$

In this case, 
$
g_i = m^{-1/2}(\varphi(x_i;w_1),\varphi(x_i;w_2),\dots,\varphi(x_i;w_m))^T\in\RR^m. 
$
 Define the empirical kernel matrix: $\hat{k}(x,x')=\frac{1}{m}\sum_{s=1}^m \varphi(x;w_s)\varphi(x';w_s)$. Then $g_i^Tg_j=\hat{k}(x_i,x_j)$ and as $n,m\to \infty$, we have
\begin{align}
\chi_i &= \fn \sumjn (g_i^Tg_j)^2 = \fn \sumjn \hk^2(x_i,x_j)\to \EE_{x}[k^2(x_i,x)].
\end{align}
Therefore, in the limit, by our assumption $\chi_i\geq \chi:=\inf_{x'}\EE_{x}[k^2(x',x)]$ for any $i\in [n]$. The remaining issue is to transfer this to a non-asymptotic one.

We first make the following assumption on the feature function.
\begin{assumption}\label{assump: feature-function}
Let $\Psi=\{\varphi(x;\cdot) \,|\, x\in \cX\}$. 
Assume $\sup_{x\in \cX,w\in\Omega}|\varphi(x;w)|\leq b$ and $\erad(\Psi)\leq b/\sqrt{n}$.
\end{assumption}
Note that the assumption $\erad(\Psi)=O(n^{-1/2})$ is satisfied by the popular feature function $\varphi(x;w)=\sigma(w^Tx)$ with $\sigma:\RR\mapsto\RR$ being a Lipschitz activation function.

\begin{proposition}\label{pro: rfm-alignment}
Suppose $m\geq n$. 
Let $\chi(x)=\EE_{x'\sim\rho}[k^2(x,x')]$ and $\bchi=\EE_{x\sim\rho}[\chi(x)]$.  For any $\delta\in (0,1/e)$, \wp larger than $1-\delta$ over the sampling of data and random features, we have $\mu_1(\theta)\geq \inf_x \frac{\chi(x)}{\bchi}-\epsilon_n$, $\mu_2(\theta)\leq \tau(\cK)+\epsilon_n$, and 
$
\mu(\theta)\geq \inf_x \frac{\chi(x)}{\bchi}-\tau(\cK) -2\epsilon_n,
$
where $\tau(\cK)=\lambda_1^2(\cK)/\sum_j \lambda_j^2(\cK)$ and $\epsilon_n=Cb^2\sqrt{\frac{\log(1/\delta)}{n}}$.
\end{proposition}
By this proposition, ensuring the  alignment property only needs $\inf_{x}\chi(x)/\bchi>0$.  This condition is very mild and satisfied by most popular kernels, e.g., the dot-product kernel, which appears naturally in the analysis of neural nets \cite{jacot2018neural}.  Next we proceed to prove this proposition.

We first need the following lemmas.

\begin{lemma}\label{lemma: rkhs-unitball}
Let $\Phi_1=\{k(\cdot,x)\,|\, x\in\cX\}$. Then, $\erad(\Phi_1)\leq 1/\sqrt{n}$.
\end{lemma}
\begin{proof}
Denote by $\cH_k$ the reproducing kernel Hilbert space (RKHS). Then, by the Moore-Aronsajn theorem \cite{aronszajn1950theory},
\[
    \|k(\cdot,x)\|_{\cH_k}^2 = \langle k(\cdot,x), k(\cdot,x)\rangle_{\cH_k} = k(x,x)\leq 1.
\]
Therefore, $\Phi_1$ is a subset of the unit ball of $\cH_k$, for which 
it is well-known that the Rademacher complexity  is bounded by $\sqrt{\sumin k(x_i,x_i)}/n\leq \sqrt{1/n}$. 
\end{proof}

\begin{lemma}\label{lemma: uniform-approx-kernel}
For any $\delta\in (0,1)$, \wp at least $1-\delta$, we have 
\[
    \sup_{x,x'\in \cX} |k(x,x')-\hk(x,x')|\leq b^2 \sqrt{\frac{\log(1/\delta)}{m}}.
\]
\end{lemma}
\begin{proof}
Let $\Psi_2 = \{\varphi(x;\cdot)\varphi(x';\cdot)\,|\, x,x'\in \cX\}$. Notice that 
\[
\varphi(x;w)\varphi(x';w)=\frac{(\varphi(x;w)+\varphi(x';w)^2)}{4} - \frac{(\varphi(x;w)-\varphi(x';w)^2)}{4}.
\]
Then by the same argument in the proof of Lemma \ref{lemma: Rademacher-multiply-space}, we can obtain 
$\erad(\Psi_2)\lesssim b\erad(\Psi)\lesssim b^2/\sqrt{m}$, where the last step follows from Lemma \ref{assump: feature-function}. Then, applying Theorem \ref{thm: gen-err-rademacher-complexity}, we complete the proof.
\end{proof}

\paragraph*{Prove the first part of Proposition \ref{pro: rfm-alignment}}
Notice that 
\begin{align}\label{eqn: rfm-1}
\notag |\chi(x_i)-\chi_i| &= |\EE_{x}[k^2(x,x_i)]-\fn\sumjn k^2(x_j,x_i)| + |\fn\sumjn k^2(x_j,x_i) - \fn\sumjn \hk(x_j,x_i)^2|\\
&\leq \sup_{x'}|\EE_{x}[k^2(x,x')]-\fn\sumjn k^2(x_j,x')| +  \sup_{x,x'} |k^2(x,x')-\hk^2(x,x')|
\end{align}
Let $\Phi_2=\{k^2(x,\cdot)\,|\, x\in\cX\}$. Since $\sup_{x,x'}|k(x,x')|\leq 1$, the Ledoux-Talagrand inequality (Lemma \ref{lemma: contraction}) implies that $\erad(\Phi_2)\leq 0.5\erad(\Phi_1)\leq 1/\sqrt{n}$, where the last inequality follows from Lemma \ref{lemma: rkhs-unitball}. Then, applying Theorem \ref{thm: gen-err-rademacher-complexity}, for any $\delta\in (0,1/e)$, we have \wp $1-\delta$ that 
\begin{align}\label{eqn: no-name}
\notag    \sup_{x'}|\EE_{x}[k^2(x,x')]&-\fn\sumjn k^2(x_j,x')|\lesssim 2\erad(\Phi_2) + \sqrt{\frac{\log(2/\delta)}{n}}\\
    &\leq \frac{2}{\sqrt{n}} + \sqrt{\frac{\log(2/\delta)}{n}}\lesssim \sqrt{\frac{\log(2/\delta)}{n}}.
\end{align}

Let $\Delta = \sup_{x,x'}|k(x,x')-\hk(x,x')|$. Therefore, 
\[
    |k^2(x,x')-\hk^2(x,x')| =  |k(x,x')+\hk(x,x')||k(x,x')-\hk(x,x')|\leq (2k(x,x')+\Delta)\Delta\lesssim \Delta.
\]
Applying Lemma \ref{lemma: uniform-approx-kernel} and together with  \eqref{eqn: no-name}, we have 
\[
|\chi(x_i)-\chi_i|\lesssim \sqrt{\frac{\log(1/\delta)}{n}} + b\sqrt{\frac{\log(1/\delta)}{m}}=:\epsilon_n.
\]

On the other hand, by the Hoeffding’s inequality, \wp at least $1-\delta$, we have 
\begin{align*}
\tilde{\chi} = \fn \sumin \chi_i\leq \fn \sumin \chi(x_i) + C\epsilon_n\leq \EE_{x}[\chi(x)] + \sqrt{\frac{\log(1/\delta)}{n}}+C\epsilon_n \leq \bchi + C \epsilon_n.
\end{align*}
By Proposition \ref{pro: alignment-linearized-nonlinear-model}, we have
\[
\mu_1(\theta) \geq \frac{\inf_i \chi_i}{\tilde{\chi}} \geq \frac{\inf\chi(x)-C\epsilon_n}{\bchi+C\epsilon_n} \geq \frac{\inf\chi(x)}{\bchi} -C\epsilon_n.
\]
\qed

The above proves the first part of Proposition \ref{pro: rfm-alignment}. We now turn to the second part.  We will frequently use the following McDiarmid's inequality.
\begin{theorem}[McDiarmid's inequality]\label{thm: McDiarmid}
Let $X_1,\dots,X_n$ are \iid random variables. 
Assume for all $i\in [n]$ and $x_1,\dots,x_n, \tilde{x}_i\in \cX$ that 
\[
|f(x_1,\dots,x_{i-1},x_i,x_{i+1},\dots,x_n) - f(x_1,\dots,x_{i-1},\tilde{x}_i,x_{i+1},\dots,x_n)|\leq D_i.
\]
Let et $\sigma^2:=\frac{1}{4}\sum_{i=1}^n D_i^2$. Then, for any $\delta\in (0,1)$, \wp at least $1-\delta$ over the sampling of $X_1,\dots,X_n$, we have 
\[
    |f(X_1,\dots,X_n)-\EE[f]|\leq \sqrt{2\log(2/\delta)} \sigma.
\]
\end{theorem}

We define two matrices
\begin{itemize}
\item The kernel matrix: $K=(K_{i,j})\in\RR^{n\times n}$ with $K_{i,j}=\fn k(x_i,x_j)$;
\item The approximate kernel matrix: $ \hK=(\hK_{i,j})\in\RR^{n\times n}$ with 
 \[
 \hK_{i,j}=\frac{1}{nm}\sum_{s=1}^m \varphi(x_i;w_s)\varphi(x_j;w_s)=:\fn\hat{k}(x_i,x_j).
\]
 \end{itemize}
 Note that the approximate kernel matrix is the normalized Fisher matrix, i.e., $\hK=G/n$. We will frequently use the following inequality:
 \begin{align}\label{eqn: rfm-l2-approximation-err}
\EE_{w}|\hat{k}(x,x') - k(x,x')|^2 \leq \frac{\EE_{w}[\varphi^2(x;w)\varphi^2(x';w)]}{m}\leq \frac{b^4}{m}.
 \end{align}

\paragraph*{Bounding the largest eigenvalue.}
\begin{lemma}\label{lemma: approximate-kernel-eigen}
There exists a constant $c>0$ such that for any $\delta\in (0,1)$, \wp at least $1-\delta$ over the sampling of random features, we have 
$\lambda_1(\hK)\leq \lambda_1(K)+C b^2\sqrt{\log(2/\delta)/m}$.
\end{lemma}
\begin{proof}
Let $\Delta(w_1,\dots,w_m) = \|K-\hK\|_F$. 
By  Jensen's inequality, we have 
\begin{align*}
\EE[\Delta] &\leq \sqrt{\EE[\Delta^2]}\leq \sqrt{\EE[\frac{1}{n^2}\sum_{i,j=1}^n |k(x_i,x_j)-\hat{k}(x_i,x_j)|^2]}\lesssim\sqrt{\frac{b^4}{m}},
\end{align*}
where the last inequality follows from \eqref{eqn: rfm-l2-approximation-err}.
Denote by $\hK'$  the approximate kernel matrix associated with $(w_1,\dots,w_s',\dots,w_m)$. Then,
\begin{align*}
D_s=|\Delta(w_1,&\dots,w_j,\dots,w_m) - \Delta(w_1,\dots,w_j',\dots,w_m)|\leq \|\hK-\hK'\|_F\\
&=\sqrt{\frac{1}{n^2}\sum_{i,j=1}^n \frac{1}{m^2}(\varphi(x_i;w_s)\varphi(x_j;w_s)-\varphi(x_i;w_s')\varphi(x_j;w_s'))^2}\lesssim \frac{b^2}{m}.
\end{align*}
Therefore $\sigma^2=\frac{1}{4}\sum_{s=1}^m D_s^2\lesssim b^4/m$. By McDiarmid's inequality (Theorem \ref{thm: McDiarmid}), for any $\delta\in (0,1)$, \wp at least $1-\delta$ over the sampling of random features, we have 
\[
    \Delta \lesssim \EE[\Delta] + d\sqrt{\frac{\log(2/\delta)}{m}}\leq b^2\sqrt{\frac{\log(2/\delta)}{m}}.
\]
Plugging the above estimate to the Weyl's inequality: $\lambda_1(\hK)\leq \lambda_1(K) + \|K-\hK\|_2$, we complete the proof.
\end{proof}

\begin{lemma}\label{lemma: variation-principle}
Suppose  $k(\cdot,\cdot)$ to be positive semi-definite kernel and let  $\phi: \cX\mapsto \cH$ be a feature map satisfying $k(x,y)=\<\phi_x, \phi_y\>_{\cH}$. Then, 
\begin{equation}
\lambda_1(\cK) = \sup_{\|h\|_{\cH}= 1} \EE_{x}[\< h, \phi_x\>_{\cH}^2].
\end{equation}
\end{lemma}
\begin{proof}
By the variational principle of eigenvalues, we have 
\begin{align*}
\lambda_1(\cK) &= \sup_{\|u\|_{L_2(\rho)}=1} \EE_{x,y}[k(x,y)u(x)u(y)] = \sup_{\|u\|_{L_2(\rho)}=1} \EE_{x,y}[\< \phi_x, \phi_y\>_{\cH}u(x)u(y)]\\
&= \sup_{\|u\|_{L_2(\rho)}=1}\|\EE_x[u(x)\phi_x]\|_{\cH}^2 = \sup_{\|u\|_{L_2(\rho)}=1}\sup_{\|h\|_{\cH}=1}\langle h, \EE_x[u(x)\phi_x]\rangle^2_{\cH}\\
&= \sup_{\|h\|_{\cH}=1}\sup_{\|u\|_{L_2(\rho)}=1}\EE_{x}[u(x)\langle h, \phi_x \rangle_{\cH}]^2= \sup_{\|h\|_{\cH}=1}\EE_{x}[\langle h, \phi_x \rangle_{\cH}^2].
\end{align*}
\end{proof}

\begin{lemma}\label{lemma: kernel-eigen}
For any $\delta\in (0,1/e)$, \wp at least $1-\delta$ over the sampling of data, we have 
\[
    \lambda_1(K)\leq \lambda_1(\cK) + C\sqrt{\frac{\log(2/\delta)}{n}}.
\]
\end{lemma}
\begin{proof}
By Lemma \ref{lemma: variation-principle}, we have
\begin{equation}
\begin{aligned}
\lambda_1(\cK) = \sup_{\|h\|_{\cH}=1}\EE_{x}[\langle \phi_x, h\rangle_{\cH}^2],\quad \lambda_1(K) = \sup_{\|h\|_{\cH}=1}\hEE_{x}[\langle \phi_x, h\rangle_{\cH}^2].
\end{aligned}
\end{equation}
Let $\cF_1=\{ f(x):=\langle h,\sigma_{x}\rangle_{\cH} \,|\, \|h\|_{\cH}\leq 1\}$ and $\cF_2=\{f^2 \, |\,f\in \cH_1\}$.  Then, we have:  (1) for any $f\in \cH_1$, $|f(x)|\leq \|h\|_{\cH}\|\sigma_{x}\|_{\cH}=\sqrt{k(x,x)}\lesssim 1$; (2) $\erad(\cF_1)=\sqrt{\sumin k(x_i,x_i)}/n\lesssim 1/\sqrt{n}$ by Lemma \ref{lemma: rkhs}. Using contraction inequality (Lemma \ref{lemma: contraction}), we have 
\[
\erad(\cF_2)\lesssim \erad(\cF_1)\lesssim 1/\sqrt{n}.
\]
Using Theorem \ref{thm: gen-err-rademacher-complexity}, for any $\delta\in (0,1/e)$, \wp larger than $1-\delta$ over the sampling of data, we have
\begin{align}\label{eqn: kernel-matrix-largest-eig}
\sup_{\|h\|_{\cH}\leq 1}|\hEE[\langle h,\sigma_{x_i}\rangle^2] - \EE[\<h,\sigma_x \>^2]\lesssim \erad(\cF_2) +  \sqrt{\frac{\log(1/\delta)}{n}}\leq \sqrt{\frac{\log(1/\delta)}{n}}.
\end{align}

For any $\varepsilon>0$, let $\hh\in\cH$ such that $\lambda_1(K)\leq \hEE[\langle \phi_x,\hh\rangle_{\cH}^2]+\varepsilon$. Then, using \eqref{eqn: kernel-matrix-largest-eig}, we have 
\begin{align*}
\lambda_1(K)&\leq \hEE[\langle \phi_x,\hh\rangle_{\cH}^2]+\varepsilon= \EE[\langle \phi_x,\hh\rangle_{\cH}^2] + (\hEE[\langle \phi_x,\hh\rangle_{\cH}^2]-\EE[\langle \phi_x,\hh\rangle_{\cH}^2])+\varepsilon\\
&\lesssim \lambda_1(\cK) + \sqrt{\frac{\log(1/\delta)}{n}} + \varepsilon.
\end{align*}
Taking $\varepsilon\to 0$ completes the proof. 
\end{proof}

\begin{lemma}\label{lemma: largest-eigen-Gram}
For any $\delta\in (0,1)$, \wp at least $1-\delta$ over the sampling of data and random features, we have 
\[
\lambda_1(\hK)\leq \lambda_1(\cK) +  C\left(\sqrt{\frac{\log(1/\delta)}{n}}+b^2\sqrt{\frac{\log(2/\delta)}{m}}\right).
\]
\end{lemma}
\begin{proof}
The conclusion directly follows from the combination of Lemma  \ref{lemma: kernel-eigen} and \ref{lemma: approximate-kernel-eigen}.
\end{proof}

\paragraph*{Bounding the Frobenius norm.}
The following lemma provide a lower bound of the Frobenius norm of the approximate kernel matrix.
\begin{lemma}\label{lemma: Gram-Frobenius-norm}
For any $\delta\in (0,1/e)$, \wp at least $1-\delta$ over the sampling of data and random features, we have 
\[
\|\hK\|_F^2\geq \sum_{i}\lambda_i^2(\cK) - C \left(b^2\sqrt{\frac{\log(1/\delta)}{m}}+\sqrt{\frac{\log(1/\delta)}{n}}\right).
\]
\end{lemma}
\begin{proof}
Denote by $\{(\lambda_s, v_s)\}_{s\geq 1}$ the eigen-pairs of the kernel $k(\cdot,\cdot)$, and thus $k(x,x')=\sum_{s}\lambda_s v_s(x)v_s(x')$. Let $x,x'$ be independently drawn from $\rho$. Then,
\begin{align}\label{eqn: square-eigen}
\notag    \EE[k^2(x,x')] &= \sum_{s_1,s_2}\lambda_{s_1}\lambda_{s_2} \EE[v_{s_1}(x)v_{s_2}(x)v_{s_1}(x')v_{s_2}(x')]\\
\notag    &=\sum_{s_1,s_2}\lambda_{s_1}\lambda_{s_2} \EE[v_{s_1}(x)v_{s_2}(x)]\EE[v_{s_1}(x')v_{s_2}(x')]\\
    &=\sum_{s_1,s_2}\lambda_{s_1}\lambda_{s_2}\delta_{s_1,s_2}=\sum_{s}\lambda_s^2.
\end{align}
Using the above equality, we have
\begin{align}\label{eqn: K-fro-exp}
\notag \EE[\|K\|_F^2] &= \frac{1}{n^2}\sum_{i,j=1}^n \EE[k(x_i,x_j)^2]=\fn\EE[k^2(x,x)]+\frac{1}{n^2}\sum_{i\neq j}\EE[k^2(x_i,x_j)]\\
&= \sum_s \lambda_s^2 + \frac{\EE_x[k^2(x,x)]-\EE_{x,x'}[k^2(x,x')]]}{n}\geq \sum_s \lambda_s^2.
\end{align}
Here, the last inequality is due to $\EE[k^2(x,x)]\geq \EE[k^2(x,x')]$ as explained as follows.
Notice that for any $x,x'\in \cX$,
 $k^2(x,x')\leq k(x,x)k(x',x')$ (See, e.g., \cite[Lemma 35]{sejdinovic2012rkhs}). Therefore, $\EE[k^2(x,x')]\leq \EE[k(x,x)]\EE[k(x',x')]=(\EE[k(x,x)])^2\leq \EE[k^2(x,x)]$. The last inequality follows from that $(\EE[X])^2\leq \EE[X^2]$ holds for any random variable $X$.

Let $K'$ be the kernel matrix corresponding to the $(x_1,\dots,\tilde{x}_i,\dots,x_n)$. Then, 
\begin{align*}
D_i &=|\|K\|_F^2 - \|K'\|_F^2| \\
&= \big|\frac{1}{n^2}\sum_{j\neq i} (k^2(x_i,x_j) -  k^2(\tilde{x}_i,x_j)) + \frac{1}{n^2} (k^2(x_i,x_i) - k^2(\tilde{x}_i, \tilde{x}_i))\big|\lesssim \frac{1}{n}.
\end{align*}
Therefore, $\sigma^2=\frac{1}{4}\sumin |D_i|^2\lesssim 1/n$. Using Theorem \ref{thm: McDiarmid} and Eq.~\eqref{eqn: K-fro-exp}, we have \wp at least $1-\delta$ over the sampling of data that 
\begin{equation}\label{eqn: K-pro-hp}
\|K\|_F^2\geq \EE[\|K\|_F^2] - C\sqrt{\frac{\log(1/\delta)}{n}} \geq  \sum_{s}\lambda_s^2 - C\sqrt{\frac{\log(1/\delta)}{n}}.
\end{equation}
In addition, following the proof of Lemma \ref{lemma: approximate-kernel-eigen}, we have for any $\delta\in (0,1/e)$, \wp larger than $1-\delta$ over the sampling of random features that 
$
\|\hK-K\|_F\leq d\sqrt{\log(2/\delta)/m}. 
$
Thus, 
\[
|\|\hK\|_F^2-\|K\|_F^2|\leq (\|\hK\|_F+\|K\|_F|)(\|\hK\|_F-\|K\|_F|)\lesssim b^2\|\hK-K\|_F\leq  b^2 \sqrt{\log(2/\delta)/m}.
\]
Combing with \eqref{eqn: K-pro-hp}, we complete the proof.
\end{proof}

\paragraph*{Prove the second part of Proposition \ref{pro: rfm-alignment}.}
Recall $\epsilon_n = \sqrt{\log(1/\delta)/n} + b^2 \sqrt{\log(1/\delta)/m}$.
Combining Lemma \ref{lemma: largest-eigen-Gram} and \ref{lemma: Gram-Frobenius-norm}, we have 
\begin{align}
\mu_2(\theta)=\frac{\lambda_1^2(\hK)}{\|\hK\|_F^2}\leq \frac{(\lambda_1(\cK)+C\epsilon_n)^2}{\sum_i \lambda_i^2(\cK)-C \epsilon_n}\leq \frac{\lambda_1^2(\cK)}{\sum_i \lambda_i^2(\cK)} +C\epsilon_n.
\end{align}

\subsection{Proof of Proposition \ref{pro: random-relu}}

Denote by $\tau_{d-1}$ the uniform distribution over the unit sphere $\SS^{d-1}=\{x\in \RR^d: \|x\|_2=1\}$.
According to \cite{bach2017breaking}, the eigenfunctions of $\cK$ are the spherical harmonics. 
In particular, the eigenfunction corresponding to the largest eigenvalue is the first spherical harmonics: $Y_1(x)\equiv 1$.  Therefore, 
\[
    \lambda_1(\cK)=\lambda_1(\cK) Y_1(x) = \EE_{x'\sim\tau_{d-1}}[\kappa(x^Tx')Y_1(x')] = \EE_{x'\sim\tau_{d-1}}[\kappa(x'_1)],
\]
where the last equality uses the rotational symmetry of $\tau_{d-1}$. Moreover, by \eqref{eqn: square-eigen}, $\sum_i \lambda_i^2(\cK)=\EE_{x,x'}[\kappa^2(x^Tx')]=\EE_{x}[\kappa^2(x_1)]$.

For the ReLU activation function, 
\cite{cho2009kernel} shows that $k(x,x')=\kappa(x^Tx')$ with 
\[
\kappa(z) = \frac{\sqrt{1-z^2}+(\pi-\arccos(z))z}{2\pi}.
\]
The eigenvalues of this kernel has been derived in \cite{bach2017breaking,wu2022spectral}. Specifically, we have 
\[
\lambda_1(\cK)\sim 1,\quad  \sum_{i=2}^\infty \lambda_i^2(\cK) \sim \frac 1 d.
\]
Hence, 
\[
\frac{\lambda_1^2(\cK)}{ \sum_{i=1}^\infty \lambda_i^2(\cK)}\sim \frac{1}{1+\frac{1}{d}} \leq  1 - \frac{C}{d}.
\]
Then, applying Proposition \ref{pro: rfm-alignment}, we complete the proof. 
\qed

\section{Proofs of Section \ref{sec: linear-stability}}
\label{sec: proof-linear-stability}

We first need the following technical lemma.

\begin{lemma}\label{lemma: appendix-matrix-lemma}
For $a,b>0$, let 
$
g(a,b,\theta)=-a \frac{\theta^TH^2\theta}{\theta^TH\theta}+b\frac{\theta^T_tH^3\theta}{\theta^TH\theta}.
$
Then, $\inf_{\theta} g(a,b,\theta)\geq- a^2/(4b)$.
\end{lemma}
\begin{proof}
Let $u=H^{1/2}\theta/\|H^{1/2}\theta\|$ and $H=\sum_{j} \lambda_j e_j e_j^T$ the eigen-decomposition of $H$. 
Suppose $u=\sum_j s_j e_j$. Then $\sum_j s_j^2=1$ and 
\begin{align*}
-a \frac{\theta^TH^2\theta}{\theta^TH\theta}+b\frac{\theta^T_tH^3\theta}{\theta^TH\theta} &= - a u^TH u + bu^TH^2u=\sum_{j} (b\lambda_j^2-a\lambda_j) s_j^2\\
&\geq \inf_{\lambda\geq 0} (b\lambda^2-a\lambda)=\inf_{\lambda\geq 0} \left(b(\lambda-\frac{a}{2b})^2-\frac{a^2}{4b}\right)\geq -\frac{a^2}{4b}.
\end{align*}
\end{proof}

We first consider GD,  for which the stability only imposes the flatness constraint: $\lambda_1(H)\leq 2/\eta$. This means that the flatness seen by GD is only the largest eigenvalue of Hessian. 
\begin{lemma}\label{lemma: GD-factor}
(1) $\inf_\theta r(\theta)\geq 0$; (2) $\sup_\theta r(\theta)\leq 1$ if $\eta\leq 2/\lambda_1(H)$.
\end{lemma}
\begin{proof}
Recall that $
r(\theta) = 1-2\eta\frac{\theta^TH^2\theta}{\theta^TH\theta}+\eta^2 \frac{\theta^T_tH^3\theta}{\theta^TH\theta}$.
Let $u=H^{1/2}\theta/\|H^{1/2}\theta\|$. Then we have 
\begin{align*}
r(\theta) &=1-2\eta u^T Hu + \eta^2 u^TH^2u = \|\eta Hu-u\|^2\geq 0.
\end{align*} 
Let $u=\sum_{j} s_j e_j$ with $\{e_j\}_{j}$ being the eigenvectors of $H$. Then $\sum_{j}s_j^2=1$ due to $\|u\|=1$.
By the assumption, $(\eta\lambda_j-1)^2\leq 1$ for all $j\in [n]$. Then, 
\begin{align}
r(\theta) = \|\sum_{j} \eta \lambda_j s_j e_j - \sum_j s_j e_j\|^2=\sum_{j} (\eta \lambda_j-1)^2s_j^2\leq \sum_j s_j^2=1.
\end{align}
\end{proof}

\paragraph*{Proof of Proposition \ref{pro: flatness-control-mu1-alignment}}
Using $\mu_1(\theta)\geq \bmu_1$ and $\Sigma(\theta)=\Sigma_1(\theta)-\Sigma_2(\theta)$, we have 
\begin{align*}
\nu(\theta)&=\frac{1}{2B}\tr(H\Sigma(\theta)) = \frac{1}{2B}\tr(H\Sigma_1(\theta))- \frac{1}{2B}\tr(H\Sigma_2(\theta))\\
&\geq \frac{\bmu_1\|H\|_F^2}{B}\erisk(\theta) - \frac{1}{2B}\nabla L(\theta)^T H \nabla L(\theta) = \erisk(\theta)\left(\frac{\bmu_1\|H\|_F^2}{B} - \frac{\theta^TH^3\theta}{B\theta^TH\theta}\right),
\end{align*}
where the contribution of $\Sigma_2(\theta)$ is $-\theta^TH^3\theta/(B\theta^TH\theta)$. Notice that there is a similar term in the contribution of mean gradient $r(\theta)$.  Together we have $\EE[\erisk(\theta_{t+1})]\geq \EE[\erisk(\theta_t)\gamma(\theta)]$ with 
\begin{equation}\label{eqn: o-o}
\begin{aligned}
\gamma(\theta) &\geq 1-2\eta \frac{\theta^TH^2\theta}{\theta^TH\theta}+ \eta^2\frac{\theta^TH^3\theta}{\theta^TH\theta}+\frac{\bmu_1\eta^2}{B}  \|H\|_F^2 - \frac{\eta^2 \theta^TH^3\theta}{B\theta^TH\theta} \\
&= 1-2\eta \frac{\theta^TH^2\theta}{\theta^TH\theta}+ \eta^2\left(1-\frac{1}{B}\right)\frac{\theta^TH^3\theta}{\theta^TH\theta}+\frac{\bmu_1\eta^2}{B}  \|H\|_F^2
\end{aligned}
\end{equation}
By Lemma \ref{lemma: appendix-matrix-lemma}, we have 
\[
\gamma(\theta)\geq 1-\frac{4\eta^2}{4\eta^2(1-B^{-1})} + \frac{\bmu_1\eta^2}{B}\|H\|_F^2=:\gamma_0.
\]
Let $\gamma_0\leq 1$ yields $\|H\|_F^2\leq B/(\eta\sqrt{(B-1)\bmu_1})$. This bound is trivial for the case $B=1$, where 
\[
\gamma(\theta)\geq 1-2\eta \frac{\theta^TH^2\theta}{\theta^TH\theta} + \frac{\bmu_1\eta^2}{B}\|H\|_F\geq 1-2\eta \lambda_1(H) + \frac{\bmu_1\eta^2}{B}\|H\|_F.
\]
Noticing that $\|H\|_F^2=\sum_{j}\lambda_j^2(H)$, the stability condition needs $1-2\eta \lambda_1(H) + \frac{\bmu_1\eta^2}{B}\|H\|_F\leq 1$, leading to
\[
\sum_{j}\lambda_j^2(H)\leq \frac{2B}{\bmu_1\eta}u^THu\leq \frac{2B}{\bmu_1\eta}\lambda_1(H).
\]
Thus, $\lambda_1(H)\leq 2B/(\bmu_1\eta)$. Consequently, $\|H\|_F^2\leq \frac{2B}{\bmu_1\eta}\lambda_1(H) \leq (\frac{2B}{\bmu_1\eta})^2$.

Combing them, we complete the proof.
\qed 
\end{document}